\documentclass[10pt,reqno]{article}
\usepackage{a4wide}
\usepackage[english]{babel}

\usepackage{amssymb, amsfonts, amsbsy, latexsym}
\usepackage{amsmath}
\usepackage{amsthm}
\usepackage{graphicx}
\usepackage[colorlinks=true, allcolors=blue]{hyperref}

\usepackage{makecell}

\usepackage{lipsum}
\usepackage{amsfonts}
\usepackage{graphicx}
\usepackage{epstopdf}

\usepackage[shortlabels]{enumitem}

\title{Generalization Analysis of Message Passing Neural Networks on Large Random Graphs}

\date{}

\newcommand*{\email}[1]{\href{mailto:#1}{\nolinkurl{#1}} }

\author{Sohir Maskey \thanks{Equal contribution} \thanks{Department of Mathematics, LMU Munich, 80333 Munich, Germany
  (\email{maskey@math.lmu.de}, \email{ylee@math.lmu.de}, \email{levieron@technion.ac.il}, \email{kutyniok@math.lmu.de}).}
\and Ron Levie \footnotemark[1] \thanks{Faculty of Mathematics, Technion - Israel Institute of Technology
}
\and Yunseok Lee\footnotemark[2]
\and Gitta Kutyniok\footnotemark[2] \thanks{Department of Physics and Technology, University of Tromsø, 9019 Tromsø, Norway}
  }

\usepackage{xurl}
\usepackage{abstract}

\usepackage{lipsum}

\usepackage{microtype}
\usepackage{graphicx}
\usepackage{subfigure}
\usepackage{booktabs} 
\usepackage{comment}

\usepackage{array}
\usepackage{booktabs}
\usepackage[dvipsnames]{xcolor}
\usepackage{tikz}

\usepackage{float}

\newcolumntype{M}[1]{>{\centering\arraybackslash}m{#1}}

\usepackage{caption}

\usepackage{multirow}
\usepackage{booktabs}

\usepackage{amssymb, amsfonts, amsbsy, latexsym}
\usepackage{amsmath}
\usepackage{amsthm}
\usepackage{graphicx}
\usepackage[colorlinks=true, allcolors=blue]{hyperref}
\usepackage{bbm}

\usepackage{lipsum}
\usepackage{amsfonts}
\usepackage{graphicx}
\usepackage{epstopdf}

\usepackage[shortlabels]{enumitem}

\def\E{\mathrm{E}}

\def\L2L{L^2(\mathcal{J}) \rightarrow L^2(\mathcal{J})}

\def\cl{L_W}
\def\cmin{\mathrm{d}_{\mathrm{min}}}
\def\cmax{\|W\|_\infty}
\def\Lipf{L_f}
\def\Lipfl{L_{f^{(l)}}}

\def\LipPsi{L_{\Psi}}
\def\d{\mathrm{dist}}

\def\LipW{L_W}

\DeclareMathOperator{\Var}{Var}
\renewcommand{\P}{\mu}
\renewcommand{\E}{\mathbb{E}}
\renewcommand{\L}{\mathcal{L}}

\theoremstyle{plain}

\newtheorem{theorem}{Theorem}[section]
\newtheorem{definition}[theorem]{Definition}

\newtheorem{remark}[theorem]{Remark}
\newtheorem{lemma}[theorem]{Lemma}

\newtheorem{assumption}[theorem]{Assumption}

\newtheorem{approximation problem}[theorem]{Approximation problem}

\newtheorem{corollary}[theorem]{Corollary}

\usepackage{amsmath}
\usepackage{amssymb}
\usepackage{mathtools}
\usepackage{amsthm}

\usepackage[capitalize,noabbrev]{cleveref}

\usepackage{hyperref}
\hypersetup{colorlinks=true, urlcolor=blue, linkcolor=blue, citecolor=blue}

\newcounter{RonCounter}

\newcounter{SohirCounter}

\begin{document}
\maketitle

\begin{abstract}
Message passing neural networks (MPNN) have seen a steep rise in popularity since their introduction as generalizations of convolutional neural networks to graph structured data, and are now considered state-of-the-art tools for solving a large variety of graph-focused problems. We study the generalization error of MPNNs in graph classification and regression. We assume that graphs of different classes are sampled from different random graph models.
We show that, when training  a MPNN on a dataset sampled from such a distribution, the generalization gap increases in the complexity of the MPNN, and  decreases, not only with respect to the number of training samples, but also with the average number of nodes in the graphs. This shows how a MPNN with high complexity can generalize from a small dataset of graphs, as long as the graphs are large.
The generalization bound is derived from a uniform convergence result, that shows that any MPNN, applied on a graph, approximates the MPNN applied on the geometric model that the graph discretizes. 

\end{abstract}

\section{Introduction}
A graph is an abstract structure that represents a set of objects along with the connections that exist between those objects. 
In many important fields, such as chemistry, biology, social networks, or computer graphics, data can be described by graphs.
This has led to a tremendous interest in the development of
machine learning models for graph-structured  data in recent years.
A ubiquitous tool for processing such data
are graph convolutional neural networks (GCNNs), which extend standard Euclidean convolutional neural networks (CNNs) to graph-structured data.

Most GCNNs used in practice can be described using the general architecture of \emph{Message Passing Neural Networks (MPNNs)}.
MPNNs generalize the convolution operator to graph domains by a neighborhood aggregation or message passing scheme. By $\mathbf{f}_i^{(t-1)}$
denoting the feature of node $i$ in layer $t-1$ and $\mathbf{e}_{j,i}$ denoting edge features from node $j$ to $i$, one layer in a message passing graph neural network is given by
\begin{equation}
    \label{eq:gMPNN}
    \mathbf{f}_i^{(t)} = \Psi^{(t)} \Big(\mathbf{f}_i^{(t-1)}, \mathbf{AGG}  \big\{ \Phi^{(t)} (\mathbf{f}_i^{(t-1)}, 
    \mathbf{f}_j^{(t-1)}, \mathbf{e}_{j,i} )\big\}_{j \in \mathcal{N}(i)} \Big),
\end{equation}
where $\mathcal{N}(i)$ is the set of nodes connected to node $i$, $\mathbf{AGG}$ denotes a differentiable and permutation invariant function, e.g., sum, mean, or max, and  $\Psi^{(t)}$ and $\Phi^{(t)}$ denote differentiable functions such as MLPs (Multi-Layer Perceptrons) \cite{Fey/Lenssen/2019}.

MPNNs have shown state-of-the-art performance in many graph machine learning tasks such as  node or graph classification.
As such, MPNNs had a tremendous impact to the applied sciences, with promising achievements such as discovering a new class of antibiotics \cite{STOKES2020688}, and has impacted the industry with applications in social media, recommendation systems, and 3D reconstruction, among others  (see, e.g., \cite{10.1145/3219819.3219890, 10.1145/3219819.3219869,wang2018pixel2mesh, monti2019fake, 10.1145/3308558.3313488}).
The practical success of MPNNs led to a significant boost in research aimed at understanding the theoretical properties of MPNNs. See, e.g., the variational inference point of view of MPNNs \cite{10.5555/3045390.3045675}, and algorithmic alignment of MPNNs with combinatorial algorithms \cite{xu2018how, Morris_Ritzert_Fey_Hamilton_Lenssen_Rattan_Grohe_2019}.

In this paper we study the generalization capabilities of MPNNs  {with mean aggregation} in a graph classification task. 
{Previous works developed generalization bounds that do not depend on any model of the data, namely, graphs in these works can be generated and labeled in any arbitrary way \cite{scarselli2018vapnik, pmlr-v119-garg20c, liao2021a}. In this work, we consider a generative model for the graphs which is theoretically powerful and general on the one hand, and allows much tighter generalization bounds on  the other hand.}

Formally, we are given pairs of graphs and graph signals $\mathbf{x} = (G, \mathbf{f})$ and a target output $\mathbf{y}$, where $(\mathbf{x}, \mathbf{y})$ are jointly drawn from  a distribution $\mu_\mathcal{G}(\mathbf{x},\mathbf{y})$. The goal is to learn a MPNN $\Theta$ that approximates $\mathbf{y}$ by $\Theta(\mathbf{x})$. For this, one uses a loss function $\mathcal{L}$, 
which measures the discrepancy between the true label $\mathbf{y}$ and the output of the MPNN $\Theta(\mathbf{x})$.  The aim of a machine learning algorithm is to minimize the statistical loss (also called expected loss)
\[
R_{exp}(\Theta) = 
\E_{(\mathbf{x},\mathbf{y})\sim \mu_\mathcal{G}} \Big[ \mathcal{L}(\Theta(\mathbf{x}), \mathbf{y} ) \Big].
\]

In (data-driven) machine learning one has only access to a training set instead of knowing the distribution $\mu_\mathcal{G}$. Namely, we consider a multi-graph setting, where the training set 
$\mathcal{T} = (\mathbf{x}^i = (G^i, \mathbf{f}^i), \mathbf{y}^i)_{i=1}^m$ is a collection of $m$  samples drawn i.i.d. from the distribution $\mu_\mathcal{G}(\mathbf{x},\mathbf{y})$. Then, instead of minimizing the statistical loss,
one minimizes the empirical loss, given by
\[
R_{\mathrm{emp}}(\Theta) = \frac{1}{m} \sum_{i=1}^m \mathcal{L}(\Theta(\mathbf{x}^i), \mathbf{y}^i ).
\]
The optimized MPNN then depends on the dataset, and is hence denoted by $\Theta_{\mathcal{T}}$.
The \emph{generalization error} is defined to be 
\begin{equation}
\label{eq:genError}
GE(\Theta_{\mathcal{T}}) = | R_{\mathrm{exp}}(\Theta_{\mathcal{T}}) - R_{\mathrm{emp}}(\Theta_{\mathcal{T}})|.
\end{equation}
One then usually bounds (\ref{eq:genError}) by the \emph{uniform generalization error} 
\begin{equation}
\label{eq:genError2}
GE = \sup_{\Theta} | R_{\mathrm{exp}}(\Theta) - R_{\mathrm{emp}}(\Theta)|,
\end{equation}
where the supremum is taken over some space of MPNNs. 
Bounds of $GE$ typically take the form $GE^2\leq  \frac{C}{m}q(N)$, where $C$ is a constant that describes the complexity of the model class  (e.g., number of parameters), $m$ is the size of the training set, and $q(N)$ is a constant that depends on the (average) size of the graphs. For such bounds, see, e.g., VC-dimension based bounds \cite{scarselli2018vapnik}, Rademacher complexity based bounds \cite{pmlr-v119-garg20c}, and PAC-Bayesian based bounds \cite{liao2021a}.

While in previous bounds from the literature $q(N)$ either increases in $N$ or in the average degree,  
  in this paper we  develop a generalization bound that decays in the average number of nodes $N$. The idea is to treat the nodes of each graph as randomly sampled from some random graph model.
In this point of view, not only the different graphs $\mathbf{x}^i$ are seen as random samples, but the union of all nodes of all graphs comprise together the random samples of the empirical loss. 
 In the spirit of Monte Carlo theory, such a point of view should lead to a decay of the error between the empirical and statistical losses as $N$ increases. As opposed to graphs, nodes cannot be seen as independent, due to the correlations entailed by the graph structure. Hence, our analysis focuses on developing Monte Carlo error bounds in a correlated nodes regime. 



Since in our approach we model graphs as randomly sampled from underlying continuous models, 
 we
define the application of message passing neural networks, not only on graphs, but also on the underlying space from which graphs are sampled. We then formulate and prove the following convergence result, that we write here informally.
Let $\mathbf{x} = (G,\mathbf{f})$ be drawn from the model $\chi$, then with high probability, we have for all MPNNs $\Theta$ 
\[
    \|\Theta(\mathbf{x}) - \Theta(\chi)\| = O(N^{-\alpha}),
\]
where $N$ is the number of nodes in $\mathbf{x}$ and $\alpha>0$. 
Based on this convergence result, we are able prove a generalization bound that decays in $N$.

\subsection{{Validity of the Proposed Model}}
\label{SecNew1}

{The random graph models in our work are graphons \cite{lovasz} with associated graphon signals (see Definition  \ref{def:RGM}). The main assumption in our analysis is that graphs that are sampled from the same graphon belong to the same class. While this may seem like a limitation, it is actually a very mild and reasonable assumption. It is well known that equivalence classes of isomorphic graphs can be characterized by homomorphism densities \cite{lovasz67}. Namely, given two graphs $G_1,G_2$, if (and only if) for every simple graph $F$ the number of homomorphisms from $F$ to $G_1$ is equal to the number of homomorphisms from $F$ to $G_2$, then $G_1$ is isomorphic to $G_2$. Graphon analysis relaxes this observation to a continuous similarity measure. A sequence of graphs $\{G_j\}_{j\in\mathbb{N}}$ is said to converge in the graphon sense, if for every simple graph $F$ the homomorphism densities of $F$ in the graphs $\{G_j\}_{j\in\mathbb{N}}$ converge to some value.  Graphs from such a sequence can be thought of as being similar in some sense which relaxes the combinatorial notion of graph isomorphism. Moreover, for each such converging sequence, there is a unique (up to some symmetry) limit object, called a graphon. This graphon is also seen as a  generative model for graphs in the respective sequence, where graphs are generated by randomly sampling the graphon (see Definition \ref{def:RGM}). Now, since it is well known that MPNNs cannot distinguish between isomorphic graphs, it is also unreasonable to expect them to separate two graphs that are sampled from the same graphon. We hence assume that two graphs that are sampled from the same graphon belong to the same class (but not necessarily vice versa). This assumption allows us to derive a generalization bound that is much tighter than previously proposed bounds (see Figure \ref{tbl:Gen1} for comparison).} 

{From a practical stance, our graphon assumption is reasonable since many graph models are special cases of graphons, like Erdős–Rényi, stochastic block model, and random geometric graphs \cite{RGGPenrose}. Moreover, the decoder of a graph variational autoencoder \cite{KipfWellingGraphVAE} can be seen as a graphon.} 


\subsection{Related Work}
\label{subsec:relatedWork}
In this subsection we  briefly survey  different approaches for studying the convergence and  generalization capabilities  of GCNNs  that were introduced in previous contributions. We give a comparison with our results in Section \ref{sec:Main}. 

 
In \cite{levie2021transferability}, the authors introduce the notion of GCNN transferability -- the ability to transfer a GCNN between different graphs, which is closely related to generalization. 
For example, \cite{ LevieTransfSpectralGraphFiltersShort, Gama_2020, kenlay2021interpretable} show that the output of spectral-based GCNNs is linearly stable  with respect to perturbations of the input graphs. 
\cite{levie2021transferability} prove that spectral-based methods are transferable under graphs and graph signals that are sampled from the same latent space.
\cite{keriven2020convergence, 9356126,ruiz2020graphon, maskey2021transferability} show that spectral-based GCNNs are transferable under graphs that approximate the same limit object -- the so called graphon. 

 
In \cite{scarselli2018vapnik}, the authors provide generalization bounds that are comparable to VC-dimension bounds known for CNNs.
These bounds are improved in  \cite{pmlr-v119-garg20c}, which provides the first
data dependent generalization bounds for MPNNs with sum aggregation that  are comparable to Rademacher bounds for recurrent neural networks. \cite{liao2021a} derive a generalization bound via a PAC-Bayesian approach that is governed by the maximum node degree and spectral norms of the weights. \cite{StabGenofGCNN} consider generalization abilities of single-layer spectral GCNNs for node-classification task and provide a generalization bound that is directly proportional to the largest eigenvalue of the graph Laplacian.
Another paper of this flavor is
\cite{yehudai2020size}, showing that certain MPNNs (with sum aggregation) do not generalize from small to large graphs.

\subsection{Main Contributions}
\label{subsec:mainContr}

We follow the route of \cite{keriven2020convergence} and consider
graphs as discretizations of continuous spaces in our analysis, called random graph models (RGM, see Definition \ref{def:RGM}).  We introduce a continuous version of message passing neural networks -- the realization of MPNNs 
on random graph models, which we call cMPNNs. Such cMPNNs are seen as limit objects of graph MPNNs, when the number of graph nodes goes to infinity.
We prove, up to our knowledge, the first  convergence result  of the graph MPNN to the corresponding cMPNN as the number of nodes increases, which is uniform in the choice of the MPNN.  
 
For the generalization analysis, we assume that the data distribution $\mu_\mathcal{G}$ represents graphs which are randomly sampled from a collection of template RGMs, with a random number of nodes. Using our convergence 
results, we can then prove that the generalization error between the training set and the true  distribution is small.  
Here, we give the following informal  version of Theorem \ref{thm:mainGenError}.

\begin{theorem}[Informal version of Theorem \ref{thm:mainGenError}]   
Consider a graph classification task with $m$ training samples $\mathcal{T} = (\mathbf{x}^i = (G^i, \mathbf{f}^i), \mathbf{y}^i)_{i=1}^m$ drawn i.i.d. from the data distribution $\mu_\mathcal{G}(\mathbf{x},\mathbf{y})$ on a metric-measure space $\chi$ of dimension $D_\chi$. Suppose that the size $N$ of each graph in $\mathcal{T}$  is drawn from a distribution $\nu$.
Then 
\begin{align*}
 \E_{\mathcal{T}\sim \mu_\mathcal{G}^m}&\left[ \sup_\Theta\big(R_{emp}(\Theta) -   R_{exp}(\Theta)  \big)^2\right]   \leq \frac{C}{m}\E_{N \sim \nu}\big[N^{-\frac{1}{D_\chi+1}} \big].
\end{align*}
\end{theorem}

 The constant $C$  represents the complexity of the hypothesis space of the network, via the Lipschitz constants of the message and update functions and the depth of the MPNNs. 
 
 Theorem \ref{thm:mainGenError} shows how we can use fewer graphs $m$ than model complexity $C$ when training MPNNs if the graphs are sufficiently large.


\section{Preliminaries}
\label{sec:Prelimnaries}
A weighted graph $G=(V, \mathbf{W}, E)$ with $N$ nodes is a tuple, where $V=\{1, \ldots, N\}$  is the node set. The edge set is given by $E \subset V \times V$, where $(i,j) \in E$ if node $i$ and $j$ are connected by an edge. $\mathbf{W} = (w_{k,l})_{k,l}$  is the weight matrix, assigning the weight $w_{i,j}$ to the edge $(i,j) \in E$, and assigning zero if $(i,j)$ is not an edge. 
The degree $\mathrm{d}_i$ of a node $i$ is defined as $\mathrm{d}_i = \sum_{j=1}^N w_{i,j}$.
If $G$ is a simple graph, i.e., a weighted graph with $\mathbf{W} \in \{0,1\}^{N \times N}$, the degree $\mathrm{d}_i$ is the number of nodes connected to node $i$ by an edge.
We define a \emph{graph signal} $\mathbf{f}: V \rightarrow \mathbb{R}^F$ as a function that maps nodes to their features in $\mathbb{R}^F$, where $F\in\mathbb{N}$ is the feature dimension. The signal $\mathbf{f}$ can be represented by a matrix $\mathbf{f}=(\mathbf{f}_1,\ldots,\mathbf{f}_N)\in \mathbb{R}^{N\times F}$, where $\mathbf{f}_i \in \mathbb{R}^F$ is the feature at node $i$. 
We also call $\mathbf{f}$ a \emph{(graph) feature map}.

For a random variable $Y$ distributed according to $\kappa$, and a function $F$ of $Y$, we denote by $\E_{Y \sim \kappa}[F(Y)]$ the expected value of $F(Y)$. Similarly, we denote by $\Var_{Y \sim \kappa}[F(Y)]$  the variance of $F(Y)$.

\subsection{Message Passing Graph Neural Networks} 
\emph{Message passing graph neural networks (gMPNNs)} are defined by realizing an architecture of a \emph{message passing neural network (MPNN)} on a graph. MPNNs are defined independently of a particular graph.

\begin{definition}
\label{def:MPNN}
Let $T \in \mathbb{N}$ denote the number of layers. For $t=1, \ldots, T$, let  $\Phi^{(t)}: \mathbb{R}^{2 F_{t-1}} \to \mathbb{R}^{H_{t-1}}$ and $\Psi^{(t)}:\mathbb{R}^{F_{t-1} + H_{t-1}} \to \mathbb{R}^{F_{t}}$ be functions that we call the \emph{message} and \emph{update} functions, where $F_t \in \mathbb{N}$ is called the feature dimension of layer $t$. 
The corresponding \emph{message passing neural network (MPNN)} 
$\Theta$ is defined to be the sequence 
\[
\Theta = ((\Phi^{(t)})_{t=1}^T, (\Psi^{(t)})_{t=1}^T).
\]
\end{definition}

The message and the update function in Definition \ref{def:MPNN} are often defined as multi-layer-perceptrons (MLPs). 
In a MPNNs, messages are sent between nodes and aggregated. 
An \emph{aggregation scheme} is a permutation invariant function that takes the collection of features in the edges of each node and computes a new nodes feature. In this paper, we consider MPNNs with \emph{mean aggregation}.
Then, a gMPNN processes graph signals by realizing a MPNN on the graph as follows.

\begin{definition}
\label{def:gMPNN}
Let $G=(V,\mathbf{W})$ be a weighted graph and $\Theta$ be a MPNN, as defined in Definition \ref{def:MPNN}. For each $t \in \{ 1, \ldots, T \}$, we define the \emph{gMPNN} $\Theta^{(t)}_G$ as the mapping that maps input graph signals $\mathbf{f}=\mathbf{f}^{(0)}\in \mathbb{R}^{N\times F_0}$ to the features in the $t$-th layer by
\begin{equation*}
    \Theta^{(t)}_G: \mathbb{R}^{N \times F_0} \rightarrow \mathbb{R}^{N \times F_t}, \ \ \  \mathbf{f} \mapsto \mathbf{f}^{(t)} = (\mathbf{f}^{(t)}_i)_{i=1}^N,
\end{equation*}
where $\mathbf{f}^{(t)}\in\mathbb{R}^{N\times F_t}$ are defined sequentially by
\begin{equation*}
\label{eq:graphAgg}
    \begin{aligned}
& \mathbf{m}_i^{(t)} := \frac{1}{\mathrm{d}_i}\sum_{j=1}^N w_{i,j} \Phi^{(t)}(\mathbf{f}_i^{(t-1)}, \mathbf{f}_j^{(t-1)}) \\
& \mathbf{f}_i^{(t)} := \Psi^{(t)}(\mathbf{f}_i^{(t-1)}, \mathbf{m}_i^{(t)}),\\
\end{aligned}
\end{equation*}
for every $i\in V$.  
We call $\Theta_G := \Theta_G^{(T)}$ a \emph{message passing graph neural network (gMPNN)}.
\end{definition}

Given a MPNN $\Theta$ as defined in Definition \ref{def:MPNN}, the output $\Theta_G(\mathbf{f}) \in \mathbb{R}^{N  \times F_T}$ is a graph signal. In graph classification or regression, the network should output a single feature for the whole graph.
Hence, the output of a gMPNN after \emph{global pooling} is a single vector $\Theta_G^P(\mathbf{f}) \in \mathbb{R}^{F_T}$, defined by
\[
 \Theta_{G}^P(\mathbf{f}) =   \frac{1}{N} \sum_{i=1}^N \Theta_G(\mathbf{f})_i.
\]

For brevity, in this paper we typically do not distinguish between a MPNN and its realization on a graph.

\subsection{Random Graph Models}
Let $(\chi, d, \P)$ be a metric-measure space,   where $\chi$ is a set, $d$ is a metric and $\P$ is a probability Borel measure. 

A \emph{kernel} (also called a \emph{graphon}), is a measurable mapping $W: \chi \times \chi \to \mathbb{R}$. The points $x\in \chi$ of the metric space are seen as the nodes of a continuous model,  and the kernel is seen as a continuous version of a weight matrix. Kernels are treated as generative graph models using the following definition. 

\begin{definition}
\label{def:RGM}
A \emph{random graph model (RGM)}   on $(\chi, d, \mu)$ is defined as a pair $(W,f)$  of a kernel $W:\chi \times \chi \to \mathbb{R}$ and a measurable function  $f: \chi \to \mathbb{R}$ called a \emph{metric-space signal}. We define a \emph{random graph} with corresponding node features $(G, \mathbf{f})$ by sampling $N$  i.i.d. random points $X_1, \ldots, X_N$ from $\chi$, with probability density $\mu$, as the nodes of $G$. The weight matrix $\mathbf{W}=(w_{i,j})_{i,j}$ of $G$ is defined by $w_{i,j} = W(X_i, X_j)$ for $i,j =1, \ldots,N$. The graph signal $\mathbf{f}$ is defined by $\mathbf{f}_i = f(X_i)$. We say that $(G, \mathbf{f})$ is \emph{drawn} from $W$, and denote $(G, \mathbf{f}) \sim (W,f)$. 
\end{definition}

\subsection{Continuous Message Passing Neural Networks}
\label{subsec:cMPNN}
Given a MPNN, we define \emph{continuous message passing neural networks (cMPNNs)} that act on kernels and metric-space signals $f:\chi\rightarrow \mathbb{R}^F$, 
 by replacing the graph node features and the aggregation scheme in (\ref{eq:graphAgg}) by continuous counterparts.  
Let $W$ be a kernel.
We define the \emph{kernel degree} of $W$ at $x\in\chi$ by 
\begin{equation}
    \label{eq:kernelDegreeMain}
    \mathrm{d}_W(x) = \int_\chi W(x,y) d\P(y).
\end{equation}
Consider a message signal $U:\chi\times\chi\rightarrow\mathbb{R}^H$, where $U(x,y)$ is interpreted as a message sent from the point $y$ to the point $x$ in $\chi$.
We define the continuous mean aggregation of $U$ by 
\[
M_W(U)(x) = \int_\chi
\frac{W(x, y)}{ \mathrm{d}_W(x) }  U(x, y) d\P(y).
\]
Given the messages $U(x,y)=\Phi(f(x),f(y))$, where $\Phi:\mathbb{R}^{2F}\rightarrow\mathbb{R}^H$, we have
\[
M_W(U)(x) = M_W\Big(\Phi\big(f(\cdot),f(\cdot\cdot)\big)\Big)(x) 
=
\int_\chi
\frac{W(x, y)}{ \mathrm{d}_W(x) }  \Phi\big(f(x),f( y)\big) d\P(y).
\]
By abuse of notation, we often denote in short
 $\Phi(f,f):=\Phi\big(f(\cdot),f(\cdot\cdot)\big)$.


By replacing mean aggregation by continuous mean aggregation in Definition \ref{def:gMPNN}, the same  message and update functions that define a graph MPNN can also process metric-space signals. 

\begin{definition}
\label{def:cMPNN}
Let $W$ be a kernel and $\Theta$ be a MPNN, as defined in Definition \ref{def:MPNN}. For each $t \in \{ 1, \ldots, T \}$, we define $\Theta^{(t)}_W$  as the mapping that maps the input signal  to the signal in the $t$-th layer by
\begin{equation}
    \Theta^{(t)}_W: L^2(\chi) \rightarrow L^2(\chi), \ \ \ f \mapsto f^{(t)},
    \label{cMPNNdef}
\end{equation}
where $f^{(t)}$ are defined sequentially by
\begin{equation}
    \begin{aligned}
& g^{(t)}(x) = M_W \Big(\Phi^{(t)}\big(f^{(t-1)}, f^{(t-1)}\big)\Big) (x) \\
& f^{(t)}(x) = \Psi^{(t)}\Big(f^{(t-1)}(x), g^{(t)}(x)\Big)\\
\end{aligned}
\end{equation}
and $f^{(0)} = f : \chi\rightarrow \mathbb{R}^{F_0}$ is the input metric-space signal. We call $\Theta_W := \Theta_W^{(T)}$ a \emph{continuous message passing neural network (cMPNN)}.
\end{definition}

As with graphs, 
 the output of a cMPNN $\Theta_W$ on a metric-space signal $f:\chi \to \mathbb{R}^{F_0}$ is another metric-space signal $\Theta_W(f):\chi \to \mathbb{R}^{F_T}$.
The output of a cMPNN after \emph{global pooling} is a single vector $\Theta_W^P(f) \in \mathbb{R}^{F_T}$, defined by 
$
 \Theta_{W}^P(\mathbf{f}) =   \int_\chi \Theta_W(f)(x) d\P(x)$. 

\subsection{Data Distribution for  Graph Classification Tasks}
\label{subsec:GenDataDistr}

In the following, we consider  a training data  
$\mathcal{T} = \big(\mathbf{x}^i = (G^i, \mathbf{f}^i), \mathbf{y}^i\big)_{i=1}^m$ 
of graphs $G^i$, graph signals $\mathbf{f}^i$, and corresponding values $\mathbf{y}^i$ that can represent the classes of the graph-signal pairs. The training data is assumed to be drawn i.i.d. from a distribution $\mu_\mathcal{G}(\mathbf{x},\mathbf{y})$ that we describe next.

In this paper, we focus on classification tasks.
More precisely we have 
classes $j = 1, \ldots, \Gamma$, each represented by a RGM $(W^j, f^j)$ on a metric-measure space $(\chi^j, d^j, \P^j)$. In fact, we suppose that each class corresponds to a set of metric spaces. For example, a graph representing a chair can be sampled from a template of either an office chair, a garden chair, a bar stool, etc., and each of these is represented by a metric space. For simplicity of the exposition, we however treat every template metric space as its own class. This does not affect our analysis. 

The distribution
$\mu_{\mathcal{G}}(\mathbf{x},\mathbf{y})$  is defined via the following procedure of data sampling. 
For sampling one graph, first, choose a class with probability $\gamma_j$, i.e., for $(\mathbf{x},\mathbf{y}) \sim \mu_\mathcal{G}$ and $j = 1 , \ldots, \Gamma$,
 $
\gamma_j = \mathbb{P}(\mathbf{y} = j)
$. 
Independently of the choice of the class, choose the number of nodes $N \sim \nu$, where $\nu$ is a discrete distribution on 
    $ N \in \mathbb{N}$.
After choosing a class $\mathbf{y} \in \{1, \ldots, \Gamma \}$ and the graph size $N$, a random graph $(G, \mathbf{f}) \sim (W^{\mathbf{y}}, f^{\mathbf{y}})$ with $N$ nodes is drawn from the space $\chi^{\mathbf{y}}$  with probability density of the nodes $(\P^{\mathbf{y}})^N$.


The notation $\mathcal{T} \sim \mu_\mathcal{G}^m$ describes a dataset $\mathcal{T}$ consisting of $m$ samples $(\mathbf{x}^1, \mathbf{y}^1), \ldots, (\mathbf{x}^m, \mathbf{y}^m)$ drawn i.i.d. from $\mu_\mathcal{G}$. We refer to Subsection \ref{subsec:ProbSpaceDataset} in the appendix for a detailed definition of the distribution $\mu_\mathcal{G}$.

\section{Convergence and Generalization of MPNNs}
\label{sec:Main}


In this section, we provide our main results on   convergence (Subsection \ref{sec:Convergence}) and generalization (Subsection \ref{subsec:Generalization}) of MPNNs.  
For $z \in \mathbb{R}^F$, we define  $\|z\|_\infty = \max_{j=1, \ldots, F} |z_j|$. Given a metric space $(\mathcal{Y},d_\mathcal{Y})$, we define the infinity norm of a vector valued function $g:\mathcal{Y} \to \mathbb{R}^F$ by  $\|g\|_\infty = \max_{j=1, \ldots, F} {\rm ess}\sup_{y \in \mathcal{Y}} | (g(y) )_j|$. The function $g$
is called \emph{Lipschitz continuous} if there exists a constant  $L_g \in \mathbb{R} $ such that for all $y,y' \in \mathcal{Y}$, 
\[
\|g(y) - g(y')\|_\infty \leq L_g d_{\mathcal{Y}}(y,y').
\]
If the domain $\mathcal{Y}$ is Euclidean, we always endow it with the $L^{\infty}$-metric.

We measure the error between the output of a continuous MPNN and a gMPNN after pooling as follows. Given a graph signal $\mathbf{f}  \in \mathbb{R}^{N  \times F}$ and a metric-space signal $f: \chi \to \mathbb{R}^F$, both the graph and the continuous MPNN map to the same output space, i.e,   $\Theta_W^P(f) ,\Theta_G^P(\mathbf{f})\in \mathbb{R}^{F_T}$. Namely, the output dimension of $\Theta^P$ is independent of the random graph model it is realized on and also independent of the graph. Hence, we define the error to be the supremum norm $\|\Theta_W^P(f) -\Theta_G^P(\mathbf{f}) \|_{\infty}$.
 We define the $\varepsilon$-covering numbers of the metric space $\chi$, denoted by $\mathcal{C}(\chi, \varepsilon, d)$, as the minimal number of balls of radius $ \varepsilon$ required to cover $\chi$.

 For every $j=1, \ldots, \Gamma$, we make the following assumptions, which hold for the remainder of the paper.  We assume that there exist constants $C_{\chi^j}, D_{\chi^j} > 0$ such that
 \begin{equation}
 \label{eq:dim}
     \mathcal{C}(\chi^j, \varepsilon, d) \ \leq \  C_{\chi^j} \  \varepsilon^{-D_{\chi^j}}
 \end{equation} for every $\varepsilon > 0$. Denote $D_{\chi}= \max_j D_{\chi^j}$ and $C_{\chi}= \max_j C_{\chi^j}$ Such constants exist  for every metric space with finite Minkowski dimension (see Appendix \ref{AppendixA}). We assume that ${\rm diam}(\chi^j):=\sup_{x,y\in\chi^j}\{d(x,y)\} \leq 1$.  
 Further, we only consider kernels $W^j$ such that there exists a constant $\cmin > 0$ satisfying
\begin{equation}
\label{eq:d_W2}
{\rm d}_{W^j}(x) \geq \cmin,
\end{equation}
where the kernel degree ${\rm d}_{W^j}$ is defined in (\ref{eq:kernelDegreeMain}).
We moreover assume that $W^j(x, \cdot)$ and $W^j(\cdot, x)$  are Lipschitz continuous (with respect to its second and first variable, respectively) with Lipschitz constant $L_{W^j}$ for every $x \in \chi$. 
We also assume that the metric-space signal $f^j: \chi \to \mathbb{R}^{F}$ is  Lipschitz continuous. Since the diameter of $\chi^j$ is finite,  this means that $f^j \in L^\infty(\chi)$. 
We consider the following class of MPNNs  
\[
\begin{aligned}
& \mathrm{Lip}_{L, B} := \\
&\Big\{ \Theta = \big((\Phi^{(l)})_{l=1}^T, (\Psi^{(l)})_{l=1}^T\big) \; \Big| \;  \forall l=1,\ldots, T, \ \ 
\Phi^{(l)}:\mathbb{R}^{F_l} \to \mathbb{R}^{H_l} {\rm\ and\ }   \Psi^{(l)}:\mathbb{R}^{F_l + H_l}  \to \mathbb{R}^{F_{l+1}}  \\ 
& \quad \quad \quad \quad \text{ satisfy \ }   L_{\Phi^{(l)}}, L_{\Psi^{(l)}}  \leq L  \text{ and } \|\Phi^{(l)}(0,0)\|_\infty , \|\Psi^{(l)}(0,0)\|_\infty \leq B
\Big\}. 
\end{aligned}
\]

\subsection{Convergence}
\label{sec:Convergence}
In this subsection we show that the error between the cMPNN and the according gMPNN decays when the number of nodes increases. 

\begin{theorem}
\label{thm:MainInProb}
Let $W:\chi^2\rightarrow \mathbb{R}$ be a Lipschitz continuous kernel with Lipschitz constant $L_W$, where the metric space $\chi$ satisfies (\ref{eq:dim}) with respect to the constants $C_{\chi}, D_{\chi} > 0$, and $W$ satisfies (\ref{eq:d_W2}).
Consider a graph $(G, \mathbf{f}) \sim (W, f)$ with $N$ nodes $X_1, \ldots, X_N$ drawn i.i.d. from $\chi$ with probability density $\P$. 
Then,  for every Lipschitz continuous $f:\chi \to \mathbb{R}^F$, 
\[
\E_{X_1, \ldots, X_N \sim \mu^N} \left[ \sup_{\Theta \in \mathrm{Lip}_{L,B}} \big\| 
\Theta^P_G(\mathbf{f}) - 
\Theta^P_W(f) \big\|_\infty^2 \right] 
\leq
C'\big(1+\|f\|_{\infty}^2 + L_f^2\big) \frac{\log(N)}{N^{1/(D_{\chi}+1)}}+\mathcal{O}(N^{-1}),
\]
where  $C'$ is defined in Subsection \ref{subsec:conv} of the appendix. 
\end{theorem}

\begin{remark}
\label{rem:ThmMainProb}
The constant $C'$ in Theorem \ref{thm:MainInProb} 
depends  polynomially on the Lipschitz constants $L_{\Phi^{(l)}}$ and $L_{\Psi^{(l)}}$ of the message and update functions $\Phi^{(l)}$ and $\Psi^{(l)}$, on the  so called \emph{formal biases} $\|\Phi^{(l)}(0,0)\|_\infty$ and $\|\Psi^{(l)}(0,0)\|_\infty$, on $\|W\|_\infty$, on the Lipschitz constant $L_W$ of $W$, on $\sqrt{\log(C_\chi)}  + \sqrt{D_\chi}$,  and on $\frac{1}{{\rm d}_{\min}}$, where the degree of the polynomial is  $T$.  
A regularization of  these  constants can alleviate the exponential dependency of the bound on $T$.  
\end{remark}

The proof of Theorem \ref{thm:MainInProb} is given in Subsection \ref{subsec:conv} of the appendix.

\paragraph{Discussion and Comparison to other Convergence Results}
The work closest related to our convergence results is \cite{keriven2020convergence}, where the authors show convergence of  a fixed spectral GCNN to its continuous counterpart with comparable regularity assumptions as in Theorem \ref{thm:MainInProb}.  
Our result  holds for MPNNs, which are more general than spectral GCNNs. Moreover, our bound is uniform in the choice of the MPNN $\Theta$. This last property is essential for leveraging the convergence result to  derive a  generalization error.  Indeed, using the bound from \cite{keriven2020convergence}, for each MPNN $\Theta$ there is a different high probability event $\mathcal{E}_{\Theta}$ where the convergence error is small. However, the trained MPNN $\Theta=\Theta_{\mathcal{T}}$ depends on the dataset $\mathcal{T}$ and cannot be fixed in the analysis. Hence, we would need to intersect all events $\bigcap_{\Theta}\mathcal{E}_{\Theta}$ to guarantee a small convergence error of the trained network $\Theta_{\mathcal{T}}$, which would not result in an event of high probability.



\subsection{Generalization}
\label{subsec:Generalization}

In this subsection, we state the main result of our paper, which provides a non-asymptotic bound on the generalization error of MPNNs, as defined in (\ref{eq:genError2}). 
 We consider a graph classification task with a training set $\mathcal{T} = (\mathbf{x}^i = (G^i, \mathbf{f}^i), \mathbf{y}^i)_{i=1}^m $ and $\Gamma$ classes. The graphs and graph features in $\mathcal{T}$ are drawn i.i.d. from a probability distribution $\mu_\mathcal{G}(\mathbf{x},\mathbf{y})$ as described in Subsection \ref{subsec:GenDataDistr}. We recall that the distribution that samples the size of the graph  is denote by $\nu$. 

Given a MPNN with pooling, $\Theta^P$, and its output dimension $\mathbb{R}^{F_T}$, we consider a non-negative loss function 
$
\mathcal{L}: \mathbb{R}^{F_T}\times \{1, \ldots ,\Gamma\} \to [0,\infty)
$. 
Additionally, we assume that $\mathcal{L}$ is Lipschitz continuous with Lipschitz constant $L_\mathcal{L}$. 
Note that although the cross-entropy loss, a popular choice for loss function in classification tasks, is not Lipschitz-continuous, cross-entropy composed on softmax is.

\begin{theorem}
\label{thm:mainGenError}
There exists a constant $C>0$ such that 
\[
\begin{aligned} &   \E_{\mathcal{T}\sim p^m}\left[\sup_{\Theta \in \mathrm{Lip}_{L,B}} \Big(R_{emp}(\Theta^P)   - R_{exp}(\Theta^P) \Big)^2  \right]\leq   \frac{2^\Gamma8\|\mathcal{L}\|_\infty^2\pi}{m} \\ &+ \frac{2^\Gamma L_{\mathcal{L}}^2   C}{m} \sum_{j}  \gamma_j \big(1 + \|f^j\|^2_\infty + L_{f^j}^2\big) 
\cdot    \left( \E_{N \sim \nu} \left[ \frac{1}{N} + \frac{1 + \log(N)}{N^{1/(D_{\chi^j} + 1)}} + \mathcal{O}\left( \exp(-N)N^{3T-\frac{3}{2}} \right) \right] \right),
\end{aligned}
\]
where $C$ is specified in Subsection \ref{subsec:GenErrorProof} of the appendix. 
\end{theorem}
The proof of Theorem \ref{thm:mainGenError} is given in Subsection \ref{subsec:GenErrorProof} of the appendix.

\begin{remark}
\label{Rem34}
The constant $C$ in Theorem \ref{thm:mainGenError} represents the complexity of the class $\mathrm{Lip}_{L,B}$ and can be bounded similarly to the constant $C'$ from Theorem \ref{thm:MainInProb}, as described in Remark \ref{rem:ThmMainProb}. We summarize its dependencies on the parameters of the MPNN and the RGM by  $\sqrt{C}\lesssim B L^{2T}  \frac{1}{\mathrm{d}_{\mathrm{min}}^{T+1}} \max_{j=1,\ldots, \Gamma} \big(\sqrt{\log (C_{\chi^j})} +  \sqrt{D_{\chi^j}} \big) L_{W^j} \|W^j\|_\infty^{T}$ and refer to  Subsection \ref{subsec:disGenBound} of the appendix for more details. 
  %
  Similarly to Remark \ref{rem:ThmMainProb} the exponential dependency of the constant $C$ in Theorem \ref{thm:mainGenError} on the depth $T$ and the polynomial dependency on the uniform Lipschitz bound $L$ can be alleviated by regularizing the latter. We also note that the exponential dependency on the number of classes $\Gamma$   in Theorem \ref{thm:mainGenError} can be eliminated by assuming that the data is representative, i.e., if the number of training samples that fall into class $j=1, \ldots, \Gamma$ is deterministically $\gamma_j m$.
\end{remark} 



The term $\frac{2^\Gamma8\|\mathcal{L}\|_\infty^2\pi}{m}$ in Theorem \ref{thm:mainGenError} does not depend on the model complexity and is typically much smaller than the second term. Hence, it does not affect  bias–variance tradeoff considerations, and can be ignored in the  situation where $m \gg  C \mathbb{E}_{N\sim\nu}[\log(N)N^{-\frac{1}{D_{\chi}+1}}]\gg 1$. 
 Theorem \ref{thm:mainGenError} allows us to think not just about graphs as samples, but also about individual nodes as samples. However, nodes are correlated with their neighbors, and the higher the dimension $D_{\chi}$ is, the larger the neighborhoods are. This is why the dependency on the number of nodes is $N^{-\frac{1}{2(D_{\chi}+1)}}$ and not $N^{-1/2}$. 
  Still, this dependency of the bound on $N$ explains one way in which we train on less graphs than model complexity and still generalize well.  
{Another insight is that the generalization bound becomes smaller the smaller the Lipschitz constants of the message and update functions  (see Remark \ref{Rem34}). This indicates that regularization methods like weight decay promote generalization.}

\begin{table}[t]
\centering
\caption{Comparison of generalization bounds for GNNs. 
We consider the following formula for a generic generalization bound: $GE \leq  m^{-1/2}  A(d,N) B(h) C(L, T) + Em^{-1/2}$, where $m$ is the samples size, 
$T$ is the depth, $L$ is the bound of the Lipschitz constants of the message and update functions, $h$ is the maximum hidden dimension, $d$ is the average node degree and  $N$ is the graphs size and $E$ is a term that does not depend on the model complexity.  } 
\vskip 0.15in
\begin{tabular}{cccc} 
\midrule
   &  $A(d, N)$   & $B(h)$ & $C(L)$ \cr  
\midrule
\makecell{VC-Dimension  \cite{scarselli2018vapnik}}  &   $ \mathcal{O}(\log(N) N)$ & $\mathcal{O}(h^4)$ & -   \\
\makecell{Rademacher \\ Complexity   \cite{pmlr-v119-garg20c} } & $\mathcal{O}(d^{T-1} \sqrt{\log(d^{2T-3})} )$ & $\mathcal{O}(h\sqrt{\log(h)})$ & $\mathcal{O}(L^{2T})$   \\
\makecell{PAC-Bayesian \\ \cite{liao2021a} } & $\mathcal{O}(d^{T-1})$ & $\mathcal{O}(\sqrt{h\log(h)})$ & $\mathcal{O}(L^{2T})$  \\  Ours  & $\mathcal{O}(\E_{N \sim \nu}[\log(N)N^{-\frac{1}{2(D_\chi+1)}}]  )$ & $\mathcal{O}(1)$ & $\mathcal{O}(L^{2T})$      \\
\midrule
\end{tabular}
\label{table:comparison}
\end{table}

\paragraph{Comparison to other generalization bounds in graph classification} 
We compare our generalization bound with other generalization bounds derived by bounding the VC-dimension \cite{scarselli2018vapnik}, the Rademacher complexity \cite{pmlr-v119-garg20c}, and using a PAC-Bayesian approach \cite{liao2021a}. We do not compare with \cite{verma2019stability} since they derive generalization bounds for single-layered MPNNs in node-classification tasks. Hence, the role of depth is unexplored. Furthermore, their bound scales as $\mathcal{O}(\lambda^{2T}_{\mathrm{max}}/m)$, where $T$ is the number of SGD steps and $\lambda_{\mathrm{max}}$ is the largest eigenvalue of the graph Laplacian. Hence, the generalization bound can increase monotonically for increasing $T$ (see \cite{liao2021a} for more details). We summarize the comparison in Table \ref{table:comparison}  and provide more details, specially on the comparability, in Subsection \ref{subsec:comp} of the appendix.

 Our analysis derives a generalization bound on MPNNs  that has essentially the same dependency on the sample size $m$ (up to a logarithmic factor), but does not directly depend on the number of hidden units. 
 We emphasis that our bound  depends on negative moments of the expected node size $N$. In contrast, the VC-dimension based bound \cite{scarselli2018vapnik} scales as $\mathcal{O}(\log(N)N)$, the Rademacher complexity based bound \cite{pmlr-v119-garg20c} scales as  $\mathcal{O}(d^{T-1} \sqrt{\log(d^{2T-3})} )$, and the PAC-Bayesian  approach based bound \cite{liao2021a} scales as $\mathcal{O}(d^{T-1})$, where $d$ denotes the maximum node degree.  

\section{Numerical Experiments}
\label{sec:NumResults}
{We give empirical evaluations of our generalization bound in comparison to the PAC-Bayesian based bound \cite{liao2021a} and the Rademacher complexity based bound \cite{pmlr-v119-garg20c}. We note that the VC dimension bound of \cite{scarselli2018vapnik} is written in O notations and hence cannot be quantitatively evaluated. We experiment on a synthetic dataset of 100K random graphs of 50 nodes, sampled from three different RGMs: the Erdös-Rényi model (ERM) with edge probability $0.4$, a smooth version of a stochastic block model (SBM), based on the kernel  $K(x,y) = \sin(2\pi x )\sin(2\pi y)/2 \pi +0.25$ on $[0,1]^2$, and a geometric graph with kernel $K(x,y) = \exp(-|x-y|^2)$. The corresponding signals are given in Appendix \ref{APP:Dataset}. Each RGM represents one class in three binary classification problems, comparing all pairs of RGMs. 
For the MPNN we consider  GraphSAGE \cite{hamilton2017inductive} with mean aggregation, and number of layers $T=1,2$ or 3, implemented using Pytorch Geometric \cite{Fey/Lenssen/2019}.  We consider a maximal hidden dimension of $128$.
In Appendix \ref{appendix:numericalExp} we give more details and also consider synthetic data sampled from additional RGMs.}


{Our generalization bound becomes smaller the smaller the Lipschitz constants of the message and update functions are. To control the Lipschitz constants, we consider two learning settings. First, we train with weight decay regularization, which decreases the Lipschitz bounds, and second, we train with no regularization. For each setting (each choice of the number of layers and regularization) we train the MPNN, and read the resulting Lipschitz constants of the network. We then plug all constants into our generalization bound formula (see Theorem \ref{thm:reformulatedTheorem2} in the appendix for the full formula), and into the generalization bound formulas of the PAC-Bayes and Rademacher bounds (see Appendix \ref{subsec:comp} for the formulas). The results are reported in Figure \ref{tbl:Gen1}. We observe that our generalization bounds are orders of magnitude smaller than the other works.  In fact, theoretical generalization bounds typically teach us about the asymptotic behavior of generalization, and about the hyperparameters that affect generalization, but rarely give realistic numerical bounds (less than 1) that guarantee generalization. Nevertheless, in one of the scenarios (one layer MPNN)  our theory gives the bounds 0.08911 and 0.13299 (respectively in the two datasets of Figure \ref{tbl:Gen1}),   which guarantees generalization in practice.} 


\begin{figure}[t]
\begin{center}
    \begin{tabular}{M{61mm}M{61mm}}
  \includegraphics[width=.92\linewidth]{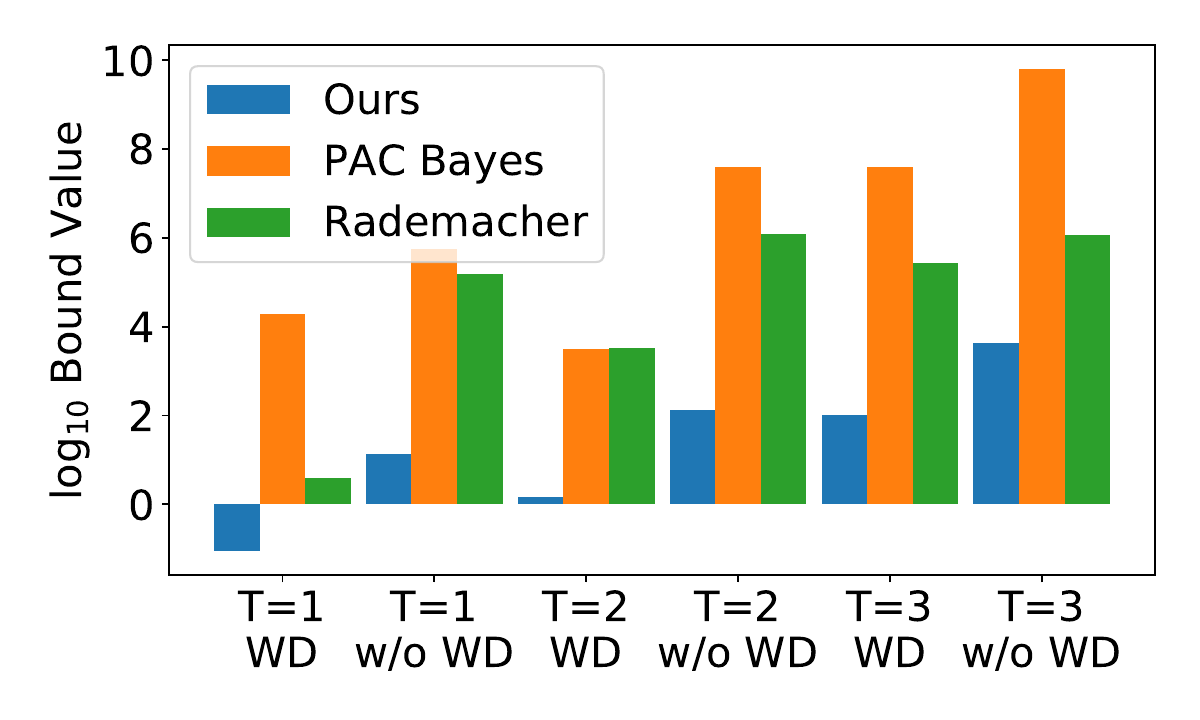} 
  &  
  \includegraphics[width=.92\linewidth]{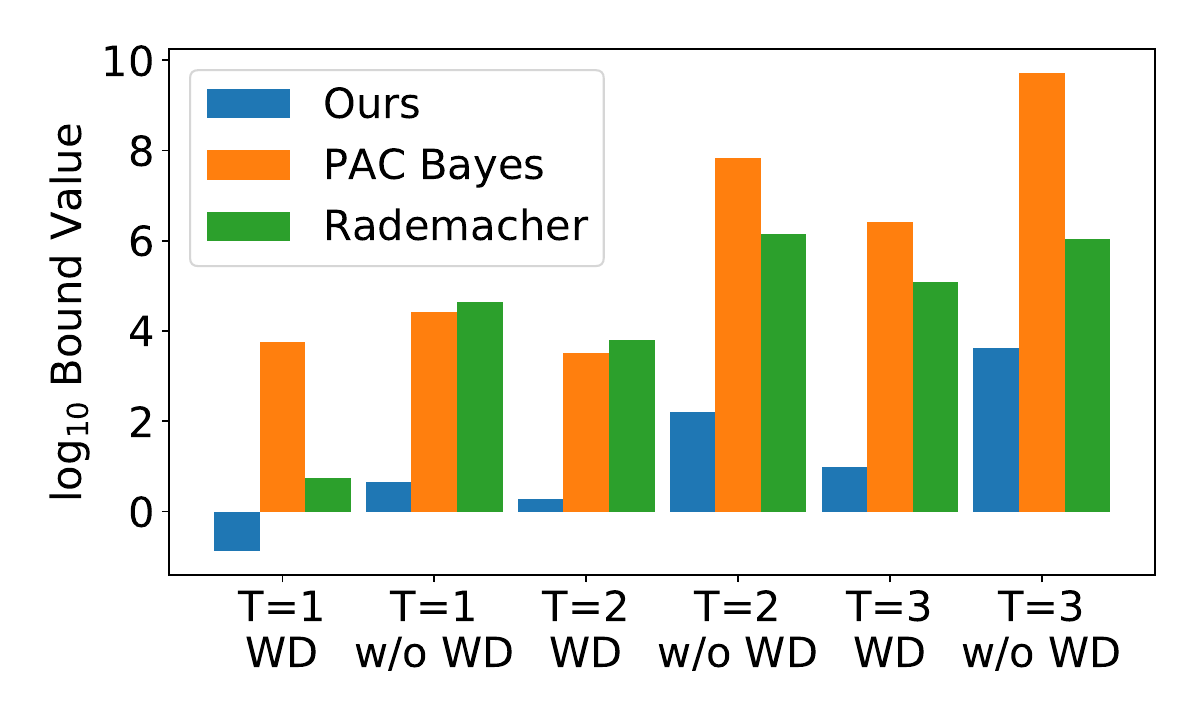}
  \end{tabular}
    \end{center}
    \caption{Generalization bounds given by our theory, PAC-Bayes \cite{liao2021a} and Rademacher complexity \cite{pmlr-v119-garg20c} on a binary classification problem over Erdös-Rényi and SBM (left) and   Erdös-Rényi and a geometric graph (right). Training is done with weight decay (WD) and without weight decay (w/o WD), and on three models with $T=1,2$ and $3$ layers.}
    \label{tbl:Gen1}
\end{figure}

\section{Conclusion}
\label{sec:conclusion}
In this paper we proved that MPNNs with mean aggregation generalize from training to test data in classification tasks, if the graphs are sampled from RGMs that represent the different classes. This follows from the fact that the MPNN on sampled graphs converges to the MPNN on the RGM when the number of nodes goes to infinity. 
Our generalization bounds become smaller the larger the graphs, which gives one explanation to how MPNNs with high complexity can generalize well from a relatively small dataset of large graphs.
We observe two main limitations of our current model. First, the dependency of the generalization bound on the size of the graph $N$ is $\mathcal{O}(N^{-\frac{1}{2(D_{\chi}+1)}})$, which is typically slower than the observed decay in experiments {(See Appendix \ref{Convergence Experiments})}. One potential future direction is to improve this dependency  using a more sophisticated models of the trained network and of the message and update functions. 
 Secondly, our model of the data is somewhat limited. One future direction is to allow deformations of the RGMs, to consider a continuum of RGMs instead of a finite set, and to consider sparse graphs. 

\bibliography{test}

\newcommand{\etalchar}[1]{$^{#1}$}
\begin{thebibliography}{WHZ{\etalchar{+}}18}

\bibitem[BBL{\etalchar{+}}17]{Bronstein_2017}
Michael~M. Bronstein, Joan Bruna, Yann LeCun, Arthur Szlam, and Pierre
  Vandergheynst.
\newblock Geometric deep learning: Going beyond euclidean data.
\newblock {\em IEEE Signal Processing Magazine}, 34(4):18–42, Jul 2017.

\bibitem[DDS16]{10.5555/3045390.3045675}
Hanjun Dai, Bo~Dai, and Le~Song.
\newblock Discriminative embeddings of latent variable models for structured
  data.
\newblock In {\em Proceedings of the 33rd International Conference on
  International Conference on Machine Learning - Volume 48}, ICML'16, page
  2702–2711. JMLR.org, 2016.

\bibitem[FL19]{Fey/Lenssen/2019}
Matthias Fey and Jan~E. Lenssen.
\newblock Fast graph representation learning with {PyTorch Geometric}.
\newblock In {\em ICLR Workshop on Representation Learning on Graphs and
  Manifolds}, 2019.

\bibitem[FML{\etalchar{+}}19]{10.1145/3308558.3313488}
Wenqi Fan, Yao Ma, Qing Li, Yuan He, Eric Zhao, Jiliang Tang, and Dawei Yin.
\newblock Graph neural networks for social recommendation.
\newblock In {\em The World Wide Web Conference}, WWW '19, page 417–426, New
  York, NY, USA, 2019. Association for Computing Machinery.

\bibitem[GBR20]{Gama_2020}
Fernando Gama, Joan Bruna, and Alejandro Ribeiro.
\newblock Stability properties of graph neural networks.
\newblock {\em IEEE Transactions on Signal Processing}, 68:5680–5695, 2020.

\bibitem[GJJ20]{pmlr-v119-garg20c}
Vikas Garg, Stefanie Jegelka, and Tommi Jaakkola.
\newblock Generalization and representational limits of graph neural networks.
\newblock In Hal~Daumé III and Aarti Singh, editors, {\em Proceedings of the
  37th International Conference on Machine Learning}, volume 119 of {\em
  Proceedings of Machine Learning Research}, pages 3419--3430. PMLR, 13--18 Jul
  2020.

\bibitem[HYL17]{hamilton2017inductive}
William~L Hamilton, Rex Ying, and Jure Leskovec.
\newblock Inductive representation learning on large graphs.
\newblock In {\em Proceedings of the 31st International Conference on Neural
  Information Processing Systems}, pages 1025--1035, 2017.

\bibitem[KBV20]{keriven2020convergence}
Nicolas Keriven, Alberto Bietti, and Samuel Vaiter.
\newblock Convergence and stability of graph convolutional networks on large
  random graphs.
\newblock {\em stat}, 1050:23, 2020.

\bibitem[KTD21]{kenlay2021interpretable}
Henry Kenlay, Dorina Thanou, and Xiaowen Dong.
\newblock Interpretable stability bounds for spectral graph filters.
\newblock In {\em Proceedings of the 38th International Conference on Machine
  Learning}. PMLR, 2021.

\bibitem[KW16]{KipfWellingGraphVAE}
Thomas~N Kipf and Max Welling.
\newblock Variational graph auto-encoders.
\newblock {\em arXiv preprint arXiv:1611.07308}, 2016.

\bibitem[LHB{\etalchar{+}}21]{levie2021transferability}
Ron Levie, Wei Huang, Lorenzo Bucci, Michael Bronstein, and Gitta Kutyniok.
\newblock Transferability of spectral graph convolutional neural networks.
\newblock {\em Journal of Machine Learning Research}, 22(272):1--59, 2021.

\bibitem[LIK19]{LevieTransfSpectralGraphFiltersShort}
Ron Levie, Elvin Isufi, and Gitta Kutyniok.
\newblock On the transferability of spectral graph filters.
\newblock In {\em 13th International conference on Sampling Theory and
  Applications (SampTA)}. IEEE, 2019.

\bibitem[Lov67]{lovasz67}
L.~Lov{\'a}sz.
\newblock Operations with structures.
\newblock {\em Acta Mathematica Academiae Scientiarum Hungarica},
  18(3):321--328, 1967.

\bibitem[Lov12]{lovasz}
László Lovász.
\newblock {\em Large networks and graph limits}.
\newblock Colloquium Publications, Budapest, 2012.

\bibitem[LUZ21]{liao2021a}
Renjie Liao, Raquel Urtasun, and Richard Zemel.
\newblock A {\{}pac{\}}-bayesian approach to generalization bounds for graph
  neural networks.
\newblock In {\em International Conference on Learning Representations}, 2021.

\bibitem[MFE{\etalchar{+}}19]{monti2019fake}
Federico Monti, Fabrizio Frasca, Davide Eynard, Damon Mannion, and Michael~M
  Bronstein.
\newblock Fake news detection on social media using geometric deep learning.
\newblock {\em arXiv preprint arXiv:1902.06673}, 2019.

\bibitem[MLK21]{maskey2021transferability}
Sohir Maskey, Ron Levie, and Gitta Kutyniok.
\newblock Transferability of graph neural networks: an extended graphon
  approach.
\newblock {\em arXiv preprint arXiv:2109.10096}, 2021.

\bibitem[MRF{\etalchar{+}}19]{Morris_Ritzert_Fey_Hamilton_Lenssen_Rattan_Grohe_2019}
Christopher Morris, Martin Ritzert, Matthias Fey, William~L. Hamilton, Jan~Eric
  Lenssen, Gaurav Rattan, and Martin Grohe.
\newblock Weisfeiler and leman go neural: Higher-order graph neural networks.
\newblock {\em Proceedings of the AAAI Conference on Artificial Intelligence},
  33(01):4602--4609, Jul. 2019.

\bibitem[Pen03]{RGGPenrose}
Mathew Penrose.
\newblock {\em Random Geometric Graphs}.
\newblock Oxford Scholarship Online, 2003.

\bibitem[RGR21]{9356126}
Luana Ruiz, Fernando Gama, and Alejandro Ribeiro.
\newblock Graph neural networks: Architectures, stability, and transferability.
\newblock {\em Proceedings of the IEEE}, 109(5):660--682, 2021.

\bibitem[RWR21]{ruiz2020graphon}
Luana Ruiz, Zhiyang Wang, and Alejandro Ribeiro.
\newblock Graphon and graph neural network stability.
\newblock In {\em International Conference on Acoustics, Speech and Signal
  Processing (ICASSP)}. IEEE, 2021.

\bibitem[STH18]{scarselli2018vapnik}
Franco Scarselli, Ah~Chung Tsoi, and Markus Hagenbuchner.
\newblock The vapnik--chervonenkis dimension of graph and recursive neural
  networks.
\newblock {\em Neural Networks}, 108:248--259, 2018.

\bibitem[SYS{\etalchar{+}}20]{STOKES2020688}
Jonathan~M. Stokes, Kevin Yang, Kyle Swanson, Wengong Jin, Andres
  Cubillos-Ruiz, Nina~M. Donghia, Craig~R. MacNair, Shawn French, Lindsey~A.
  Carfrae, Zohar Bloom-Ackermann, Victoria~M. Tran, Anush Chiappino-Pepe,
  Ahmed~H. Badran, Ian~W. Andrews, Emma~J. Chory, George~M. Church, Eric~D.
  Brown, Tommi~S. Jaakkola, Regina Barzilay, and James~J. Collins.
\newblock A deep learning approach to antibiotic discovery.
\newblock {\em Cell}, 180(4):688--702.e13, 2020.

\bibitem[Ver18]{vershynin_2018}
Roman Vershynin.
\newblock {\em High-Dimensional Probability: An Introduction with Applications
  in Data Science}.
\newblock Cambridge Series in Statistical and Probabilistic Mathematics.
  Cambridge University Press, 2018.

\bibitem[VW96]{Vaart}
Aad~W. Vaart and Jon~A. Wellner.
\newblock {\em Weak Convergence and Empirical Processes}.
\newblock Springer New York, NY, 1996.

\bibitem[VZ19a]{StabGenofGCNN}
Saurabh Verma and Zhi-Li Zhang.
\newblock Stability and generalization of graph convolutional neural networks.
\newblock In {\em Proceedings of the 25th ACM SIGKDD International Conference
  on Knowledge Discovery \& Data Mining}, pages 1539--1548, 2019.

\bibitem[VZ19b]{verma2019stability}
Saurabh Verma and Zhi-Li Zhang.
\newblock Stability and generalization of graph convolutional neural networks,
  2019.

\bibitem[WHZ{\etalchar{+}}18]{10.1145/3219819.3219869}
Jizhe Wang, Pipei Huang, Huan Zhao, Zhibo Zhang, Binqiang Zhao, and Dik~Lun
  Lee.
\newblock Billion-scale commodity embedding for e-commerce recommendation in
  alibaba.
\newblock In {\em Proceedings of the 24th ACM SIGKDD International Conference
  on Knowledge Discovery \& Data Mining}, KDD '18, page 839–848, New York,
  NY, USA, 2018. Association for Computing Machinery.

\bibitem[WZL{\etalchar{+}}18]{wang2018pixel2mesh}
Nanyang Wang, Yinda Zhang, Zhuwen Li, Yanwei Fu, Wei Liu, and Yu-Gang Jiang.
\newblock Pixel2mesh: Generating 3d mesh models from single rgb images.
\newblock In {\em Proceedings of the European Conference on Computer Vision
  (ECCV)}, pages 52--67, 2018.

\bibitem[XHLJ19]{xu2018how}
Keyulu Xu, Weihua Hu, Jure Leskovec, and Stefanie Jegelka.
\newblock How powerful are graph neural networks?
\newblock In {\em International Conference on Learning Representations}, 2019.

\bibitem[YFM{\etalchar{+}}21]{yehudai2020size}
Gilad Yehudai, Ethan Fetaya, Eli Meirom, Gal Chechik, and Haggai Maron.
\newblock From local structures to size generalization in graph neural
  networks.
\newblock In {\em Proceedings of the 38th International Conference on Machine
  Learning}, volume 139 of {\em Proceedings of Machine Learning Research},
  pages 11975--11986. PMLR, 18--24 Jul 2021.

\bibitem[YHC{\etalchar{+}}18]{10.1145/3219819.3219890}
Rex Ying, Ruining He, Kaifeng Chen, Pong Eksombatchai, William~L. Hamilton, and
  Jure Leskovec.
\newblock Graph convolutional neural networks for web-scale recommender
  systems.
\newblock In {\em Proceedings of the 24th ACM SIGKDD International Conference
  on Knowledge Discovery \& Data Mining}, KDD '18, page 974–983, New York,
  NY, USA, 2018. Association for Computing Machinery.

\end{thebibliography}
\bibliographystyle{alpha}

\appendix
\onecolumn
\section*{Appendix}
\section{Definitions and Notation}
\label{AppendixA}

We denote metric spaces by $(\chi,d)$, where $d:\chi\times\chi\rightarrow \left[0,\infty\right)$ denotes the metric in the space $\chi$. The ball around $x\in\chi$ of radius $\epsilon>0$ is defined to be $B_{\epsilon}(x) = \{y\in \chi\ | \  d(x,y)<\epsilon\}$.
Since, in our analysis, the nodes of the graph are taken as the sample points $X=(X_1,\ldots,X_N)$ in $\chi$, we identify node $i$ of the graph $G$ with the point $X_i$, for every $i=1,\ldots,N$. Moreover, since graph signals $\mathbf{f}=(\mathbf{f}_1,\ldots,\mathbf{f}_N)$ represent mappings from nodes in $V$ to feature values, we denote, by abuse of notation,  $\mathbf{f}(X_i) := \mathbf{f}_i$ for $i= 1,\ldots, N$. 

\begin{definition}[\cite{vershynin_2018}]
\label{def:epsCover}
Let $(\chi,d)$ be a compact metric space.
\begin{enumerate}
    \item 
    The $\varepsilon$-covering numbers of $\chi$, denoted by $\mathcal{C}(\chi, \varepsilon, d)$, is the minimal number of balls of radius $ \varepsilon$ required to cover $\chi$.
    \item
    The  \emph{Minkowski dimension} of $\chi$ is defined to be
    \[\mathrm{dim}(\chi) = \inf\{ D\geq 0 \ |\  \forall \varepsilon \in (0,1) \ \mathcal{C}(\chi, \varepsilon, d) \leq \varepsilon^{-D}\}. \]
\end{enumerate}
\end{definition}

Next, we define various notions of degree.

\begin{definition}
\label{def:degrees}
Let $W:\chi \times \chi \to \left[0,\infty\right)$ be a kernel , $X= (X_1, \ldots, X_N)$  sample points, and $G$ the corresponding sampled graph. 
\begin{enumerate}
    \item We define the \emph{kernel degree} of $W$ at $x\in\chi$ by
\begin{equation}
\label{eq:d_W}
 \mathrm{d}_W(x) = \int_{\chi} W (x,y ) d\P(y).    
\end{equation}
\item Given a point $x\in\chi$ that need not be in $X$, we define the \emph{graph-kernel degree} of $X$  at $x$ by
\begin{equation}
    \label{eq:d_X}
    \mathrm{d}_X(x) = \frac{1}{N} \sum_{i=1}^N W(x, X_i).
\end{equation}
\item The \emph{normalized degree} of $G$ at the node $X_c\in X$ is defined by
\begin{equation}
    \label{eq:d_G}
    \mathrm{d}_G(X_c) = \frac{1}{N} \sum_{i=1}^N W(X_c, X_i).
\end{equation}
\end{enumerate}
\end{definition}

When $x\notin X$, $d_X(x)$ is interpreted as the degree of the node $x$ in the graph $(x,X_1,\ldots,X_n)$ with edge weights sampled from $W$.

Based on the different version of degrees in Definition \ref{def:degrees}, we define the corresponding three versions of mean aggregation.

\begin{definition}
\label{def:contMeanAgg}
Given the kernel $W$, we define the \emph{continuous mean aggregation} of the metric-space message signal $U:\chi\times\chi\rightarrow\mathbb{R}^F$ by 
\[
M_W U = \int_\chi
\frac{W(\cdot, y)}{ \mathrm{d}_W(\cdot) }  U(\cdot, y) d\P(y).
\]
\end{definition}
In Definition \ref{def:contMeanAgg}, $U(x,y)$ represents a message sent from the point $y$ to the point $x$ in the metric space. Given a metric-space signal $f:\chi\rightarrow\mathbb{R}^{F'}$ and a message function $\Phi$, we have 
\[
M_W \Phi (f,f) = \int_\chi
\frac{W(\cdot, y)}{ \mathrm{d}_W(\cdot) }  \Phi\big(f(\cdot), f(y)\big) d\P(y).
\]

\begin{definition}
\label{def:graphkernelMeanAgg}
Let $W$ be a kernel $X= X_1, \ldots, X_N$ sample points.
For a metric-space message signal $U:\chi\times \chi \rightarrow \mathbb{R}^F$,
we define the \emph{graph-kernel mean aggregation} by 
\[
M_X U = \frac{1}{N} \sum_j
\frac{W(\cdot, X_j)}{ \mathrm{d}_X(\cdot) }  U(\cdot, X_j).
\]
\end{definition}

Note that in the definition of $M_X$, messages are sent from graph nodes to arbitrary points in the metric space. Hence, $M_X U: \chi \rightarrow \mathbb{R}^F$ is a metric-space signal. 

\begin{definition}
Let $G$ be a graph with nodes $X =  X_1, \ldots, X_N$.
For  a graph message signal $\mathbf{U}:X\times X \rightarrow \mathbb{R}^F$, where $\mathbf{U}(X_i,X_j)$ represents a message sent from the node $X_j$ to the node $X_i$, we define the \emph{mean aggregation}  by
\[
(M_G \mathbf{U})(X_i) =  \frac{1}{N} \sum_j
\frac{W(X_i, X_j)}{ \mathrm{d}_X(X_i) }  \mathbf{U}(X_i, X_j). 
\]
\end{definition}

Note that $M_G \mathbf{U}:X\rightarrow \mathbb{R}^F$ is a graph signal.

\begin{remark}
Given a graph signal $\mathbf{f}:X \to \mathbb{R}^F$, which can be written as a finite sequence $\mathbf{f} = (\mathbf{f}_i)_i$, and a message function $\Phi:\mathbb{R}^{2F} \rightarrow \mathbb{R}^{H}$, we define 
\[
\Phi( \mathbf{f}, \mathbf{f} ) := \big(\Phi( \mathbf{f}_i, \mathbf{f}_j)\big)_{i,j=1}^N.
\]
\end{remark}
Hence, given a graph signal $\mathbf{f}:X\rightarrow\mathbb{R}^{F}$ and the graph messages $\mathbf{U}(X_i,X_j)=\Phi(\mathbf{f}(X_i), \mathbf{f}(X_j))$, we have
\[
M_G \mathbf{U} = M_G \Phi(\mathbf{f},\mathbf{f}) = \frac{1}{N} \sum_j
\frac{W(\cdot, X_j)}{ \mathrm{d}_X(\cdot) }  \Phi\big(\mathbf{f}(\cdot), \mathbf{f}(X_j)\big).
\]

Next, we define the different norms used in our analysis.

\begin{definition}
$ $
\begin{enumerate}
\item For a vector $\mathbf{z}=(z_1,\ldots,z_F) \in \mathbb{R}^F$, we define as usual
\[
\|\mathbf{z}\|_\infty = \max_{ 1 \leq k \leq F } |z_k|.
\]
    \item 
For a function $g : \chi \to \mathbb{R}^F$,  we define
\[
\|g\|_\infty  = \max_{ 1 \leq k \leq F } \sup_{x \in \chi} \big| \big(g(x)\big)_k \big|,
\]
\item Given a graph with $N$ nodes, we define the norm $\| \mathbf{f} \|_{2;\infty}$ of graph feature maps $\mathbf{f}=(\mathbf{f}_1,\ldots,\mathbf{f}_N) \in \mathbb{R}^{N \times F}$, with feature dimension $F$, as the root mean square over the infinity norms of the node features, i.e.,
\[
\|\mathbf{f}\|_{2; \infty} = \sqrt{\frac{1}{N} \sum_{i=1}^N  \|\mathbf{f}_i\|_{\infty}^2}.
\]
\end{enumerate}
\end{definition}

\begin{definition}
For a metric-space signal $f: \chi \rightarrow \mathbb{R}^F$ and samples $X=(X_1, \ldots, X_N)$ in $\chi$, we define the sampling operator $S^X$ by
\[S^Xf = \big(f(X_i) \big)_{i=1}^N \in \mathbb{R}^{N \times F}. \]
\end{definition}

For a metric-space signal $f:\chi \to \mathbb{R}^F$ and a graph signal $\mathbf{f} \in \mathbb{R}^{N \times F}$, we define the distance ${\rm dist}$ 
as $\d(\mathbf{f}, f) = \| \mathbf{f} - S^Xf \|_{2;\infty}$., i.e,
\begin{equation}
\label{eq:distGraphMetric}
    \d(f, \mathbf{f} ) = \left(\frac{1}{N} \sum_{i=1}^N \|\mathbf{f}_i -  (S^Xf)_i \|^2_\infty\right)^{1/2}.
\end{equation}
Given a MPNN, we define the \emph{formal bias} of the update and message functions  by $\|\Psi^{(l)}(0,0)\|_\infty$ and
$\|\Phi^{(l)}(0,0)\|_\infty$ respectively. Furthermore, we say that a function $\Phi: \mathbb{R}^{F} \rightarrow \mathbb{R}^{H}$ is \emph{Lipschitz continuous} if there exists a $L_\Phi > 0$ such that for every $x,x' \in \mathbb{R}^H$, we have
\[
\|\Phi(x) - \Phi(x')\|_\infty \leq L_\Phi \| x- x'\|_\infty.
\]
Similarly, a function $f:\chi \to \mathbb{R}^F$ is Lipschitz continuous if there exists a $L_f > 0$ such that for every  $x,x' \in \chi$, we have
\[
\|\Phi(x) - \Phi(x')\|_\infty \leq L_f d(x,x').
\]

Next we introduce notations for the mappings between consecutive layers of a MPNN. 
\begin{definition}
\label{def:LayerMapping}
Let $\Theta = ((\Phi^{(l)})_{l=1}^T, (\Psi^{(l)})_{l=1}^T)$  be a MPNN with $T$ layers and  feature dimensions $(F_l)_{l=1}^T$. For $l=1, \ldots, T$, we define the mapping from the $(l-1)$'th layer to the $l$'th layer of the gMPNN as 
\[
\begin{aligned}
 \Lambda^{(l)}_{\Theta_G}: \mathbb{R}^{N  \times F_{l-1}} &\to \mathbb{R}^{N \times F_l} \\
\mathbf{f}^{(l-1)} & \mapsto \mathbf{f}^{(l)}.
\end{aligned}
\]
Similarly, we define $\Lambda_{\Theta_W}^{(l)}$ as  the mapping from the $(l-1)$'th layer to the $l$'th layer of the cMPNN $f^{(l-1)}\mapsto f^{(l)}$.
\end{definition}

Definition \ref{def:LayerMapping} leads to the following,
\[
\Theta^{(T)}_{G } =  \Lambda^{(T)}_{\Theta_G} \circ  \Lambda^{(T-1)}_{\Theta_G}\circ \ldots \circ  \Lambda^{(1)}_{\Theta_G}
\]
and 
\[
\Theta^{(T)}_{W } =  \Lambda^{(T)}_{\Theta_W} \circ  \Lambda^{(T-1)}_{\Theta_W}\circ \ldots \circ  \Lambda^{(1)}_{\Theta_W}
\]

Lastly, we formulate the following assumption on the space $\chi$, the kernel $W$, and the MPNN $\Theta$, to which we will refer often in Appendix \ref{AppendixB}.

\begin{assumption}
\label{ass:graphon}
Let $(\chi,d)$ be a metric space and
 $W: \chi \times \chi \rightarrow [0, \infty)$. 
 Let $\Theta$ be a MPNN with message and update  functions $\Phi^{(l)}: \mathbb{R}^{2F_l} \rightarrow \mathbb{R}^{H_l}$ and  $\Psi^{(l)}: \mathbb{R}^{F_l+H_l} \rightarrow \mathbb{R}^{F_{l+1}}$, $l=1,\ldots,T-1$. 
\begin{enumerate}
\item \label{ass:graphon1}
The space $\chi$ is compact, and there exist $D_\chi, C_\chi \geq 0 $ such that $\mathcal{C}(\chi, \varepsilon, d) \leq C_\chi \varepsilon^{-D_\chi} $ for every $\varepsilon>0$. \footnote{The  Minkowski dimension  $\mathrm{dim}(\chi)$ is a lower bound for all such possible $D_\chi$.}  
\item \label{ass:diamChi}
The diameter of $\chi$ is bounded by 1. Namely, $\mathrm{diam}(\chi):=\sup_{x,y\in\chi}d(x,y) \leq 1$.
    \item \label{ass:bddKernel}The kernel satisfies $\|W\|_\infty < \infty$. 
    \item\label{ass:graphon11}
    For every $y \in \chi$, the function $W(\cdot, y)$ is Lipschitz continuous (with respect to its first variable) with Lipschitz constant $L_W$. 
    \item \label{ass:GraphonLip2nd} For every $x \in \chi$, the function $W(x,\cdot)$ is Lipschitz continuous (with respect to its second variable) with Lipschitz constant $L_W$. 
    \item\label{ass:graphon12}There exists a constant $\cmin > 0$ such that for every $x \in \chi$, we have $d_W(x) \geq \cmin$.
    \item\label{ass:graphon4}  For every $l=1,\ldots,T$, the message function $\Phi^{(l)}$ and update function $\Psi^{(l)}$ are Lipschitz continuous with Lipschitz constants $L_{\Phi^{(l)}}$ and $L_{\Psi^{(l)}}$ respectively.
    \item\label{Ass:KernelDiagBounded}
There exists a constant ${\rm W}_{\mathrm{diag}}>0$ such that for every $x \in \chi$, we have
$
W(x,x) \geq {\rm W}_{\mathrm{diag}}>0
$. 

\end{enumerate}
\end{assumption}

\section{Convergence Analysis}
\label{AppendixB}
In this section we provide the proofs for Theorem \ref{thm:MainInProb} from Section \ref{sec:Main}.

\subsection{Preparation}
This section is a preparation for the upcoming proof of Theorem \ref{thm:MainInProb} from Section \ref{sec:Main}. An important goal of this section is to formulate and prove Lemma \ref{lemma:C2}, which provides a uniform concentration of measure of the uniform error between the continuous mean aggregation $M_W$ and the graph-kernel mean aggregation $M_X$. We then show in Corollary \ref{cor:Uniform3} that this uniform bound is preserved by application of an update function. 
We begin with the following concentration of error lemma which is a slight modification of \cite[Lemma 4]{keriven2020convergence}, and can be derived directly from \cite[Lemma 4]{keriven2020convergence}, by using the assumption $\mathcal{C}(\chi, \varepsilon, d) \leq C_\chi \varepsilon^{-D_\chi}$ instead of $\mathcal{C}(\chi, \varepsilon, d) \leq   \varepsilon^{-\dim(\chi)} $. 

\begin{lemma}[Lemma 4, \cite{keriven2020convergence}.]
\label{lemma:kerivenLemma4}
Let $(\chi,d, \P) $ be a metric-measure space and
$W$ be a kernel s.t.  Assumptions \ref{ass:graphon}.\ref{ass:graphon1}-\ref{ass:graphon11}. are satisfied. Consider a metric-space signal  $f: \chi \to \mathbb{R}$ with $\|f\|_\infty < \infty$. Suppose that $X_1, \ldots, X_N$ are drawn i.i.d. from $\P$ on $\chi$ and let $p \in (0,1)$. Then, with probability at least $1-p$, we have 
\begin{equation*}
   \begin{aligned}
& \left\|
\frac{1}{N} \sum_{i=1}^N W(\cdot,X_i)f(X_i) - \int_\chi W(\cdot, x) f(x) d\P(x)\right\|_\infty
 \\ & \leq  
\frac{\|f\|_\infty \Big(\zeta \cl (\sqrt{\log (C_\chi)} +  \sqrt{D_\chi}) + (\sqrt{2}\cmax+\zeta \cl) \sqrt{\log2/p} \Big) }{ \sqrt{N}},
\end{aligned}
\end{equation*}
where 
\begin{equation}
    \label{eq:zeta}
    \zeta  := \frac{2}{\sqrt{2}}e\Big(\frac{2}{\ln(2)} +1 \Big)\frac{1}{\sqrt{\ln(2)}} C
\end{equation}
 and $C$ is the universal constant from Dudley's inequality (see Theorem 8.1.6 \cite{vershynin_2018}).
\end{lemma}

As a consequence of Lemma \ref{lemma:kerivenLemma4}, we can derive a sufficient condition on the sample size $N$ which ensures that the graph-kernel degrees  
 are uniformly bounded from below.

\begin{lemma}
\label{lemma:LowerBoundDegree}
Let $(\chi,d, \P) $ be a metric-measure space and
$W$ be a kernel s.t.  Assumptions \ref{ass:graphon}.\ref{ass:graphon1}-\ref{ass:graphon11}. and \ref{ass:graphon}.\ref{ass:graphon12}. are satisfied. 
Suppose that $X_1, \ldots, X_N$ are drawn i.i.d. from $\P$ on $\chi$ and let $p \in (0,1)$.
Let
\begin{equation}
\label{eq:largeN}
    \sqrt{N} \geq  2 \Big(\zeta   \frac{\cl}{\cmin} \big(\sqrt{\log (C_\chi)} +  \sqrt{D_\chi}\big) +
    \frac{\sqrt{2} \cmax + \zeta \cl}{ \cmin } \sqrt{\log 2/p}\Big),
\end{equation}
where $\zeta$ is defined in (\ref{eq:zeta}). Then, with probability at least $1-p$ the following two inequalities hold: For every $x\in\chi$,
\begin{equation}
    \label{eq:d_XLowerBound}
\mathrm{d}_X( x) \geq \frac{\cmin}{2}
\end{equation}
and 
\begin{equation}
\label{eq:lemmab1-1}
     \begin{aligned}
& \left\|
\frac{1}{N} \sum_{i=1}^N W(\cdot,X_i)f(X_i) - \int_\chi W(\cdot, x) f(x) d\P(x)\right\|_\infty
 \\ & \leq  
\frac{\|f\|_\infty \Big(\zeta \cl (\sqrt{\log (C_\chi)} +  \sqrt{D_\chi}) + (\sqrt{2}\cmax+\zeta \cl) \sqrt{\log2/p} \Big) }{ \sqrt{N}}.
\end{aligned}
\end{equation}
\end{lemma}
\begin{proof}
By Lemma \ref{lemma:kerivenLemma4}, with $f=1$, with probability at least $1-p$ we have
\[\|\mathrm{d}_X(\cdot) - \mathrm{d}_W(\cdot)\|_\infty  \leq   
\frac{ \Big(\zeta\cl \big(\sqrt{\log (C_\chi)} +  \sqrt{D_\chi}\big) + \big(\sqrt{2}\cmax+\zeta\cl\big) \sqrt{\log2/p} \Big) }{ \sqrt{N}}.
\]
By using the lower bound (\ref{eq:largeN}) of $\sqrt{N}$, we have
$\|\mathrm{d}_X(\cdot) - \mathrm{d}_W(\cdot)\|_\infty \leq \frac{\cmin}{2}$. Let $x \in \chi$. By Assumption \ref{ass:graphon}.\ref{ass:graphon12}, we have $|\mathrm{d}_W(x)| \geq \cmin$, hence $|\mathrm{d}_X(x)| \geq \mathrm{\cmin}/2$.
\end{proof}

The following lemma is a uniform concentration of measure of the Monte Carlo approximation of Lipschitz functions. Related results about uniform law of large numbers for Lipschitz functions can be found in \cite[Chapter 8.2]{vershynin_2018}.  Our result holds for general metric spaces with finite  Minkowski dimension. 
\begin{lemma}
\label{lemma:NumAnaApprox}
Let $(\chi, d, \mu)$ be a metric-measure space s.t. Assumption \ref{ass:graphon}.\ref{ass:graphon1}.  is satisfied. Suppose that $X_1, \ldots, X_N$ are drawn i.i.d. from $\P$ on $\chi$. For every $p > 0$, there exists an event $\mathcal{E}_{\rm Lip}^p \subset \chi^N$  regarding the choice of $(X_1,\ldots, X_N)\in \chi^N$, with probability $\P^N(\mathcal{E}_{\rm Lip}^p) \geq 1-p$,  such that the following uniform bound is satisfied:  For every Lipschitz continuous function $F: \chi \to \mathbb{R}^F$ 
 with Lipschitz constant  $L_F$, we have 
\[
\begin{aligned}
 &  \left\| \frac{1}{N} \sum_{i=1}^N F(X_i) - \int_\chi F(x) d\mu(x)\right \|_\infty \\
 & \leq N^{-\frac{1}{2(D_\chi + 1)}}
\left(
2 L_F  + \frac{C_\chi}{\sqrt{2}} \| F \|_\infty \sqrt{ \log(C_\chi) + \frac{D_\chi}{2(D_\chi + 1)} \log(N) + \log(2/p)}
\right).
\end{aligned}
\]

\end{lemma}

For completion, we provide a proof of Lemma \ref{lemma:NumAnaApprox}.

\begin{proof}
Let $r > 0$.
By Assumption \ref{ass:graphon}.\ref{ass:graphon1}, there exists an open covering $(B_j)_{j\in \mathcal{J}}$ of $\chi$  by a family of  balls with radius $r$ 
such that $|\mathcal{J}| \leq C_\chi r^{-D_\chi}$.
For $j = 2, \ldots, |\mathcal{J}|$, we define $I_j := B_j \setminus \cup_{i < j} B_i$, and define $I_1=B_1$. Hence, $(I_j)_{j \in \mathcal{J}}$ is a family of measurable sets  such that $I_j \cap I_i = \emptyset$ for all $i\neq j \in \mathcal{J}$,  $\bigcup_{j \in \mathcal{J}} I_j = \chi$, and $\mathrm{diam}(I_j) \leq 2r$ for all $j \in \mathcal{J}$, where by convention $\mathrm{diam}(\emptyset)=0$. For each $j \in \mathcal{J}$, let $z_j$ be the center of the ball $B_j$.  

 Next, we compute a concentration of error bound on the difference between the measure of $I_j$ and its Monte Carlo approximation, which is uniform in $j\in\mathcal{J}$.  
Let $j \in \mathcal{J}$ and $q \in (0,1)$. By Hoeffding's inequality, there is an event $\mathcal{E}_j^q$ with probability  $\mu(\mathcal{E}_j)\geq 1-q$, in which
\begin{equation}
\label{eq:stepfctEvent1}
\left\| \frac{1}{N} \sum_{i=1}^N \mathbbm{1}_{I_j}(X_i) - \mu(I_k)\right\|_\infty \leq \frac{1}{\sqrt{2}}\frac{\sqrt{ \log(2/q)}}{\sqrt{N}}.
\end{equation}
Consider the event 
\[\mathcal{E}_{\rm Lip}^{|\mathcal{J}|q} = \bigcap_{j=1}^{|\mathcal{J}|}\mathcal{E}_j^q,\]
with probability  $\mu^N(\mathcal{E}_{\rm Lip}^{|\mathcal{J}|q})\geq 1- |\mathcal{J}|q $.
In this event,  (\ref{eq:stepfctEvent1}) holds for all $j\in\mathcal{J}$. We change the failure probability variable $p = |\mathcal{J}|q$, and denote $\mathcal{E}_{\rm Lip}^{p}= \mathcal{E}_{\rm Lip}^{|\mathcal{J}|q}$.

Next we bound uniformly the Monte Carlo approximation error of the integral of bounded Lipschitz continuous functions $F:\chi\rightarrow \mathbb{R}^F$.
Let $F: \chi \to \mathbb{R}^F$ be a bounded Lipschitz continuous function with Lipschitz constant $L_F$. We define the step function
\[
F^r(y)=  \sum_{j \in \mathcal{J}} F(z_j) \mathbbm{1}_{I_j}(y).
\]
Then, 
\begin{equation}
\label{eq:123uniformBound2}
\begin{aligned}
 \left\| \frac{1}{N }\sum_{i=1}^N F(X_i)
 - \int_\chi F(y) d\P(y)\right\|_\infty
 & \leq \left\|\frac{1}{N }\sum_{i=1}^N F(X_i) -  \frac{1}{N }\sum_{i=1}^N F^r(X_i)\right\|_\infty
  \\
 & + \left\|\frac{1}{N} \sum_{i=1}^N F^r(X_i) - \int_\chi F^r(y) d\mu(y) \right\|_\infty
 \\
 & + \left\| \int_\chi F^r(y)  d\mu(y) - \int_\chi F(y)  d\P(y) \right\|_\infty
 \\
 & =: (1) + (2) + (3).
\end{aligned}
\end{equation}

To bound (1), we define for each $X_i$ the unique index $j_i \in \mathcal{J}$ s.t. $X_i \in I_{j_i}$. We calculate,
\begin{align*}
 \left\|\frac{1}{N }\sum_{i=1}^N F(X_i) -  \frac{1}{N }\sum_{i=1}^N F^r(X_i)\right\|_\infty 
\leq & \frac{1}{N}\sum_{i=1}^N\left\| F(X_i)  - \sum_{j \in \mathcal{J}} F(z_j) \mathbbm{1}_{I_j}(X_i)\right\|_\infty \\
= & \frac{1}{N}\sum_{i=1}^N\left\|F(X_i)  -  F(z_{j_i})\right\|_\infty\\
\leq & r  L_F.
\end{align*}

We proceed by bounding (2). In the event of $\mathcal{E}_{\rm Lip}^p$, which holds with probability at least $1- p$, equation (\ref{eq:stepfctEvent1}) holds for all $j\in\mathcal{J}$. In this event, we get 
\[
\begin{aligned}
 \left\|\frac{1}{N} \sum_{i=1}^N F^r(X_i) - \int_\chi F^r(y) d\mu(y) \right\|_\infty 
& =  \left\|\sum_{j \in \mathcal{J}} \left( \frac{1}{N} \sum_{i=1}^N F(z_j) \mathbbm{1}_{I_j}(X_i) - \int_{I_j}F(z_j)   dy  \right) \right\|_\infty \\
& \leq \sum_{j \in \mathcal{J}} \|F\|_\infty \left| \frac{1}{N} \sum_{i=1}^N \mathbbm{1}_{I_j}(X_i) - \mu(I_j) \right| \\
& \leq |\mathcal{J}| \| F\|_\infty \frac{1}{\sqrt{2}}\frac{\sqrt{ \log(2|\mathcal{J}|/p)}}{\sqrt{N}}.
 \end{aligned}
\]
Recall that $|\mathcal{J}| \leq C_\chi r^{-D_\chi}$. 
Then, with probability at least $1-p$ 
\begin{align*}
& \left\|\frac{1}{N} \sum_{i=1}^N F^r(X_i) - \int_\chi F^r(y) d\mu(y) \right\|_\infty \\
& \leq C_\chi r^{-D_\chi} \| F\|_\infty \frac{1}{\sqrt{2}}\frac{\sqrt{\log(C_\chi)-D_\chi \log(r)+ \log(2/p)}}{\sqrt{N}}.
\end{align*}

To bound (3), we calculate
\[
\begin{aligned}
 \left\| \int_\chi F^r(y)  d\mu(y) - \int_\chi F(y)  d\P(y) \right\|_\infty 
& = \left\| \int_\chi \sum_{j \in \mathcal{J}} F(z_j) \mathbbm{1}_{I_j} d\mu(y) - \int_\chi F(y) d\P(y) \right\|_\infty \\
& \leq \sum_{j \in \mathcal{J}} \int_{I_j} \left\|F(z_j) - F(y)\right\|_\infty d\P(y) \\ 
& \leq r L_F.
\end{aligned}
\]

By plugging the bounds of $(1), (2)$ and $(3)$ into (\ref{eq:123uniformBound2}), we get
\[
\begin{aligned}
\left\| \frac{1}{N }\sum_{i=1}^N F(X_i)
 - \int_\chi F(y) d\P(y)\right\|_\infty
& \leq 
2r L_F
+ C_\chi r^{-D_\chi} \|F \|_\infty  \frac{1}{\sqrt{2}}\frac{\sqrt{\log(C_\chi)-D_\chi \log(r)+ \log(2/p)}}{\sqrt{N}}. 
\end{aligned}
\]

Lastly, choosing $r = N^{-\frac{1}{2(D_\chi+1)}}$ gives us an overall error of 
\[
\begin{aligned}
& \left\| \frac{1}{N }\sum_{i=1}^N F(X_i)
 - \int_\chi F(y) d\P(y)\right\|_\infty \\
& \leq N^{-\frac{1}{2(D_\chi+1)}}\Bigg(
2 L_F
+ C_\chi \|F \|_\infty \frac{1}{\sqrt{2}} \sqrt{\log(C_\chi) + \frac{D_\chi}{2(D_\chi+1)} \log(N) + \log(2/p)}  \Bigg)
\end{aligned}
\]
Since the event $\mathcal{E}_{\rm Lip}^p$ is independent of the choice of $F:\chi \to \mathbb{R}^F$, the proof is finished.
\end{proof}

The next lemma is based on Lemma \ref{lemma:NumAnaApprox}, and provides a uniform concentration of measure on
the $L^\infty$-error between a non-normalized version of the kernel mean aggregation from Definition
\ref{def:contMeanAgg} and a non-normalized version of the graph-kernel mean aggregation from Definition \ref{def:graphkernelMeanAgg}.

\begin{lemma}
\label{lemma1:Uniform}
Let $(\chi,d, \P) $ be a metric-measure space and
$W$ be a kernel s.t.  Assumptions \ref{ass:graphon}.\ref{ass:graphon1}-\ref{ass:bddKernel} and \ref{ass:graphon}.\ref{ass:GraphonLip2nd}. are satisfied. Let $p \in (0,1)$.  
Suppose that $X_1, \ldots, X_N$ are drawn i.i.d. from $\P$ on $\chi$ such that $(X_1, \ldots, X_N)\in\mathcal{E}_{\rm Lip}^p$, where the event $\mathcal{E}_{\rm Lip}^p$ is defined in Lemma \ref{lemma:NumAnaApprox}.
Then,  for every $x \in \chi $,  $f:\chi \to \mathbb{R}^{F}$ with  Lipschitz constant $L_f$, and $\Phi: \mathbb{R}^{2F} \to \mathbb{R}^{H}$ with   Lipschitz constant $L_\Phi$, 
we have 
\begin{equation}
\begin{aligned}
& \left\| \frac{1}{N }\sum_{i=1}^N W(x,X_i)\Phi\big(f(x),  f(X_i)\big)
 - \int_\chi W(x, y) \Phi\big(f(x),  f(y)\big) d\P(y) \right\|_\infty \\  & \leq N^{-\frac{1}{2(D_\chi+1)}}\Bigg(
2 \Big(
\|W\|_\infty L_\Phi L_f + \LipW \big( L_\Phi \|f\|_\infty + \|\Phi(0,0)\|_\infty \big)  \Big)
 \\
& + C_\chi \Big(\|W\|_\infty \big( L_\Phi \|f\|_\infty + \|\Phi(0,0)\|_\infty \big) \Big)
\frac{1}{\sqrt{2}} \sqrt{\log(C_\chi) + \frac{D_\chi}{2(D_\chi+1)} \log(N) + \log(2/p)}  \Bigg).
\end{aligned}
\end{equation}
\end{lemma}

\begin{proof}

For 
 any $x \in \chi $, $f:\chi \to \mathbb{R}^{F}$ and $\Phi: \mathbb{R}^{2F} \to \mathbb{R}^{H}$, we define the random variable  
\[
Y_{x; \Phi} = \frac{1}{N }\sum_{i=1}^N W(x,X_i)\Phi\big(f(x),  f(X_i)\big)
 - \int_\chi W(x, y) \Phi\big(f(x),  f(y)\big) d\P(y)
\]
on the sample space $\chi^N$. 
Applying Lemma \ref{lemma:NumAnaApprox} on the integrand $F_x(y) := W(x,y)\Phi\big(f(x), f(y)\big)$, uniformly on the choice of the parameter $x\in \chi$, yields in the event  $\mathcal{E}_{\rm Lip}^p$:
\begin{equation}
\| Y_{x;\Phi}\|_\infty \leq N^{-\frac{1}{2(D_\chi+1)}}\Bigg(
2 L_{F_x}
+ C_\chi \|F_x \|_\infty \frac{1}{\sqrt{2}} \sqrt{\log(C_\chi) + \frac{D_\chi}{2(D_\chi+1)} \log(N) + \log(2/p)}  \Bigg).
\end{equation}

So it remains to calculate the Lipschitz constant and the infinity-norm of $F_x$.
For this, calculate for $y,y' \in \chi$
\begin{align*}
\|F_x(y) - F_x(y')\|_\infty =& \|W(x,y)\Phi\big(f(x), f(y)\big) - W(x,y')\Phi\big(f(x), f(y')\big)\|_\infty\\
\leq&\|W(x,y)\Phi\big(f(x), f(y)\big) - W(x,y)\Phi\big(f(x), f(y')\big)\|_\infty\\
+ &\|W(x,y)\Phi\big(f(x), f(y')\big) - W(x,y')\Phi\big(f(x), f(y')\big)\|_\infty\\
\leq & \big(\|W\|_\infty  L_\Phi L_f + L_W(L_\Phi \|f\|_\infty + \|\Phi(0,0)\|_\infty) \big) d(y,y')
\end{align*}
and 
\begin{align*}
    \|F_x(\cdot)\|_\infty & = \| W(x,\cdot) \Phi\big( f(x), f(\cdot) \big) \|_\infty \\
    & \leq \|W\|_\infty (L_\Phi \|f\|_\infty + \|\Phi(0,0)\|_\infty).
\end{align*}
\end{proof}

The next lemma provides a uniform concentration of measure bound on the error between the
graph-kernel mean aggregation $M_X$ and the continuous mean aggregation $M_W$ .

\begin{lemma}
\label{lemma:C2}
Let $(\chi,d, \P) $ be a metric-measure space and
$W$ be a kernel s.t.  Assumptions   \ref{ass:graphon}.\ref{ass:graphon1}-\ref{ass:graphon12}. are satisfied.  Let $N \in \mathbb{N}$ satisfy (\ref{eq:largeN}). Let  $\mathcal{E}_{\rm Lip}^p$ be the event  defined in Lemma \ref{lemma:NumAnaApprox}.
There exists an event $\mathcal{F}_{\rm Lip}^p \subset \mathcal{E}_{\rm Lip}^p$ regarding the choice of i.i.d $X_1,\ldots, X_N$ from  $\P$ in $\chi$, with probability $\P(\mathcal{F}_{\rm Lip}^p) \geq 1-2p$, such that  condition (\ref{eq:d_XLowerBound}) together with (\ref{eq:lemmab5-12}) below are satisfied: for every $f:\chi \to \mathbb{R}^{F}$ with  Lipschitz constant $L_f$  and $\Phi: \mathbb{R}^{2F} \to \mathbb{R}^{H}$ with   Lipschitz constant $L_\Phi$ 
\begin{equation}
\begin{aligned}
\label{eq:lemmab5-12}
&  \| (M_X   - M_W) \big(\Phi (f,f)\big)  \|_\infty     \leq 4 \frac{\varepsilon_1}{\sqrt{N}\mathrm{d}_{min}^2} \|W\|_\infty(L_\Phi \|f\|_\infty + \|\Phi(0,0)\|_\infty) \\ & +  N^{-\frac{1}{2(D_\chi+1)}}\Bigg(
2\Big(
\frac{\|W\|_\infty}{\mathrm{d}_{min}} L_\Phi L_f +   \frac{\LipW}{\mathrm{d}_{min}}   \big( L_\Phi \|f\|_\infty + \|\Phi(0,0)\|_\infty \big)  \Big)
 \\
& + C_\chi \Big(\frac{\|W\|_\infty}{\mathrm{d}_{min}} \big( L_\Phi \|f\|_\infty + \|\Phi(0,0)\|_\infty \big) \Big)
\frac{1}{\sqrt{2}}\sqrt{\log(C_\chi) + \frac{D_\chi}{2(D_\chi+1)} \log(N) + \log(2/p)}  \Bigg),
\end{aligned}
\end{equation}
where 
\begin{equation}
    \label{eq:eps1-2}
   \varepsilon_1 =  \cl \big(\sqrt{\log (C_\chi)} +  \sqrt{D_\chi}\big) +\big(\sqrt{2}\cmax + \cl\big) \sqrt{\log 2/p}.
\end{equation}  
  \end{lemma}
  
\begin{proof}
By Lemma \ref{lemma:LowerBoundDegree}, we have  with probability at least $1-p$ 
\begin{equation}
\label{eq:lemmaC2Event1}
    \begin{aligned}
\|\mathrm{d}_X - \mathrm{d}_W\|_\infty & \leq \frac{\varepsilon_1}{\sqrt{N}}  =
\zeta   \frac{\cl \big(\sqrt{\log (C_\chi)} +  \sqrt{D_\chi}\big) +\big(\sqrt{2}\cmax + \cl\big) \sqrt{\log 2/p}}{\sqrt{N}} \\
& \leq 
\frac{\cmin}{2},
\end{aligned}
\end{equation}
where the second inequality follows from $(\ref{eq:largeN})$. Furthermore, in the same event we have
\[
|\mathrm{d}_X(x)|_\infty \geq \frac{\cmin}{2}
\]
for all $x \in \chi$. 
Moreover,  $|\mathrm{d}_W(x)|_\infty \geq \cmin$ by Assumption \ref{ass:graphon}.\ref{ass:graphon12}. Hence, for all $x \in \chi$, we have
\begin{equation}
\label{eq:lemmaC2-2}
    \begin{aligned}
\left| \frac{1}{\mathrm{d}_X(x)} - \frac{1}{\mathrm{d}_W(x)} \right| & = \frac{|\mathrm{d}_W(x) -\mathrm{d}_X(x)  |}{| \mathrm{d}_X(x) \mathrm{d}_W(x) |} \\
& \leq 4 \frac{\varepsilon_1}{\sqrt{N} \cmin^2}.
\end{aligned}
\end{equation} 
 
Denote that intersection of $\mathcal{E}_{\rm Lip}^p$ and the event in which (\ref{eq:lemmaC2Event1})  occur by $\mathcal{F}_{\rm Lip}^p$.  Let $(X_1,\ldots,X_N)$ be i.i.d samples in $\mathcal{F}_{\rm Lip}^p$.
Define
$\tilde{W}(x,y) = \frac{W(x,y)}{\mathrm{d}_W(x)}$. Next
we apply Lemma \ref{lemma1:Uniform} on the kernel $\tilde{W}$.  
 For this, note that for $x \in \chi$ the kernel
$\tilde{W}(x,\cdot)$ is Lipschitz continuous (with respect to the second variable) with Lipschitz constant $  L_{\tilde{W}} = \frac{\LipW}{\cmin} $, since for $y,y' \in \chi$, we have 
\[
\begin{aligned}
\left|\frac{W(x,y)}{\mathrm{d}_W(x)} - \frac{W(x,y')}{\mathrm{d}_W(x)}\right| & \leq \frac{\LipW}{\mathrm{d}_{min}} d(y,y').
\end{aligned}
\] 
Moreover, for all $y \in \chi$ we have $\|\tilde{W}(\cdot, y)\|_\infty \leq \frac{\cmax}{\cmin}$.

 Then, we use Lemma \ref{lemma1:Uniform} to obtain, for every $f:\chi \to \mathbb{R}^F$ and  $\Phi: \mathbb{R}^{2F}\to \mathbb{R}^H$ as specified in the lemma, 
\begin{equation}
\label{eq:lemmaC2-1}
\begin{aligned}
& \left\|\frac{1}{N }\sum_{i=1}^N \tilde{W}(\cdot ,X_i)\Phi\big(f(\cdot ),  f(X_i)\big)
 - \int_\chi \tilde{W}(\cdot , y) \Phi\big(f(\cdot ),  f(y)\big) d\P(y)\right\|_\infty 
 \\ & \leq N^{-\frac{1}{2(D_\chi+1)}}\Bigg(
2\Big(
\|\tilde{W}\|_\infty L_\Phi L_f +  L_{\tilde{W}} \big( L_\Phi \|f\|_\infty + \|\Phi(0,0)\|_\infty \big)  \Big)
 \\
& + C_\chi \Big(\|\tilde{W}\|_\infty \big( L_\Phi \|f\|_\infty + \|\Phi(0,0)\|_\infty \big) \Big)
\frac{1}{\sqrt{2}}\sqrt{\log(C_\chi) + \frac{D_\chi}{2(D_\chi+1)} \log(N) + \log(2/p)}  \Bigg) \\
& \leq N^{-\frac{1}{2(D_\chi+1)}}\Bigg(
2\Big(
\frac{\cmax}{\cmin} L_\Phi L_f + \frac{\LipW}{\cmin} \big( L_\Phi \|f\|_\infty + \|\Phi(0,0)\|_\infty \big)  \Big)
 \\
& + C_\chi \Big(\frac{\cmax}{\cmin} \big( L_\Phi \|f\|_\infty + \|\Phi(0,0)\|_\infty \big) \Big)
\frac{1}{\sqrt{2}} \sqrt{\log(C_\chi) + \frac{D_\chi}{2(D_\chi+1)} \log(N) + \log(2/p)}  \Bigg).
\end{aligned}
\end{equation}
Then, by (\ref{eq:lemmaC2-2}) and (\ref{eq:lemmaC2-1}), for every  $f:\chi \to \mathbb{R}^F$ and  $\Phi: \mathbb{R}^{2F}\to \mathbb{R}^H$ as specified in the lemma,
\[
\begin{aligned}
& \left\| (M_X - M_W)\Phi(f,f)  \right\|_\infty 
 \\
& =\left\|  \frac{1}{N} \sum_{i=1}^N \frac{W(\cdot, X_i)}{\mathrm{d}_X(\cdot)}\Phi\big(f(\cdot),f(X_i)\big) - \int_\chi \frac{W(\cdot, x)}{\mathrm{d}_W(\cdot)}\Phi\big(f(\cdot),f(x)\big) d\P(x) \right\|_\infty \\
  & \leq   \frac{1}{N}
  \sum_{i=1}^N \big\|W(x,X_i) \Phi\big(f(\cdot),f(X_i)\big)\big\|_\infty  \left\|
  \frac{1}{\mathrm{d}_X(\cdot)} - \frac{1}{\mathrm{d}_W(\cdot)}
  \right\|_\infty \\
  & + \left\| \frac{1}{N} \sum_{i=1}^N \tilde{W}(\cdot, X_i)\Phi\big(f(\cdot),f(X_i)\big)  - \int_\chi  \tilde{W}(\cdot, x)\Phi\big(f(\cdot),f(x)\big)  d\P(x)  \right\|_\infty
\\
& \leq 4 \frac{\varepsilon_1}{\sqrt{N}\mathrm{d}_{min}^2} \|W\|_\infty(L_\Phi \|f\|_\infty + \|\Phi(0,0)\|_\infty)\\
&+  N^{-\frac{1}{2(D_\chi+1)}}\Bigg(
2\Big(
\frac{\|W\|_\infty}{\mathrm{d}_{min}} L_\Phi L_f +  \frac{\LipW}{\mathrm{d}_{min}}  \big( L_\Phi\|f\|_\infty + \|\Phi(0,0)\|_\infty \big)  \Big)
 \\
& + C_\chi \Big(\frac{\|W\|_\infty}{\mathrm{d}_{min}} \big( L_\Phi \|f\|_\infty + \|\Phi(0,0)\|_\infty \big) \Big)
\frac{1}{\sqrt{2}} \sqrt{\log(C_\chi) + \frac{D_\chi}{2(D_\chi+1)} \log(N) + \log(2/p)}  \Bigg).
\end{aligned}
\]
\end{proof}

The next corollary shows that Lemma \ref{lemma:C2} is preserved by the application of an update function.

\begin{corollary}
\label{cor:Uniform3}
Let $(\chi,d, \P) $ be a metric-measure space and
$W$ be a kernel s.t.  Assumptions \ref{ass:graphon}.\ref{ass:graphon1}-\ref{ass:graphon12}. are satisfied. Let $p>0$ and $N \in \mathbb{N}$ satisfy (\ref{eq:largeN}). Suppose that $X_1, \ldots, X_N$ are drawn i.i.d. from $\P$ on $\chi$. 
If the event $\mathcal{F}_{\rm Lip}^p$ from Lemma \ref{lemma:C2} occurs, then condition (\ref{eq:d_XLowerBound}) together with (\ref{eq:lemmab6-12}) below  are satisfied:  for every $f:\chi \to \mathbb{R}^{F}$ with  Lipschitz constant $L_f$, $\Phi: \mathbb{R}^{2F} \to \mathbb{R}^{H}$ with   Lipschitz constant $L_\Phi$ and $\Psi: \mathbb{R}^{F + H} \to \mathbb{R}^{F'}$ with   Lipschitz constant $L_\Psi$ 
\begin{equation}
\label{eq:lemmab6-12}
   \begin{aligned}
& \left\| \Psi\Big(f(\cdot), M_X\big(\Phi (f,f)\big)(\cdot) \Big) - \Psi \Big(f(\cdot) , M_W \big(\Phi (f,f) \big) (\cdot) \Big)\right\|_\infty \\ & \leq L_\Psi \Bigg(
   4 \frac{\varepsilon_1}{\sqrt{N}\mathrm{d}_{min}^2} \|W\|_\infty(L_\Phi\|f\|_\infty + \|\Phi(0,0)\|_\infty) \\
   &+  N^{-\frac{1}{2(D_\chi+1)}}\Bigg(
2\Big(
\frac{\|W\|_\infty}{\mathrm{d}_{min}} L_\Phi L_f + \ \frac{\LipW}{\mathrm{d}_{min}}     \big( L_\Phi\|f\|_\infty + \|\Phi(0,0) \|_\infty \big)  \Big)
 \\
& + \frac{C_\chi}{\sqrt{2}} \Big(\frac{\|W\|_\infty}{\mathrm{d}_{min}} \big( L_\Phi \|f\|_\infty + \|\Phi(0,0)\|_\infty \big) \Big)
\sqrt{\log(C_\chi) + \frac{D_\chi}{2(D_\chi+1)} \log(N) + \log(2/p)}  \Bigg)
\Bigg),
 \end{aligned} 
\end{equation}
where $\varepsilon_1$ is defined in (\ref{eq:eps1-2}).
\end{corollary}
\begin{proof}
We calculate,
\[
\begin{aligned}
&\left\| \Psi \Big(f(\cdot), M_X\big(\Phi (f,f) \big)(\cdot) \Big) - \Psi \Big(f(\cdot) , M_W \big(\Phi (f,f) \big) (\cdot) \Big)\right\|_\infty  \\
\leq & \LipPsi \left\| M_X\big(\Phi (f,f)\big)(\cdot) -  M_W\big(\Phi (f,f) \big) (\cdot) \right\|_\infty,
\end{aligned}
\]
and apply Lemma \ref{lemma:C2} to the right-hand-side.
\end{proof}

We continue by providing three lemmas which capture deterministic properties of cMPNNs and gMPNNs. We start by showing that the infinity norm of the output  of the $l$-th layer of a cMPNN $f^{(l)}$ can be bounded in terms of the infinity norm of its input $f$.
 
\begin{lemma}
\label{lemma:RecRelNorm}
Let $(\chi,d, \P) $ be a metric-measure space,
$W$ be a kernel and $\Theta = ((\Phi^{(l)})_{l=1}^T, (\Psi^{(l)})_{l=1}^T)$  be a MPNN s.t.  Assumptions \ref{ass:graphon}.\ref{ass:graphon1}-\ref{ass:graphon4}. are satisfied. Consider a 
metric-space signal  $f: \chi \to \mathbb{R}^F$ with $\|f\|_\infty < \infty$. 
Then, for $l=0, \ldots, T-1$, the cMPNN output $f^{(l+1)}$ satisfies 
\[
\|f^{(l+1)}\|_\infty \leq B_1^{(l+1)} + \|f\|_\infty B_2^{(l+1)},
\]
where
\begin{equation}
    \label{eq:B'}
    B_1^{(l+1)} = \sum_{k=1}^{l+1}  \big(
L_{\Psi^{(k)}} \frac{\cmax}{\cmin}\|\Phi^{(k)}(0,0)\|_\infty+ \|\Psi^{(k)}(0,0)\|_\infty \big) \prod_{l' = k+1}^{l+1}  L_{\Psi^{(l')}} \big( 1 + \frac{\cmax}{\cmin}  L_{\Phi^{(l')}} \big) 
\end{equation}
and
\begin{equation}
    \label{eq:B''}
    B_2^{(l+1)} = \prod_{k = 1}^{l+1} L_{\Psi^{(k)}} \left(1  + \frac{\cmax}{\cmin}  L_{\Phi^{(k)}} \right).
\end{equation}
\end{lemma}
\begin{proof}
Let $l=0, \ldots, T-1$. Then, for $k =0, \ldots, l$, we have
\[
\begin{aligned}
\|f^{(k+1)}(\cdot)\|_\infty & =  \Big\| \Psi^{(k+1)} \Big(f^{(k)}(\cdot), M_W \big( \Phi^{(k+1)}( f^{(k)},f^{(k)}) \big) (\cdot) \Big) \Big\|_\infty \\ 
& \leq \Big\| \Psi^{(k+1)} \Big(f^{(k)}(\cdot), M_W \big(\Phi^{(k+1)} ( f^{(k)},f^{(k)}) \big) (\cdot) \Big) - \Psi^{(k+1)}(0,0) \Big\|_\infty  +   \|\Psi^{(k+1)}(0,0)\|_\infty
\\
& \leq L_{\Psi^{(k+1)}} \Big(\| f^{(k)} \|_\infty
+ \big\| M_W\big( \Phi^{(k+1)} (f^{(k)},f^{(k)} )\big) (\cdot) \big\|_\infty \Big) + \|\Psi^{(k+1)}(0,0)\|_\infty.
\end{aligned}
\]
For the message term, we have
\[
\begin{aligned}
\big\| M_W \big( \Phi^{(k+1)} (f^{(k)},f^{(k)} )\big) (\cdot) \big\|_\infty & = 
\left\| \int_\chi \frac{W(\cdot, y)}{\mathrm{d}_W(\cdot)} \Phi^{(k+1)} \big(f^{(k)}(\cdot), f^{(k)}(y) \big)  d\P(y)\right\|_\infty \\
& \leq \frac{\cmax}{\cmin}(L_{\Phi^{(k+1)}}\|f^{(k)}\|_\infty + \|\Phi^{(k+1)}(0,0)\|_\infty).
\end{aligned}
\]

Hence,
\[
\begin{aligned}
& \|f^{(k+1)}(\cdot)\|_\infty \\ & \leq   L_{\Psi^{(k+1)}} \Big(\| f^{(k)} \|_\infty
+ \frac{\cmax}{\cmin}(L_{\Phi^{(k+1)}}\|f^{(k)}\|_\infty + \|\Phi^{(k+1)}(0,0)\|_\infty) \Big) + \|\Psi^{(k+1)}(0,0)\|_\infty,
\end{aligned}
\]
which we can reorder to
\[
\begin{aligned}
& \|f^{(k+1)}(\cdot)\|_\infty \\ & \leq  L_{\Psi^{(k+1)}}\Big(  1
+  \frac{\cmax}{\cmin}L_{\Phi^{(k+1)}} \Big) \| f^{(k)} \|_\infty + L_{\Psi^{(k+1)}} \frac{\cmax}{\cmin}\|\Phi^{(k+1)}(0,0)\|_\infty  + \|\Psi^{(k+1)}(0,0)\|_\infty.
\end{aligned}
\]
We apply Lemma \ref{lemma:RecRecGen} to solve this recurrence relation which finishes the proof.
\end{proof}

In the following, we denote by $\Lipfl$  the Lipschitz constant of $f^{(l)}$.
The next lemma bounds $L_{f^{(l+1)}}$ in terms of $ L_f $.

\begin{lemma}
\label{lemma:RecRelLipschitz}
Let $(\chi,d, \P) $ be a metric-measure space,
$W$ be a kernel and $\Theta = ((\Phi^{(l)})_{l=1}^T, (\Psi^{(l)})_{l=1}^T)$  be a MPNN s.t.  Assumptions \ref{ass:graphon}.\ref{ass:graphon1}-\ref{ass:graphon4}. are satisfied. Consider a Lipschitz continuous metric-space signal  $f: \chi \to \mathbb{R}^F$ with $\|f\|_\infty < \infty$ and Lipschitz constant $L_f$. Then, for $l=0, \ldots, T-1$, the  cMPNN output  $f^{(l+1)}$ is Lipschitz continuous with Lipschitz constant $L_{f^{(l+1)}}$ satisfying 
 \[
\begin{aligned}
L_{f^{(l+1)}} & \leq \sum_{k=1}^{l+1} \Bigg( \Big( 
L_{\Psi^{(k)}}\frac{\cl}{\cmin}( \|\Phi^{(k)}(0,0)\|_\infty +  L_{\Phi^{(k)}} \|f^{(k-1)}\|_\infty ) + L_{\Psi^{(k)}}\cmax ( \|\Phi^{(k)}(0,0)\|_\infty 
\\ 
& +   L_{\Phi^{(k)}} \|f^{(k-1)}\|_\infty )\frac{\cl}{\cmin^2}
\Big)  \prod_{l' = k+1}^{l+1} L_{\Psi^{(l')}}  \Big(1+\frac{\cmax}{\cmin} L_{\Phi^{(l')}}  \Big) \Bigg)  \\ & + 
  L_{f}  \prod_{k=1}^{l+1}
L_{\Psi^{(k)}} \Big(1+\frac{\cmax}{\cmin} L_{\Phi^{(k)}}\Big).
\end{aligned}
\]

\end{lemma}
\begin{proof}
Let $l=0, \ldots, T-1$ and consider $k =0, \ldots, l$.
For $x,x' \in \chi$, we have 
\begin{equation}
    \label{eq:lemmarecrellip-1}
\begin{aligned}
&\|f^{(k+1)}(x) - f^{(k+1)}(x') \|_\infty 
\\
& =  \Big\| \Psi^{(k+1)} \Big(f^{(k)}(x), M_W \big(\Phi^{(k+1)} (f^{(k)},f^{(k)}) \big)(x) \Big) 
\\
& - \Psi^{(k+1)}\Big(f^{(k)}(x'), M_W \big(\Phi^{(k+1)} (f^{(k)},f^{(k)})\big)(x')\Big) \Big\|_\infty \\
& \leq L_{\Psi^{(k+1)}} \Big(
\Big\|f^{(k)}(x) - f^{(k)}(x')\Big\|_\infty 
\\
& + \Big\| M_W \big(\Phi^{(k+1)} (f^{(k)},f^{(k)})\big)(x) - M_W \big(\Phi^{(k+1)} (f^{(k)},f^{(k)})\big)(x')  \Big\|_\infty
\Big) \\ & \leq 
 L_{\Psi^{(k+1)}} \Big( L_{f^{(k)}}
d(x, x') + \|M_W \big(\Phi^{(k+1)} (f^{(k)},f^{(k)})\big)(x) - M_W \big(\Phi^{(k+1)} (f^{(k)},f^{(k)})\big)(x') \|_\infty
\Big).
\end{aligned}
\end{equation}

For the second term, we have
\begin{equation}
    \label{eq:RecRelLip-1}
\end{equation}

For $(A)$, we have 
\[
\begin{aligned}
(A) & =
\int_\chi \Big\| \frac{W(x,y)}{\mathrm{d}_W(x)} \Phi^{(k+1)} \big(f^{(k)}(x),f^{(k)}(y) \big)   - \frac{W(x',y)}{\mathrm{d}_W(x)} \Phi^{(k+1)} \big(f^{(k)}(x),f^{(k)}(y) \big) \Big\|_\infty d\P(y) \\
& = \int_\chi \frac{|W(x,y) - W(x',y)|}{\mathrm{d}_W(x)} \Big\| \Phi^{(k+1)} \big(f^{(k)}(x),f^{(k)}(y) \big)   \Big\|_\infty d\P(y)  \\
& \leq  L_W\frac{d(x,x')}{\cmin} \int_\chi\Big\| \Phi^{(k+1)} \big(f^{(k)}(x),f^{(k)}(y) \big)    \Big\|_\infty d\P(y) \\
& \leq \frac{\cl}{\cmin}\big( \|\Phi^{(k+1)}(0,0)\|_\infty +  L_{\Phi^{(k+1)}} \|f^{(k)}\|_\infty \big)d(x,x').
\end{aligned}
\]
For $(B)$, we have
\[
\begin{aligned}
(B)
& = 
\int_\chi \Big\| \frac{W(x',y)}{\mathrm{d}_W(x)}\Phi^{(k+1)} \big(f^{(k)}(x),f^{(k)}(y) \big) - \frac{W(x',y)}{\mathrm{d}_W(x)} \Phi^{(k+1)}\big(f^{(k)}(x'),f^{(k)}(y) \big)\Big\|_\infty d\P(y) \\
&
=
\int_\chi \frac{|W(x',y)|}{|\mathrm{d}_W(x)|}\Big\| \Phi^{(k+1)} \big(f^{(k)}(x),f^{(k)}(y) \big) -  \Phi^{(k+1)}\big(f^{(k)}(x'),f^{(k)}(y) \big)\Big\|_\infty d\P(y)
\\
& \leq 
\frac{\cmax}{\cmin}  L_{\Phi^{(k+1)}}
\int_\chi \big\| \big(f^{(k)}(x),f^{(k)}(y) \big) -  \big(f^{(k)}(x'),f^{(k)}(y) \big)\big\|_\infty d\P(y)
\\
& \leq 
\frac{\cmax}{\cmin}  L_{\Phi^{(k+1)}}
\| f^{(k)}(x) - f^{(k)}(x')\big)\|_\infty 
\\
&  \leq \frac{\cmax}{\cmin} L_{\Phi^{(k+1)}}  L_{f^{(k)}} d(x,x').
\end{aligned}
\]
For $(C)$, we have
\[
\begin{aligned}
 (C) & = \int_\chi \Big\| \frac{W(x',y)}{\mathrm{d}_W(x)} \Phi^{(k+1)} \big(f^{(k)}(x'),f^{(k)}(y) \big) - \frac{W(x',y)}{\mathrm{d}_W(x')} \Phi^{(k+1)} \big(f^{(k)}(x'),f^{(k)}(y) \big) \Big\|_\infty d\P(y) \\
& = \int_\chi |W(x',y)| \Big|\frac{1}{\mathrm{d}_W(x)} - \frac{1}{\mathrm{d}_W(x')}\Big| \Big\| \Phi^{(k+1)} \big(f^{(k)}(x'),f^{(k)}(y) \big) \Big\|_\infty d\P(y) 
\\
&\leq \cmax ( \|\Phi^{(k+1)}(0,0)\|_\infty +  L_{\Phi^{(k+1)}} \|f^{(k)}\|_\infty )\frac{\cl}{\cmin^2}d(x,x'),
\end{aligned}
\]
where the last inequality holds since
\[
\begin{aligned}
\left|\frac{1}{\mathrm{d}_W(x)}  - \frac{1}{\mathrm{d}_W(x')} \right| 
& \leq
 \frac{|\mathrm{d}_W(x') - \mathrm{d}_W(x)|}{|\mathrm{d}_W(x)\mathrm{d}_W(x')|} \\
& \leq 
\frac{1}{ \cmin^2} |\mathrm{d}_W(x') - \mathrm{d}_W(x)| \\
& \leq 
\frac{1}{ \cmin^2} \int_\chi|W(x',y) - W(x,y)| d\mu(y) \\
& \leq 
\frac{1}{ \cmin^2} \int_\chi L_W d(x, x') d\mu(y) \\
&  \leq 
\frac{\cl}{ \cmin^2}  d(x,x').
\end{aligned}
\]
Hence, by plugging (\ref{eq:RecRelLip-1}) and our bounds for $(A), (B) $ and $(C)$ into  (\ref{eq:lemmarecrellip-1}), we have 
\[
\begin{aligned}
&\|f^{(k+1)}(x) - f^{(k+1)}(x') \|_\infty  \\
 & \leq 
 L_{\Psi^{(k+1)}} \Big( L_{f^{(k)}}
d(x, x') + \|M_W \big(\Phi^{(k+1)} (f^{(k)},f^{(k)})\big)(x) - M_W \big(\Phi^{(k+1)} (f^{(k)},f^{(k)})\big)(x') \|_\infty
\Big)  \\
& \leq 
 L_{\Psi^{(k+1)}} \Big( L_{f^{(k)}}
d(x, x') + (A) + (B) +(C)
\Big)  \\
& \leq 
L_{\Psi^{(k+1)}} \Big( L_{f^{(k)}}
 + \frac{\cl}{\cmin}\big( \|\Phi^{(k+1)}(0,0)\|_\infty +  L_{\Phi^{(k+1)}} \|f^{(k)}\|_\infty \big) \\
&  + \frac{\cmax}{\cmin} L_{\Phi^{(k+1)}}  L_{f^{(k)}}  + \cmax ( \|\Phi^{(k+1)}(0,0)\|_\infty +  L_{\Phi^{(k+1)}} \|f^{(k)}\|_\infty )\frac{\cl}{\cmin^2}
\Big) d(x, x').
\end{aligned}
\]
Hence,
\[
\begin{aligned}
L_{f^{(k+1)}} & \leq
L_{\Psi^{(k+1)}} \frac{\cl}{\cmin}( \|\Phi^{(k+1)}(0,0)\|_\infty +  L_{\Phi^{(k+1)}} \|f^{(k)}\|_\infty ) + L_{\Psi^{(k+1)}} \left(1+\frac{\cmax}{\cmin} L_{\Phi^{(k+1)}}\right)  L_{f^{(k)}}  \\ & + L_{\Psi^{(k+1)}} \cmax ( \|\Phi^{(k+1)}(0,0)\|_\infty +  L_{\Phi^{(k+1)}} \|f^{(k)}\|_\infty )\frac{\cl}{\cmin^2}.
\end{aligned}
 \]
 We finish the proof by solving the recurrence relation with Lemma \ref{lemma:RecRecGen}.
\end{proof}

\begin{corollary}
Consider the same setting as in Lemma \ref{lemma:RecRelLipschitz}. Then, for $l=0, \ldots, T-1$,
\[
 \Lipfl  \leq Z_1^{(l)} + Z_2^{(l)}\|f\|_\infty + Z_3^{(l)} \Lipf ,
\]
where $Z_1^{(l)}$, $Z_2^{(l)}$ and $Z_3^{(l)}$ are independent of $f$ and defined as 
\begin{equation}
\label{eq:z1z2z3}
  \begin{aligned}
&  Z_1^{(l)} = \sum_{k=1}^{l}  \Bigg(\Big(
L_{\Psi^{(k)}}\frac{\cl}{\cmin}  \|\Phi^{(k)}(0,0)\|_\infty  +
L_{\Psi^{(k)}}\cmax
\|\Phi^{(k)}(0,0)\|_\infty 
 \frac{\cl}{\cmin^2}\Big)  \\
 & +  B_1^{(k-1)} \Big(
L_{\Psi^{(k)}}\frac{\cl}{\cmin}  L_{\Phi^{(k)}}   +
L_{\Psi^{(k)}}\cmax
 L_{\Phi^{(k)}} 
 \frac{\cl}{\cmin^2}\Big) \Bigg)  \prod_{l' = k+1}^{l} L_{\Psi^{(l')}}  \Big(1+\frac{\cmax}{\cmin} L_{\Phi^{(l')}}  \Big), \\
 & Z_2^{(l)} = \sum_{k=1}^{l}   B_2^{(k-1)}   \Big(
L_{\Psi^{(k)}}\frac{\cl}{\cmin}  L_{\Phi^{(k)}}   +
L_{\Psi^{(k)}}\cmax
 L_{\Phi^{(k)}} 
 \frac{\cl}{\cmin^2}\Big)  \prod_{l' = k+1}^{l} L_{\Psi^{(l')}}  \Big(1+\frac{\cmax}{\cmin} L_{\Phi^{(l')}}  \Big), \\
 & Z_3^{(l)} =     \prod_{k=1}^{l}
L_{\Psi^{(k)}} \Big(1+\frac{\cmax}{\cmin} L_{\Phi^{(k)}}\Big),
\end{aligned}
\end{equation}
where $B_1^{(k)}$ and $B_2^{(k)}$ are defined in (\ref{eq:B'}) and (\ref{eq:B''}).
\end{corollary}
\begin{proof}
By Lemma \ref{lemma:RecRelLipschitz}, we have 
\[
\begin{aligned}
L_{f^{(l)}} & \leq \sum_{k=1}^{l} \Bigg( \Big( 
L_{\Psi^{(k)}}\frac{\cl}{\cmin}( \|\Phi^{(k)}(0,0)\|_\infty +  L_{\Phi^{(k)}} \|f^{(k-1)}\|_\infty ) + L_{\Psi^{(k)}}\cmax ( \|\Phi^{(k)}(0,0)\|_\infty 
\\ 
& +   L_{\Phi^{(k)}} \|f^{(k-1)}\|_\infty )\frac{\cl}{\cmin^2}
\Big)  \prod_{l' = k+1}^{l} L_{\Psi^{(l')}}  \Big(1+\frac{\cmax}{\cmin} L_{\Phi^{(l')}}  \Big) \Bigg)  \\ & + 
  L_{f}  \prod_{k=1}^{l}
L_{\Psi^{(k)}} \Big(1+\frac{\cmax}{\cmin} L_{\Phi^{(k)}}\Big) \\
& = \sum_{k=1}^{l}  \Big(
L_{\Psi^{(k)}}\frac{\cl}{\cmin}  \|\Phi^{(k)}(0,0)\|_\infty  +
L_{\Psi^{(k)}}\cmax
\|\Phi^{(k)}(0,0)\|_\infty 
 \frac{\cl}{\cmin^2}\Big)  \\ & \prod_{l' = k+1}^{l} L_{\Psi^{(l')}}  \Big(1+\frac{\cmax}{\cmin} L_{\Phi^{(l')}}  \Big) \\
 & + \sum_{k=1}^{l} \|f^{(k-1)}\|_\infty \Big(
L_{\Psi^{(k)}}\frac{\cl}{\cmin}  L_{\Phi^{(k)}}   +
L_{\Psi^{(k)}}\cmax
 L_{\Phi^{(k)}} 
 \frac{\cl}{\cmin^2}\Big) \\ &  \prod_{l' = k+1}^{l} L_{\Psi^{(l')}}  \Big(1+\frac{\cmax}{\cmin} L_{\Phi^{(l')}}  \Big)  \\ & + 
  L_{f}  \prod_{k=1}^{l}
L_{\Psi^{(k)}} \Big(1+\frac{\cmax}{\cmin} L_{\Phi^{(k)}}\Big) \\
& \leq \sum_{k=1}^{l}  \Big(
L_{\Psi^{(k)}}\frac{\cl}{\cmin}  \|\Phi^{(k)}(0,0)\|_\infty  +
L_{\Psi^{(k)}}\cmax
\|\Phi^{(k)}(0,0)\|_\infty 
 \frac{\cl}{\cmin^2}\Big) \\ &  \prod_{l' = k+1}^{l} L_{\Psi^{(l')}}  \Big(1+\frac{\cmax}{\cmin} L_{\Phi^{(l')}}  \Big) \\
 & + \sum_{k=1}^{l} (B_1^{(k-1)} + B_2^{(k-1)} \|f\|_\infty) \Big(
L_{\Psi^{(k)}}\frac{\cl}{\cmin}  L_{\Phi^{(k)}}   +
L_{\Psi^{(k)}}\cmax
 L_{\Phi^{(k)}} 
 \frac{\cl}{\cmin^2}\Big) \\ &  \prod_{l' = k+1}^{l} L_{\Psi^{(l')}}  \Big(1+\frac{\cmax}{\cmin} L_{\Phi^{(l')}}  \Big)  \\ & + 
  L_{f}  \prod_{k=1}^{l}
L_{\Psi^{(k)}} \Big(1+\frac{\cmax}{\cmin} L_{\Phi^{(k)}}\Big),
\end{aligned}
\]
where the last inequality holds by Lemma \ref{lemma:RecRelNorm}.
\end{proof}

We continue with the following simple lemma which bounds the infinity norm of the output of a gMPNN.

\begin{lemma}
\label{lemma:DeterministicMPNNBound}
Let $(\chi,d, \P) $ be a metric-measure space,
$W$ be a kernel and $\Theta = \big((\Phi^{(l)})_{l=1}^T, (\Psi^{(l)})_{l=1}^T \big)$ be a MPNN s.t.  Assumptions \ref{ass:graphon}.\ref{ass:graphon1}-\ref{Ass:KernelDiagBounded}. are satisfied. Consider a   metric-space signal  $f: \chi \to \mathbb{R}^F$ with $\|f\|_\infty < \infty$.
 Consider a graph $(G, \mathbf{f}) \sim (W, f)$ with $N$ nodes and corresponding graph features.  
Then, 
\[
 \| \Theta_G(\mathbf{f})\|_{2;\infty}^2  \leq N^{2T} (A' + A'' \|f\|_\infty^2),
\]
where
\[
\begin{aligned}
A' & = \sum_{l=1}^{T} \Big(
2 (L_{\Psi^{(l)}})^2   \frac{2}{{\rm W}_{\mathrm{diag}}^2}\|W\|_\infty^2  \|\Phi^{(l )}(0,0)\|_\infty^2  + 2\|\Psi^{(l )}(0,0)\|_\infty^2
 \Big) \\
& \prod_{l' = l+1}^T 
   2 (L_{\Psi^{(l')}})^2
   \big( \frac{2}{{\rm W}_{\mathrm{diag}}^2}\|W\|_\infty^2 (L_{\Phi^{(l')}})^2+ 1\big)
   \end{aligned}
\]
and
\[
A'' = \prod_{l=1}^T 
 2 (L_{\Psi^{(l)}})^2  
   \big(\frac{2}{{\rm W}_{\mathrm{diag}}^2}\|W\|_\infty^2 (L_{\Phi^{(l)}})^2+ 1\big)
\]
\end{lemma}
\begin{proof}
Let $l=0,\ldots, T-1$. We have
\[
\|\mathbf{f}^{(l+1)}\|_{2; \infty}^2 = \frac{1}{N} \sum_{i=1}^N \|\mathbf{f}_i^{(l+1)}\|_\infty^2,
\]
where $ \mathbf{f}_i^{(l+1)} = \Psi^{(l+1)} ( \mathbf{f}_i^{(l)}, \mathbf{m}^{(l+1)}_i ) $ with $ \mathbf{m}^{(l+1)}_i = M_G \big(  \Phi^{(l+1)}( \mathbf{f}^{(l)}, \mathbf{f}^{(l)} )\big)(X_i) $. By using the Lipschitz continuity of $\Psi^{(l+1)}$, we get  
\begin{equation}
    \label{eq:lemmab19-1}
    \begin{aligned}
\|\mathbf{f}_i^{(l+1)}\|_\infty^2 & 
 \leq 2\big( \|\Psi^{(l+1)} ( \mathbf{f}_i^{(l)}, \mathbf{m}^{(l+1)}_i ) - \Psi^{(l+1)} (0,0 )\|_\infty^2
 + \|\Psi^{(l+1)} (0,0 )\|_\infty^2 \big) \\
& \leq 
 2\big( (L_{\Psi^{(l+1)}})^2(\| \mathbf{f}_i^{(l)} \|_\infty^2 + \|\mathbf{m}_i^{(l+1)}\|_\infty^2) 
 + \|\Psi^{(l+1)} (0,0 )\|_\infty^2 \big) 
\end{aligned}
\end{equation}

For the message term we calculate
\[
\begin{aligned}
  \| \mathbf{m}^{(l+1)}_i \|_\infty^2 & = \left\|  \frac{1}{ \sum_{j=1}^N W(X_i, X_j) } \sum_{j=1}^N W(X_i, X_j) \Phi^{(l+1)}(  \mathbf{f}_i^{(l)}, \mathbf{f}_j^{(l)} )  \right\|_\infty^2 \\
 & \leq \left|\frac{1}{ \sum_{j=1}^N W(X_i, X_j) }\right|^2 \sum_{j=1}^N |W(X_i, X_j)|^2 \sum_{j=1}^N  \|\Phi^{(l+1)}(  \mathbf{f}_i^{(l)}, \mathbf{f}_j^{(l)} )  \|_\infty^2,
\end{aligned}
\]
where the inequality follows from Cauchy-Schwarz inequality. 
 Per assumption, we have $| W(X_i, X_i) | \geq {\rm W}_{\mathrm{diag}}$ and for every $i=1, \ldots, N$,
 \[
 \begin{aligned}
 \|\Phi^{(l+1)}(  \mathbf{f}_i^{(l)}, \mathbf{f}_j^{(l)} )  \|_\infty^2 & = \|\Phi^{(l+1)}(  \mathbf{f}_i^{(l)}, \mathbf{f}_j^{(l)} ) - \Phi^{(l+1)}(  0,0) + \Phi^{(l+1)}(  0,0) \|_\infty^2 \\
 & \leq 2\Big(\|\Phi^{(l+1)}(  \mathbf{f}_i^{(l)}, \mathbf{f}_j^{(l)} ) - \Phi^{(l+1)}(  0,0) \|_\infty^2+ \|\Phi^{(l+1)}(  0,0) \|_\infty^2 \Big) \\ 
   & \leq 2\Big( 
   (L_{\Phi^{(l+1)}})^2  ( \| \mathbf{f}_i^{(l)} \|_\infty^2 + \|\mathbf{f}_j^{(l)}\|_\infty^2) +  \|\Phi^{(l+1)}(  0,0) \|_\infty^2.
   \Big)
 \end{aligned}
 \]
 Hence,
\begin{equation}
    \label{eq:lemmab19-2}
    \begin{aligned}
    \|\mathbf{m}_i^{(l+1)}\|_\infty^2 &
 \leq \frac{2}{{\rm W}_{\mathrm{diag}}^2} \|W\|_\infty^2 N \sum_{j=1}^N \Big(  (L_{\Phi^{(l+1)}})^2  ( \| \mathbf{f}_i^{(l)} \|_\infty^2 + \|\mathbf{f}_j^{(l)}\|_\infty^2) +  \|\Phi^{(l+1)}(  0,0) \|_\infty^2
   \Big)
    \\
    & \leq \frac{2}{{\rm W}_{\mathrm{diag}}^2} \|W\|_\infty^2 N^2\Big(
(L_{\Phi^{(l+1)}})^2   \|\mathbf{f}_i^{(l)}\|_\infty^2 + (L_{\Phi^{(l+1)}})^2  \|\mathbf{f}^{(l)}\|_{2;\infty}^2 +   \|\Phi^{(l+1)}(  0,0) \|_\infty^2
\Big).
    \end{aligned}
\end{equation}
By (\ref{eq:lemmab19-1}) and (\ref{eq:lemmab19-2}), we have
\[
\begin{aligned}
 \|\mathbf{f}^{(l+1)}\|_{2; \infty}^2 & \leq  \frac{1}{N} \sum_{i=1}^N 2\Big( 
   (L_{\Psi^{(l+1)}})^2  ( \| \mathbf{f}_i^{(l)} \|_\infty^2 + \|\mathbf{m}_i^{(l+1)} \|_\infty^2) +  \|\Psi^{(l+1)}(  0,0) \|_\infty^2
   \Big) \\
   &\leq  
   \frac{1}{N} \sum_{i=1}^N 2\Bigg( 
   (L_{\Psi^{(l+1)}})^2  \Big( \| \mathbf{f}_i^{(l)} \|_\infty^2 +  N^2\frac{2}{{\rm W}_{\mathrm{diag}}^2} \|W\|_\infty^2\big(
(L_{\Phi^{(l+1)}})^2   \|\mathbf{f}_i^{(l)}\|_\infty^2 \\ & + (L_{\Phi^{(l+1)}})^2   \|\mathbf{f}^{(l)}\|_{2;\infty}^2  +   \|\Phi^{(l+1)}(  0,0) \|_\infty^2
\big) \Big) +  \|\Psi^{(l+1)}(  0,0) \|_\infty^2
   \Bigg) \\
    & = 2 (L_{\Psi^{(l+1)}})^2 \Big( \frac{1}{N}\sum_{i=1}^N \|\mathbf{f}^{(l)}_i\|_{\infty}^2 + N^2 \frac{2}{{\rm W}_{\mathrm{diag}}^2}\|W\|_\infty^2\big( 
    (L_{\Phi^{(l+1)}})^2  \frac{1}{N}\sum_{i=1}^N  \|\mathbf{f}^{(l)}_i\|_{\infty}^2  
   \\ & +
    (L_{\Phi^{(l+1)}})^2 \|\mathbf{f}^{(l)}\|_{2;\infty}^2
     + \|\Phi^{(l+1)}(0,0)\|_\infty^2
    \big) 
     \Big) +
   2\|\Psi^{(l+1)}(0,0)\|_\infty^2
   \\ 
   & = 2 (L_{\Psi^{(l+1)}})^2 \Big(\|\mathbf{f}^{(l)}\|_{2;\infty}^2 + N^2 \frac{2}{{\rm W}_{\mathrm{diag}}^2}\|W\|_\infty^2 \big( (L_{\Phi^{(l+1)}})^2 \|\mathbf{f}^{(l)}\|_{2;\infty}^2 + \|\Phi^{(l+1)}(0,0)\|_\infty^2 \big)\Big) 
   \\ & +
   2\|\Psi^{(l+1)}(0,0)\|_\infty^2 
   \\ & =
     2 (L_{\Psi^{(l+1)}})^2  
   \big(N^2\frac{2}{{\rm W}_{\mathrm{diag}}^2}\|W\|_\infty^2 (L_{\Phi^{(l+1)}})^2+ 1\big) 
    \|\mathbf{f}^{(l)}\|_{2;\infty}^2 
   \\ & +
    2 (L_{\Psi^{(l+1)}})^2  N^2\frac{2}{{\rm W}_{\mathrm{diag}}^2}\|W\|_\infty^2  \|\Phi^{(l+1)}(0,0)\|_\infty^2  + 2\|\Psi^{(l+1)}(0,0)\|_\infty^2
\end{aligned}
\]

 Hence, by $\|\mathbf{f}\|_{2;\infty}^2 \leq \|f\|_\infty^2$ and Lemma \ref{lemma:RecRecGen}, we have
 \[
 \begin{aligned}
  \|\mathbf{f}^{(T)}\|_{2; \infty}^2
 & \leq \sum_{l=1}^{T} \Big(
2 (L_{\Psi^{(l)}})^2  N^2\frac{2}{{\rm W}_{\mathrm{diag}}^2}\|W\|_\infty^2  \|\Phi^{(l )}(0,0)\|_\infty^2  + 2\|\Psi^{(l )}(0,0)\|_\infty^2
 \Big) \\ &
 \prod_{l' = l+1}^T 
   2 (L_{\Psi^{(l')}})^2
   \big(N^2\frac{2}{{\rm W}_{\mathrm{diag}}^2}\|W\|_\infty^2 (L_{\Phi^{(l')}})^2+ 1\big)   \\
& +  \|f\|_\infty^2\prod_{l=1}^T \Big(
 2 (L_{\Psi^{(l)}})^2  
   \big(N^{2}\frac{2}{{\rm W}_{\mathrm{diag}}^2}\|W\|_\infty^2 (L_{\Phi^{(l)}})^2+ 1\big)
 \Big) \\
 & \leq 
N^{2T} \sum_{l=1}^{T} \Big(
2 (L_{\Psi^{(l)}})^2  \frac{2}{{\rm W}_{\mathrm{diag}}^2}\|W\|_\infty^2  \|\Phi^{(l )}(0,0)\|_\infty^2  + 2\|\Psi^{(l )}(0,0)\|_\infty^2
 \Big) \\ &
 \prod_{l' = l+1}^T 
   2 (L_{\Psi^{(l')}})^2
   \big(\frac{2}{{\rm W}_{\mathrm{diag}}^2}\|W\|_\infty^2 (L_{\Phi^{(l')}})^2+ 1\big)   \\
& +  \|f\|_\infty^2 N^{2T} \prod_{l}^T \Big(
 2 (L_{\Psi^{(l)}})^2  
   \big(\frac{2}{{\rm W}_{\mathrm{diag}}^2}\|W\|_\infty^2 (L_{\Phi^{(l)}})^2+ 1\big)
 \Big).
 \end{aligned}
 \]
\end{proof}

We finish this subsection with the following easily verifiable lemma that  provides  a general solution for certain recurrence relations.

\begin{lemma}
 \label{lemma:RecRecGen}
 Let $(\eta^{(l)})_{l=0}^T$ be a sequence of real numbers satisfying $\eta^{(l+1)} \leq a^{(l+1)}\eta^{(l)} + b^{(l+1)}$ for $l = 0,\ldots, T-1$,  for some real numbers $a^{(l)}, b^{(l)}$, $l = 1, \ldots, T$. Then
 \[
 \eta^{(T)} \leq \sum_{l=1}^T b^l\prod_{l' = l+1}^Ta^{(l')} + \eta^{(0)}\prod_{l=1}^Ta^{(l)},
 \]
 where we define the product $\prod_{T+1}^{T}$ as $1$.
 \end{lemma}

\subsection{Proof of Theorem \ref{thm:MainInProb}}
\label{subsec:conv}
The idea of the Proof of Theorem \ref{thm:MainInProb} is as follows. We first use Corollary  \ref{cor:Uniform3} to bound the error between a cMPNN and a gMPNN layer-wise, when the input of layer $l$ of the gMPNN is exactly the
sampled graph signal from the output of layer $l - 1$ of the cMPNN. This is shown in Corollary \ref{cor:C5}. Then, we use this to provide a recurrence relation for the true error between a cMPNN and the
corresponding gMPNN in Lemma \ref{lemma:C5}. We solve this recurrence relation in Corollary \ref{cor:c6}, where
we have an error bound that depends only on the parameters of the MPNN, the regularity of the
kernel and the regularity of the continuous output metric-space signals of the layers of the cMPNN.
We remove the last dependency in  Theorem \ref{thm:C7}. We then analyze the additional error by a final pooling layer, leading to the formulation and final proof of Theorem \ref{thm:MainInProb}, rewritten as Theorem \ref{thm:unifExpValue}.

\begin{corollary}
\label{cor:C5}
Let $(\chi,d, \P) $ be a metric-measure space and
$W$ be a kernel
s.t. Assumptions \ref{ass:graphon}.\ref{ass:graphon1}-\ref{ass:graphon12} are satisfied. Let $p \in (0, \frac{1}{2})$.  
 Consider a graph $(G, \mathbf{f}) \sim (W, f)$ with $N$ nodes and corresponding graph features,  where $N$ satisfies  (\ref{eq:largeN}). If the event $\mathcal{F}_{\rm Lip}^p$ from Lemma \ref{lemma:C2} occurs, then condition (\ref{eq:d_XLowerBound}) together with (\ref{eq:lemmac4-1})   below   are satisfied: 
 For every MPNN $\Theta$ satisfying Assumption \ref{ass:graphon}.\ref{ass:graphon4}. and $f:\chi \to \mathbb{R}^F$ with Lipschitz constant $L_f$, 
  we have 
\begin{equation}
\label{eq:lemmac4-1}
  \d\left( \Lambda_{\Theta_G}^{(l+1)}(S^X f^{(l)}),  \Lambda_{\Theta_W}^{(l+1)}(f^{(l)})\right)
    \leq Q^{(l+1)}  
\end{equation}
for all $l = 0, \ldots, T-1$, where $f^{(l)}=\Theta_W^{(l)}f$ as defined in (\ref{cMPNNdef}), and $\Lambda_{\Theta_G}^{(l+1)}$ and $\Lambda_{\Theta_W}^{(l+1)}$ are defined in Definition \ref{def:LayerMapping}.  Here,
 \begin{equation}
 \label{eq:defDl2}
 \begin{aligned}
& Q^{(l+1)}  = L_{\Psi^{(l+1)}} \Bigg(
   4 \frac{\varepsilon_1}{\sqrt{N}\mathrm{d}_{min}^2} \|W\|_\infty(L_{\Phi^{(l+1)}}\|f^{(l)}\|_\infty + \|\Phi^{(l+1)}(0,0)\|_\infty) \\
   &+  N^{-\frac{1}{2(D_\chi+1)}}\Bigg(
2\Big(
\frac{\|W\|_\infty}{\mathrm{d}_{min}} L_{\Phi^{(l+1)}} L_{f^{(l)}} + \ \frac{\LipW}{\mathrm{d}_{min}}     \big( L_{\Phi^{(l+1)}}\|f^{(l)}\|_\infty + \|\Phi^{(l+1)}(0,0)\|_\infty \big)  \Big)
 \\
& + \frac{C_\chi}{\sqrt{2}} \Big(\frac{\|W\|_\infty}{\mathrm{d}_{min}} \big( L_{\Phi^{(l+1)}} \|f^{(l)}\|_\infty + \|\Phi^{(l+1)}(0,0)\|_\infty \big) \Big) \\
& \cdot \sqrt{\log(C_\chi) + \frac{D_\chi}{2(D_\chi+1)} \log(N) + \log(2/p)}  \Bigg)
\Bigg),
 \end{aligned}
 \end{equation}
 and  $\d$ is defined in (\ref{eq:distGraphMetric}).
\end{corollary}

\begin{proof}

Let $l=0,\ldots, T-1$. We have, 
\[
\begin{aligned}
& \Big(\d\big( \Lambda_{\Theta_G}^{(l+1)}(S^X f^{(l)}),  \Lambda_{\Theta_W}^{(l+1)}(f^{(l)})\big)\Big)^2 \\
& =  \| \Lambda_{\Theta_G}^{(l+1)}(S^X f^{(l)}) - S^X \Lambda_{\Theta_W}^{(l+1)}(f^{(l)})  \|_{2;\infty}^2  \\ & = 
 \frac{1}{N} \sum_{i=1}^N \| \Lambda_{\Theta_G}^{(l+1)}(S^X f^{(l)}) (X_i) - S^X \Lambda_{\Theta_W}^{(l+1)}(f^{(l)}) (X_i)\|_\infty^2 \\
& = \frac{1}{N} \sum_{i=1}^N  \Big\| \Psi^{(l+1)} \Big(f^{(l)}(X_i), M_G \big(\Phi^{(l+1)} (S^Xf^{(l)},S^Xf^{(l)})  \big)  (X_i)\Big) \\ & -  \Psi^{(l+1)} \Big(f^{(l)}(X_i) , M_W\big(\Phi^{(l+1)} (f^{(l)},f^{(l)}) \big) (X_i)  \Big) \Big\|_\infty^2  \\
& =  \frac{1}{N} \sum_{i=1}^N \Big\| \Psi^{(l+1)} \Big(f^{(l)}(X_i), M_X\big(\Phi^{(l+1)} (f^{(l)},f^{(l)} ) \big) (X_i)\Big)   \\ & - \Psi^{(l+1)} \Big(f^{(l)} (X_i) , M_W\big(\Phi^{(l+1)} (f^{(l)},f^{(l)}) \big) (X_i) \Big) \Big\|_\infty^2  \\
& \leq L_{\Psi^{(l+1)}}^2\Bigg(
   4 \frac{\varepsilon_1}{\sqrt{N}\mathrm{d}_{min}^2} \|W\|_\infty(L_{\Phi^{(l+1)}}\|f^{(l)}\|_\infty + \|\Phi^{(l+1)}(0,0)\|_\infty)
   \\
   & +  N^{-\frac{1}{2(D_\chi+1)}}\Bigg(
2\Big(
\frac{\|W\|_\infty}{\mathrm{d}_{min}} L_{\Phi^{(l+1)}} L_{f^{(l)}} + \ \frac{\LipW}{\mathrm{d}_{min}}     \big( L_{\Phi^{(l+1)}}\|f^{(l)}\|_\infty + \|\Phi^{(l+1)}(0,0)\|_\infty \big)  \Big)
 \\
& + \frac{C_\chi}{\sqrt{2}} \Big(\frac{\|W\|_\infty}{\mathrm{d}_{min}} \big( L_{\Phi^{(l+1)}} \|f^{(l)}\|_\infty + \|\Phi^{(l+1)}(0,0)\|_\infty \big) \Big)
\\
& \cdot  \sqrt{\log(C_\chi) + \frac{D_\chi}{2(D_\chi+1)} \log(N) + \log(2/p)}  \Bigg)
\Bigg)^2,
\end{aligned}
 \]
where the final inequality holds,  by applying  Corollary \ref{cor:Uniform3}. 
\end{proof}

\begin{lemma}
\label{lemma:C5}
Let $(\chi,d, \P) $ be a metric-measure space and
$W$ be a kernel s.t.  Assumptions \ref{ass:graphon}.\ref{ass:graphon1}-\ref{ass:graphon12}. are satisfied. 
Let $p \in (0, \frac{1}{2})$. Consider a graph $(G, \mathbf{f}) \sim (W, f)$ with $N$ nodes and corresponding graph features, where $N$ satisfies (\ref{eq:largeN}).
Denote, for $l=1,\ldots,T$, 
\[
 \varepsilon^{(l)} = \d( \Theta^{(l)}_{G}(\mathbf{f}), \Theta^{(l)}_{W}(f)  ),
\]
and $\varepsilon^{(0)} = \d(\mathbf{f}, f)$.
If the event $\mathcal{F}_{\rm Lip}^p$ from Lemma \ref{lemma:C2} occurs, then,  for every MPNN $\Theta$ satisfying Assumption \ref{ass:graphon}.\ref{ass:graphon4}. and $f:\chi \to \mathbb{R}^F$ with Lipschitz constant $L_f$, the following recurrence relation holds: 
\[
\begin{aligned}
\varepsilon^{(l)} \leq K^{(l+1)} \varepsilon^{(l)} + Q^{(l+1)}
\end{aligned}
\]
for $l=0, \ldots, T-1$. Here, $Q^{(l+1)}$ is defined in (\ref{eq:defDl2}), and
\begin{equation}
    \label{eq:lemmaC6}
K^{(l+1)}  = \sqrt{(L_{\Psi^{(l+1)}})^2 
  + 
 \frac{8\|W\|_\infty^2}{\cmin^2} (L_{\Phi^{(l+1)}})^2 (L_{\Psi^{(l+1)}})^2}.
 \end{equation}
\end{lemma}

\begin{proof}
In the event $\mathcal{F}_{\rm Lip}^p$,
 by Corollary \ref{cor:C5}, we have for every MPNN $\Theta$ satisfying Assumption \ref{ass:graphon}.\ref{ass:graphon4}. and $f:\chi \to \mathbb{R}^F$ with Lipschitz constant $L_f$, 
\begin{equation}
    \label{eq:lemmaC6-00} \d\left( \Lambda_{\Theta_G}^{(l+1)}(S^X f^{(l)}),  \Lambda_{\Theta_W}^{(l+1)}(f^{(l)})\right)
    \leq Q^{(l+1)} 
\end{equation}
 for $l=0,\ldots, T-1$, and
\begin{equation}
    \label{eq:lemmaC6-0} 
|\mathrm{d}_X(x)| \geq \frac{\mathrm{d}_{\mathrm{min}}}{2} 
\end{equation}
for all $x \in \chi$. Let $l=0,\ldots,T-1$.
We have
\begin{equation}
\label{eq:lemmaC6-1}
\begin{aligned}
  & \d( \Theta^{(l+1)}_{G}(\mathbf{f})  , \Theta^{(l+1)}_{W}(f)  ) \\ &  = \| \Theta^{(l+1)}_{G}(\mathbf{f}) - S^X \Theta^{(l+1)}_{W}(f) \|_{2; \infty} \\
    & \leq \| \Theta^{(l+1)}_{G}(\mathbf{f}) - \Lambda^{(l+1)}_{\Theta_G}(S^X f^{(l)})  \|_{2; \infty}   +  \| \Lambda^{(l+1)}_{\Theta_G}(S^X f^{(l)}) - S^X\Theta^{(l+1)}_{\Theta_W}(f)  \|_{2;\infty} \\
        & = \| \Lambda^{(l+1)}_G(\mathbf{f}^{(l)}) - \Lambda^{(l+1)}_G(S^X f^{(l)})  \|_{2; \infty}  +  \| \Lambda^{(l+1)}_{\Theta_G}(S^X f^{(l)}) - S^X\Lambda^{(l+1)}_{\Theta_W}(f^{(l)})  \|_{2; \infty} \\
& \leq  \| \Lambda^{(l+1)}_{\Theta_G}(\mathbf{f}^{(l)}) - \Lambda^{(l+1)}_{\Theta_G}(S^X f^{(l)})\|_{2; \infty}  + Q^{(l+1)}.
\end{aligned}
\end{equation}
We bound the first term on the right-hand-side of (\ref{eq:lemmaC6-1}) as follows.  
\begin{equation}
\label{eq:lemmaC6-2}  
\begin{aligned}
 &\| \Lambda^{(l+1)}_{\Theta_G}(\mathbf{f}^{(l)}) - \Lambda^{(l+1)}_{\Theta_G}(S^X f^{(l)})\|_{2; \infty}^2
 \\
 & =  \frac{1}{N} \sum_{i=1}^N  \Big\|
 \Psi^{(l+1)}  \Big( \mathbf{f}^{(l)}_i, M_G \big( \Phi^{(l+1)} (\mathbf{f}^{(l)}, \mathbf{f}^{(l)}) \big)(X_i) \Big)
\\ 
& - \Psi^{(l+1)} \Big( (S^X f^{(l)})_i, M_G\big( \Phi^{(l+1)} (S^X f^{(l)}, S^X f^{(l)} ) \big)(X_i) \Big)
 \Big\|_\infty^2 \\
 & \leq \frac{1}{N}(L_{\Psi^{(l+1)}})^2
\sum_{i=1}^N   \Big\|
  \Big( \mathbf{f}^{(l)}_i, M_G\big( \Phi^{(l+1)} (\mathbf{f}^{(l)}, \mathbf{f}^{(l)}) \big)(X_i) \Big)
  \\
& -  \Big( (S^X f^{(l)})_i, M_G\big( \Phi^{(l+1)} (S^X f^{(l)}, S^X f^{(l)} ) \big)(X_i) \Big)
\Big\|^2_\infty  \\
& \leq \frac{1}{N}(L_{\Psi^{(l+1)}})^2
\Big( \sum_{i=1}^N   \Big\|
  \mathbf{f}^{(l)}_i - (S^X f^{(l)})_i  \Big\|^2_\infty
  \\ & + \sum_{i=1}^N   \Big\|
 M_G\big( \Phi^{(l+1)} (\mathbf{f}^{(l)}, \mathbf{f}^{(l)}) \big)(X_i) 
-  M_G\big( \Phi^{(l+1)} (S^X f^{(l)}, S^X f^{(l)} ) \big)(X_i)
 \Big\|^2_\infty \Big)\\
 & \leq (L_{\Psi^{(l+1)}})^2 \Big( \big(\d(\mathbf{f}^{(l)}, f^{(l)})\big)^2\\
&  + \frac{1}{N} \sum_{i=1}^N   \big\|
 M_G\big( \Phi^{(l+1)} (\mathbf{f}^{(l)}, \mathbf{f}^{(l)}) \big)(X_i)
-  M_G\big( \Phi^{(l+1)} (S^X f^{(l)}, S^X f^{(l)} )\big)(X_i) \big\|^2_\infty
 \Big)
 \\
 & \leq (L_{\Psi^{(l+1)}})^2 \Big(
(\varepsilon^{(l)})^2 \\
&  + \frac{1}{N} \sum_{i=1}^N   \big\|
 M_G\big( \Phi^{(l+1)} (\mathbf{f}^{(l)}, \mathbf{f}^{(l)}) \big)(X_i)
-  M_G\big( \Phi^{(l+1)} (S^X f^{(l)}, S^X f^{(l)} )\big)(X_i) \big\|^2_\infty
 \Big).
\end{aligned}
\end{equation}
Now, for every $i=1, \ldots, N$, we have 
\begin{equation}
    \label{eq:lemmaC6-3}
\begin{aligned}
 & \Big\|
 M_G \Big( \Phi^{(l+1)} \big(\mathbf{f}^{(l)}, \mathbf{f}^{(l)} \big)  \Big)(X_i)
-  M_G \Big(  \Phi^{(l+1)} \big( S^X f^{(l)}, S^X f^{(l)} \big) \Big)(X_i) \Big\|^2_\infty  \\
& = \Big\|
\frac{1}{N} \sum_{j=1}^N \frac{W(X_i, X_j)}{\mathrm{d}_X(X_i)} \Phi^{(l+1)} \big(\mathbf{f}^{(l)}(X_i), \mathbf{f}^{(l)}(X_j) \big) \\&- 
\frac{1}{N}\sum_{j=1}^N \frac{W(X_i, X_j)}{\mathrm{d}_X(X_i)} \Phi^{(l+1)} \big(S^X f^{(l)}(X_i), S^X f^{(l)}(X_j) \big)
\Big\|^2_\infty \\
& = \Big\|
\frac{1}{N} \sum_{j=1}^N \frac{W(X_i, X_j)}{\mathrm{d}_X(X_i)} \Big(\Phi^{(l+1)} \big(\mathbf{f}^{(l)}(X_i), \mathbf{f}^{(l)}(X_j) \big) -  \Phi^{(l+1)} \big(S^X f^{(l)}(X_i), S^X f^{(l)}(X_j) \big)
\Big)
\Big\|^2_\infty \\
& \leq \frac{1}{N^2}  \sum_{j=1}^N \Big| \frac{W(X_i, X_j)}{\mathrm{d}_X(X_i)}\Big|^2  \sum_{j=1}^N \Big\| \Big(\Phi^{(l+1)} \big(\mathbf{f}^{(l)}(X_i), \mathbf{f}^{(l)}(X_j) \big) -  \Phi^{(l+1)} \big(S^X f^{(l)}(X_i), S^X f^{(l)}(X_j) \big)
\Big) \Big\|_\infty^2 \\
& \leq \frac{4\|W\|_\infty^2}{\cmin^2} \frac{1}{N}\sum_{j=1}^N \Big\| \Big(\Phi^{(l+1)} \big(\mathbf{f}^{(l)}(X_i), \mathbf{f}^{(l)}(X_j) \big) -  \Phi^{(l+1)} \big(S^X f^{(l)}(X_i), S^X f^{(l)}(X_j) \big)
\Big)\Big\|_\infty^2,
\end{aligned}
\end{equation}
where the second-to-last inequality holds by the Cauchy–Schwarz inequality  and the last inequality holds by (\ref{eq:lemmaC6-0}). 
Now, for the term on the right-hand-side of (\ref{eq:lemmaC6-3}), we have
\begin{equation}
    \label{eq:lemmaC6-4}
    \begin{aligned}
 & \frac{1}{N}\sum_{j=1}^N \Big\|\Phi^{(l+1)} \big(\mathbf{f}^{(l)}(X_i), \mathbf{f}^{(l)}(X_j) \big) -  \Phi^{(l+1)} \big(S^X f^{(l)}(X_i), S^X f^{(l)}(X_j) \big)
 \Big\|_\infty^2 \\& \leq (L_{\Phi^{(l+1)}})^2 \frac{1}{N}\sum_{j=1}^N \Big( \big\|\mathbf{f}^{(l)}(X_i) - S^X f^{(l)}(X_i)\big\|^2_\infty + \big\| \mathbf{f}^{(l)}(X_j) - S^X f^{(l)}(X_j) \big\|^2_\infty \Big) \\
& \leq (L_{\Phi^{(l+1)}})^2  \big\|\mathbf{f}^{(l)}(X_i) - S^X f^{(l)}(X_i)\big\|^2_\infty +  (L_{\Phi^{(l+1)}})^2 (\varepsilon^{(l)})^2.
\end{aligned}
\end{equation}
Hence, by inserting (\ref{eq:lemmaC6-4}) into (\ref{eq:lemmaC6-3}) and (\ref{eq:lemmaC6-3}) into (\ref{eq:lemmaC6-2}), we have
\[
\begin{aligned}
& \| \Lambda^{(l+1)}_{\Theta_G}(\mathbf{f}^{(l)}) - \Lambda^{(l+1)}_{\Theta_G}(S^X f^{(l)})\|_{2; \infty}^2 \\
 & \leq (L_{\Psi^{(l+1)}})^2 \Big(
(\varepsilon^{(l)})^2
  + \frac{1}{N} \sum_{i=1}^N  \big\|
  M_G\big( \Phi^{(l)}  (\mathbf{f}^{(l)}, \mathbf{f}^{(l)} ) \big) (X_i)
-  M_G \big( \Phi^{(l)} (S^X f^{(l)}, S^X f^{(l)} ) \big)(X_i) \big\|_\infty^2
 \Big) \\
  & \leq (L_{\Psi^{(l+1)}})^2 \Big(
(\varepsilon^{(l)})^2
  +\frac{4\|W\|_\infty^2}{\cmin^2} (L_{\Phi^{(l+1)}})^2 \big( \frac{1}{N} \sum_{i=1}^N  
 \big\|\mathbf{f}^{(l)}(X_i) - S^X f^{(l)}(X_i)\big\|^2_\infty+
 (\varepsilon^{(l)})^2 \big) \Big)\\
& \leq (L_{\Psi^{(l+1)}})^2 \Big(
(\varepsilon^{(l)})^2
  +\frac{4\|W\|_\infty^2}{\cmin^2} (L_{\Phi^{(l+1)}})^2 \big(  (\varepsilon^{(l)})^2 +
 (\varepsilon^{(l)})^2 \big) \Big)\\
& \leq (L_{\Psi^{(l+1)}})^2 \Big(
(\varepsilon^{(l)})^2
  + 
 \frac{8\|W\|_\infty^2}{\cmin^2} (L_{\Phi^{(l+1)}})^2 (\varepsilon^{(l)})^2 \Big).
\end{aligned}
\]
By inserting this into (\ref{eq:lemmaC6-1}), we conclude 
\[
\begin{aligned}
  \d( \Theta^{(l+1)}_{G}(\mathbf{f})  , \Theta^{(l+1)}_{W}(f)  )  \leq 
(L_{\Psi^{(l+1)}})^2 
 \big(1 + 
 \frac{8\|W\|_\infty^2}{\cmin^2} (L_{\Phi^{(l+1)}})^2 \big) (\varepsilon^{(l)})^2 + Q^{(l+1)}.
\end{aligned}
\]
\end{proof}

\begin{corollary}
\label{cor:c6}
Let $(\chi,d, \P) $ be a metric-measure space and
$W$ be a kernel s.t.  Assumptions \ref{ass:graphon}.\ref{ass:graphon1}-\ref{ass:graphon12}. are satisfied. 
Let $p \in (0, \frac{1}{2})$.  Consider a graph $(G, \mathbf{f}) \sim (W, f)$ with $N$ nodes and corresponding graph features, where $N$ satisfies (\ref{eq:largeN}). If the event $\mathcal{F}_{\rm Lip}^p$ from Lemma \ref{lemma:C2} occurs, then,   for every MPNN $\Theta$ satisfying Assumption \ref{ass:graphon}.\ref{ass:graphon4}. and every Lipschitz continuous $f:\chi \to \mathbb{R}^F$ with Lipschitz constant $L_f$,
\begin{equation*}
    \d \big(\Theta_G(f(X))  ,\Theta_W(f) \big) \leq \sum_{l=1}^{T} Q^{(l)} \prod_{l' = l+1}^{T} K^{(l')},
    \end{equation*}
  where $Q^{(l)}$ and $K^{(l')}$ are defined in (\ref{eq:defDl2}) and (\ref{eq:lemmaC6}), respectively. 
\end{corollary}

\begin{proof}
By Lemma \ref{lemma:C5},
for every MPNN $\Theta$ satisfying Assumption \ref{ass:graphon}.\ref{ass:graphon4}. and every Lipschitz continuous $f:\chi \to \mathbb{R}^F$ with Lipschitz constant $L_f$, the recurrence relation 
\[
\varepsilon^{(l+1)} \leq K^{(l+1)}\varepsilon^{(l)} + Q^{(l+1)}
\]
holds  for $l=0, \ldots, T-1$.
We use that $\varepsilon^{(0)} = 0$ and $\varepsilon^{(T)} =  \d \big(\Theta_G(f(X))  ,\Theta_W(f) \big)$,
and solve this recurrence relation  by Lemma \ref{lemma:RecRecGen} to finish the proof.
\end{proof}

\begin{theorem}
\label{thm:C7}
Let $(\chi,d, \P) $ be a metric-measure space and
$W$ be a kernel s.t.  Assumptions \ref{ass:graphon}.\ref{ass:graphon1}-\ref{ass:graphon12}. are satisfied. 
Let $p \in (0, \frac{1}{2})$. Consider a graph $(G, \mathbf{f}) \sim (W, f)$ with $N$ nodes and corresponding graph features, where $N$ satisfies (\ref{eq:largeN}). If the event $\mathcal{F}_{\rm Lip}^p$ from Lemma \ref{lemma:C2} occurs, then for every MPNN $\Theta$ satisfying Assumption \ref{ass:graphon}.\ref{ass:graphon4} and $f:\chi \to \mathbb{R}^{F}$ with Lipschitz constant $L_f$,
\[
\begin{aligned}
& \d \big(\Theta_G(f(X))  ,\Theta_W(f) \big) \\ & \leq N^{-\frac{1}{2}}\left(\Omega_1 + \Omega_2 \log(2/p) +   \Omega_3 \|f\|_\infty +    \Omega_4 \|f\|_\infty\log(2/p)\right)  \\
& + N^{-\frac{1}{2(D_\chi + 1)}}\big(\Omega_5  + \Omega_6 \|f\|_\infty + \Omega_7 L_f \big)\\
& + N^{-\frac{1}{2(D_\chi + 1)}}\sqrt{\log(C_\chi) + \frac{D_\chi}{2(D_\chi+1)} \log(N) + \log(2/p)}\cdot(\Omega_8  + \Omega_9 \|f\|_\infty),
\end{aligned}
\] 
where $\Omega_i$, for $i=1, \ldots, 9$, are constants of the MPNN $\Theta$, defined in (\ref{eq:defConstantsUniform}), which depend only on the Lipschitz constants of the message and update functions $\{L_{\Phi^{(l)}},L_{\Psi^{(l)}}\}_{l=1}^T$, and the formal biases $\{\|\Phi^{(l)}(0,0)\|_\infty\}_{l=1}^T$. 
\end{theorem}
\begin{proof}
In the event $\mathcal{F}_{\rm Lip}^p$, by Corollary \ref{cor:c6},  for every MPNN $\Theta$ satisfying Assumption \ref{ass:graphon}.\ref{ass:graphon4}. and $f:\chi \to \mathbb{R}^F$ with Lipschitz constant $L_f$,
\begin{equation}
\label{eq:C8-1}
    \d \big(\Theta_G(f(X))  ,\Theta_W(f) \big) \leq \sum_{l=1}^{T} Q^{(l)} \prod_{l' = l+1}^{T} K^{(l')},
    \end{equation} where  
    \[
    \begin{aligned}
&Q^{(l)} =  L_{\Psi^{(l)}} \Bigg(
   4 \frac{\varepsilon_1}{\sqrt{N}\mathrm{d}_{min}^2} \|W\|_\infty(L_{\Phi^{(l)}}\|f^{(l-1)}\|_\infty + \|\Phi^{(l)}(0,0)\|_\infty) 
   \\
   & +  N^{-\frac{1}{2(D_\chi+1)}}\Bigg(
2\Big(
\frac{\|W\|_\infty}{\mathrm{d}_{min}} L_{\Phi^{(l)}} L_{f^{(l-1)}} +  \frac{\LipW}{\mathrm{d}_{min}}     \big( L_{\Phi^{(l)}} \|f^{(l-1)}\|_\infty + \|\Phi^{(l)}(0,0)\|_\infty \big)  \Big)
 \\
& + \frac{C_\chi}{\sqrt{2}} \Big(\frac{\|W\|_\infty}{\mathrm{d}_{min}} \big( L_{\Phi^{(l)}} \|f^{(l-1)}\|_\infty + \|  \Phi^{(l)}(0,0)\|_\infty \big) \Big) 
\\
& \cdot
 \sqrt{\log(C_\chi) + \frac{D_\chi}{2(D_\chi+1)} \log(N) + \log(2/p)} \Bigg)
\Bigg),
 \end{aligned}
    \]
   and 
    \[
(K^{(l')})^2  = (L_{\Psi^{(l')}})^2 
  + 
   \frac{8\|W\|_\infty^2}{\cmin^2} (L_{\Phi^{(l')}})^2 (L_{\Psi^{(l')}})^2.
\]
We plug the definition of $Q^{(l)}$ into the right-hand-side of (\ref{eq:C8-1}), to get
\begin{equation}
 \label{eq:thmC8-1}
   \begin{aligned}
& \d \big(\Theta_G(f(X))  ,\Theta_W(f) \big) 
\\ & \leq \sum_{l=1}^{T} L_{\Psi^{(l)}} \Bigg(
   4 \frac{\varepsilon_1}{\sqrt{N}\mathrm{d}_{min}^2} \|W\|_\infty(L_{\Phi^{(l)}} \|f^{l-1)}\|_\infty + \|\Phi^{(l)}(0,0)\|_\infty) 
   \\
   & +  N^{-\frac{1}{2(D_\chi+1)}}\Big(
\frac{2\|W\|_\infty}{\mathrm{d}_{min}} L_{\Phi^{(l)}} L_{f^{(l-1)}} 
 + \frac{2\LipW}{\mathrm{d}_{min}}     ( L_{\Phi^{(l)}} \|f^{(l-1)}\|_\infty + \|\Phi^{(l)}(0,0)\|_\infty ) 
 \\
& + \frac{C_\chi}{\sqrt{2}} \big(\frac{\|W\|_\infty}{\mathrm{d}_{min}} ( L_{\Phi^{(l)}} \|f^{(l-1)}\|_\infty + \|\Phi^{(l)}(0,0)\|_\infty ) \big) \\
& 
 \sqrt{\log(C_\chi) + \frac{D_\chi}{2(D_\chi+1)} \log(N) + \log(2/p)} \Big)
\Bigg) \prod_{l' = l+1}^{T} K^{(l')}.
\end{aligned}
\end{equation}
By Lemma \ref{lemma:RecRelNorm}, we have
\begin{equation}
\label{eq:thmC8-infnorm}
||f^{(l)}||_{\infty} \leq B_1^{(l)}+B_2^{(l)}||f||_{\infty},
\end{equation}
where $B_1^{(l)}$, $B_2^{(l)}$ are independent of $f$.
Furthermore, we have 
\begin{equation}
\label{eq:thmC8-lip}
L_{f^{(l)}} \leq Z^{(l)}_1 + Z^{(l)}_2\|f\|_\infty + Z^{(l)}_3 L_f,
\end{equation}
where  $Z_1^{(l)}$, $Z_2^{(l)}$ and $Z_3^{(l)}$ are independent of $f$, and defined in (\ref{eq:z1z2z3}).  
We plug the bound of $L_{f^{(l-1)}}$ from (\ref{eq:thmC8-lip}) into (\ref{eq:C8-1}) 
\[
\begin{aligned}
& \d \big(\Theta_G(f(X))  ,\Theta_W(f) \big) 
\\ & \leq \sum_{l=1}^{T} L_{\Psi^{(l)}} \Bigg(
   4 \frac{\varepsilon_1}{\sqrt{N}\mathrm{d}_{min}^2} \|W\|_\infty(L_{\Phi^{(l)}} \|f^{(l-1)}\|_\infty + \|\Phi^{(l)}(0,0)\|_\infty) +  N^{-\frac{1}{2(D_\chi+1)}}
   \\
 & \cdot   \Big(
\frac{2\|W\|_\infty}{\mathrm{d}_{min}} L_{\Phi^{(l)}} (
 Z^{(l-1)}_1 + Z^{(l-1)}_2\|f\|_\infty + Z^{(l-1)}_3 L_f
) +  \frac{2\LipW}{\mathrm{d}_{min}}     ( L_{\Phi^{(l)}}\|f^{(l-1)}\|_\infty + \|\Phi^{(l)}(0,0)\|_\infty ) 
 \\
& + \frac{C_\chi}{\sqrt{2}} \big(\frac{\|W\|_\infty}{\mathrm{d}_{min}} ( L_{\Phi^{(l)}} \|f^{(l-1)}\|_\infty + \|\Phi^{(l)}(0,0)\|_\infty ) \big) \\
& \cdot \sqrt{\log(C_\chi) + \frac{D_\chi}{2(D_\chi+1)} \log(N) + \log(2/p)}  \Big)
\Bigg) \prod_{l' = l+1}^{T} K^{(l')}.
\end{aligned}
\]
We insert the bound of $\|f^{(l-1)}\|_\infty$ from (\ref{eq:thmC8-infnorm}) in the above expression, to get 
\begin{equation}
    \label{eq:thmC8-2}
    \begin{aligned}
& \leq \sum_{l=1}^{T} L_{\Psi^{(l)}} \Bigg(
   4 \frac{\varepsilon_1}{\sqrt{N}\mathrm{d}_{min}^2} \|W\|_\infty\big(L_{\Phi^{(l)}}(B_1^{(l-1)}+B_2^{(l-1)}||f||_{\infty}) + \|\Phi^{(l)}(0,0)\|_\infty \big) +  N^{-\frac{1}{2(D_\chi+1)}}
   \\
 & \cdot   \Big(
\frac{2\|W\|_\infty}{\mathrm{d}_{min}} L_{\Phi^{(l)}} (
 Z^{(l-1)}_1 + Z^{(l-1)}_2\|f\|_\infty + Z^{(l-1)}_3 L_f
) +  \frac{2\LipW}{\mathrm{d}_{min}}     \big( B_1^{(l-1)}+B_2^{(l-1)}||f||_{\infty} \big)  
 \\
& + \frac{C_\chi}{\sqrt{2}} \big(\frac{\|W\|_\infty}{\mathrm{d}_{min}} \big( L_{\Phi^{(l)}} (B_1^{(l-1)}+B_2^{(l-1)}||f||_{\infty}) + \|\Phi^{(l)}(0,0)\|_\infty \big) \big)
\\
&  \sqrt{\log(C_\chi) + \frac{D_\chi}{2(D_\chi+1)} \log(N) + \log(2/p)} \Big)
\Bigg) \prod_{l' = l+1}^{T} K^{(l')}.
\end{aligned}
\end{equation}

We insert the bound for $\varepsilon_1$, defined in (\ref{eq:eps1-2}) as 
\[\varepsilon_1 = \zeta  \Big(\cl  \big(\sqrt{\log (C_\chi)} +  \sqrt{D_\chi} \big) +\big(\sqrt{2}\cmax + \cl \big) \sqrt{\log 2/p} \Big),\]
into (\ref{eq:thmC8-2})
to get
\[
\begin{aligned}
& \leq \sum_{l=1}^{T} L_{\Psi^{(l)}} \Bigg(
   4 \frac{\zeta  \Big(\cl  \big(\sqrt{\log (C_\chi)} +  \sqrt{D_\chi} \big) +\big(\sqrt{2}\cmax + \cl \big) \sqrt{\log 2/p} \Big)}{\sqrt{N}\mathrm{d}_{min}^2}
   \\
   & \cdot \|W\|_\infty\big(L_{\Phi^{(l)}}(B_1^{(l-1)}+B_2^{(l-1)}||f||_{\infty}) + \|\Phi^{(l)}(0,0)\|_\infty \big)
   \\
 &+  N^{-\frac{1}{2(D_\chi+1)}}  \Big(
\frac{2\|W\|_\infty}{\mathrm{d}_{min}} L_{\Phi^{(l)}} (
 Z^{(l-1)}_1 + Z^{(l-1)}_2\|f\|_\infty + Z^{(l-1)}_3 L_f
)\\& +  \frac{2\LipW}{\mathrm{d}_{min}}     \big( B_1^{(l-1)}+B_2^{(l-1)}||f||_{\infty} \big) 
 \\
& + \frac{C_\chi}{\sqrt{2}} \big(\frac{\|W\|_\infty}{\mathrm{d}_{min}} \big( L_{\Phi^{(l)}} (B_1^{(l-1)}+B_2^{(l-1)}||f||_{\infty}) + \|\Phi^{(l)}(0,0)\|_\infty \big) \big)
\\
& \cdot  \sqrt{\log(C_\chi) + \frac{D_\chi}{2(D_\chi+1)} \log(N) + \log(2/p)} \Big)
\Bigg) \prod_{l' = l+1}^{T} K^{(l')}.
\end{aligned}
\]
Then, rearranging the terms yields
\[
\begin{aligned}
& = \sum_{l=1}^{T} L_{\Psi^{(l)}}  
4 \frac{\zeta \cl  \big(\sqrt{\log (C_\chi)} +  \sqrt{D_\chi} \big)}{\sqrt{N}\mathrm{d}_{min}^2} \|W\|_\infty\big(L_{\Phi^{(l)}}  B_1^{(l-1)}+ \|\Phi^{(l)}(0,0)\|_\infty \big) \prod_{l' = l+1}^{T} K^{(l')} \\
& + \sum_{l=1}^{T} L_{\Psi^{(l)}} 
4 \frac{\zeta (\sqrt{2}\cmax + \cl)  \sqrt{\log 2/p}}{\sqrt{N}\mathrm{d}_{min}^2} \|W\|_\infty\big( L_{\Phi^{(l)}}B_1^{(l-1)} + \|\Phi^{(l)}(0,0)\|_\infty \big) 
\prod_{l' = l+1}^{T} K^{(l')}\\ 
& + \sum_{l=1}^{T} L_{\Psi^{(l)}}  
4 \frac{\zeta \cl  \big(\sqrt{\log (C_\chi)} +  \sqrt{D_\chi} \big)}{\sqrt{N}\mathrm{d}_{min}^2} \|W\|_\infty\big(L_{\Phi^{(l)}} B_2^{(l-1)}\|f\|_\infty \big) \prod_{l' = l+1}^{T} K^{(l')} \\
& +  \sum_{l=1}^{T} L_{\Psi^{(l)}} 
4 \frac{\zeta (\sqrt{2}\cmax + \cl)  \sqrt{\log 2/p}}{\sqrt{N}\mathrm{d}_{min}^2} \|W\|_\infty\big(L_{\Phi^{(l)}} B_2^{(l-1)}\|f\|_\infty \big)
\prod_{l' = l+1}^{T} K^{(l')}
\\
& + \sum_{l=1}^{T} L_{\Psi^{(l)}} N^{-\frac{1}{2(D_\chi+1)}} \Big(
\frac{2\|W\|_\infty}{\mathrm{d}_{min}} L_{\Phi^{(l)}} 
 Z^{(l-1)}_1 +  \frac{2\LipW}{\mathrm{d}_{min}} B_1^{(l-1)}\Big)\prod_{l' = l+1}^{T} K^{(l')} 
 \\
 & + \sum_{l=1}^{T} L_{\Psi^{(l)}}  N^{-\frac{1}{2(D_\chi+1)}} \Big( 
\frac{2\|W\|_\infty}{\mathrm{d}_{min}} L_{\Phi^{(l)}} 
 Z^{(l-1)}_2 \|f\|_\infty +  \frac{2\LipW}{\mathrm{d}_{min}} B_2^{(l-1)}\|f\|_\infty\Big)\prod_{l' = l+1}^{T} K^{(l')} \\
 & + \sum_{l=1}^{T} L_{\Psi^{(l)}}  N^{-\frac{1}{2(D_\chi+1)}}  2 
\frac{\|W\|_\infty}{\mathrm{d}_{min}} L_{\Phi^{(l)}} 
 Z^{(l-1)}_3 L_f 
\prod_{l' = l+1}^{T} K^{(l')}
\\ 
 & + \sum_{l=1}^{T} L_{\Psi^{(l)}} N^{-\frac{1}{2(D_\chi+1)}}\frac{C_\chi}{\sqrt{2}} \frac{\|W\|_\infty}{\mathrm{d}_{min}}( L_{\Phi^{(l)}}   B_1^{(l-1)} + \|\Phi^{(l)}(0,0)\|_\infty ) 
 \\
 &\cdot  \sqrt{\log(C_\chi) + \frac{D_\chi}{2(D_\chi+1)} \log(N) + \log(2/p)}\Big)\prod_{l' = l+1}^{T} K^{(l')}
 \\
 & +\sum_{l=1}^{T} L_{\Psi^{(l)}}  N^{-\frac{1}{2(D_\chi+1)}} \frac{C_\chi}{\sqrt{2}} \frac{\|W\|_\infty}{\mathrm{d}_{min}} L_{\Phi^{(l)}} B_2^{(l-1)}\|f\|_\infty \\  & \cdot \sqrt{\log(C_\chi) + \frac{D_\chi}{2(D_\chi+1)} \log(N) + \log(2/p)}\Big) \prod_{l' = l+1}^{T} K^{(l')}
 \end{aligned}\]\[\begin{aligned}
&  =:
\Omega_1 \frac{1}{\sqrt{N}} + \Omega_2 \frac{\log(2/p)}{\sqrt{N}} +   \Omega_3 \frac{\|f\|_\infty}{\sqrt{N}} +    \Omega_4 \frac{\|f\|_\infty\log(2/p)}{\sqrt{N}}  \\
& + N^{-\frac{1}{2(D_\chi + 1)}}\big(\Omega_5  + \Omega_6 \|f\|_\infty + \Omega_7 L_f \big)\\
& + N^{-\frac{1}{2(D_\chi + 1)}}\sqrt{\log(C_\chi) + \frac{D_\chi}{2(D_\chi+1)} \log(N) + \log(2/p)}\cdot(\Omega_8  + \Omega_9 \|f\|_\infty),
\end{aligned}
\]
 
where we define 
\begin{equation}
    \label{eq:defConstantsUniform}
    \begin{aligned}
 &   \Omega_1 =  \sum_{l=1}^{T} L_{\Psi^{(l)}}  
4 \frac{\zeta \cl  \big(\sqrt{\log (C_\chi)} +  \sqrt{D_\chi} \big)}{\mathrm{d}_{min}^2} \|W\|_\infty\big(L_{\Phi^{(l)}} B_1^{(l-1)}+ \|\Phi^{(l)}(0,0)\|_\infty \big) \prod_{l' = l+1}^{T} K^{(l')} \\
& \Omega_2 =  \sum_{l=1}^{T} L_{\Psi^{(l)}} 
4 \frac{\zeta (\sqrt{2}\cmax + \cl) }{\mathrm{d}_{min}^2} \|W\|_\infty\big(L_{\Phi^{(l)}} B_1^{(l-1)} + \|\Phi^{(l)}(0,0)\|_\infty \big) 
\prod_{l' = l+1}^{T} K^{(l')}
\\
& \Omega_3 = \sum_{l=1}^{T} L_{\Psi^{(l)}}  
4 \frac{\zeta \cl  \big(\sqrt{\log (C_\chi)} +  \sqrt{D_\chi} \big)}{\mathrm{d}_{min}^2} \|W\|_\infty\big(L_{\Phi^{(l)}} B_2^{(l-1)} \big) \prod_{l' = l+1}^{T} K^{(l')}
\\
& \Omega_4 = 
\sum_{l=1}^{T} L_{\Psi^{(l)}} 
4 \frac{\zeta (\sqrt{2}\cmax + \cl)  \sqrt{\log 2/p}}{\mathrm{d}_{min}^2} \|W\|_\infty\big(L_{\Phi^{(l)}} B_2^{(l-1)} \big)
\prod_{l' = l+1}^{T} K^{(l')}
\\
& \Omega_5 =  \sum_{l=1}^{T} L_{\Psi^{(l)}}  \left(
\frac{2\|W\|_\infty}{\mathrm{d}_{min}} L_{\Phi^{(l)}} 
 Z^{(l-1)}_1 +  \frac{2\LipW}{\mathrm{d}_{min}} B_1^{(l-1)}\right)\prod_{l' = l+1}^{T} K^{(l')}
\\
& \Omega_6 =  \sum_{l=1}^{T} L_{\Psi^{(l)}}   \left( 
\frac{2\|W\|_\infty}{\mathrm{d}_{min}} L_{\Phi^{(l)}} 
 Z^{(l-1)}_2  +  \frac{2\LipW}{\mathrm{d}_{min}} B_2^{(l-1)}\right)\prod_{l' = l+1}^{T} K^{(l')}
\\
& \Omega_7 = \sum_{l=1}^{T} L_{\Psi^{(l)}}  2 
\frac{\|W\|_\infty}{\mathrm{d}_{min}} L_{\Phi^{(l)}} 
 Z^{(l-1)}_3
\prod_{l' = l+1}^{T} K^{(l')}
\\
& \Omega_8 = \sum_{l=1}^{T} L_{\Psi^{(l)}} \frac{C_\chi}{\sqrt{2}} \frac{\|W\|_\infty}{\mathrm{d}_{min}}\left( L_{\Phi^{(l)}}   B_1^{(l-1)} + \|\Phi^{(l)}(0,0)\|_\infty \right) \prod_{l' = l+1}^{T} K^{(l')}
\\
& \Omega_9 = \sum_{l=1}^{T} L_{\Psi^{(l)}}   \frac{C_\chi}{\sqrt{2}} \frac{\|W\|_\infty}{\mathrm{d}_{min}} L_{\Phi^{(l)}} B_2^{(l-1)} \prod_{l' = l+1}^{T} K^{(l')},
    \end{aligned}
\end{equation}
where $Z_1^{(l-1)},Z_2^{(l-1)},Z_3^{(l-1)}$ are defined in (\ref{eq:z1z2z3}),  $B_1^{(l-1)}$ and $B_2^{(l-1)}$ are defined in (\ref{eq:B'}) and (\ref{eq:B''}), and 
    \[
K^{(l')}  = \sqrt{(L_{\Psi^{(l')}})^2 
  + 
   \frac{8\|W\|_\infty^2}{\cmin^2} (L_{\Phi^{(l')}})^2 (L_{\Psi^{(l')}})^2}.
\]
\end{proof}

Next we study the convergence of MPNNs after global pooling. We give the following lemma.

\begin{lemma}
\label{lemma:C8}
Let $(\chi,d, \P) $ be a metric-measure space and
$W$ be a kernel s.t.  Assumptions \ref{ass:graphon}.\ref{ass:graphon1}-\ref{ass:graphon12}. are satisfied. Suppose that $X_1, \ldots, X_N$ are drawn i.i.d. from $\P$ on $\chi$ such that $(X_1, \ldots, X_N)\in\mathcal{E}_{\rm Lip}^p$, where the event $\mathcal{E}_{\rm Lip}^p$ is defined in Lemma \ref{lemma:NumAnaApprox}. Then, for every MPNN $\Theta$ satisfying Assumption \ref{ass:graphon}.\ref{ass:graphon4} and $f:\chi \to \mathbb{R}^{F}$ with Lipschitz constant $L_f$, 
\begin{equation}
\label{eq:numAnaUniformMPNN}
    \begin{aligned}
& \left\| \frac{1}{N} \sum_{i=1}^N \big(S^X \Theta_W(f) \big)(X_i) - \int_\chi \Theta_W(f)(y) d\P(y) \right\|_\infty 
 \\
 & \leq  N^{-\frac{1}{2(D_\chi+1)}}\Bigg(2 (Z_1^{(T)}+Z_2^{(T)}\|f\|_\infty+Z_3^{(T)}L_f) + \frac{C_\chi}{\sqrt{2}} (B_1^{(T)} + B_2^{(T)} \|f\|_\infty )  \\
 & \cdot \sqrt{\log(C_\chi) + \frac{D_\chi}{2(D_\chi+1)} \log(N) + \log(2/p)}  \Bigg).
\end{aligned}
\end{equation}
 Here, $Z_1^{(T)}, Z_2^{(T)}, Z_3^{(T)}$ and $B_1^{(T)}, B_2^{(T)}$ are defined in (\ref{eq:thmC8-infnorm}) and (\ref{eq:thmC8-lip}).
\end{lemma}
\begin{proof}
By Lemma \ref{lemma:RecRelNorm}, we have
\[
\|\Theta_W^{(T)}(f)\|_\infty \leq B_1^{(T)} + \|f\|_\infty B_2^{(T)}
\]
and,  by Corollary \ref{lemma:RecRelLipschitz}, we have
\[
\begin{aligned}
    L_{\Theta_W^{(T)}(f)} \leq Z_1^{(T)} + Z_2^{(T)}\|f\|_\infty + Z_3^{(T)} L_f
\end{aligned}
\]
for all MPNNs $\Theta$ and metric-space signals $f$ considered. Hence, by Lemma
\ref{lemma:NumAnaApprox}, equation (\ref{eq:numAnaUniformMPNN}) holds.
\end{proof}

\begin{corollary}
\label{cor:C10}
Let $(\chi,d, \P) $ be a metric-measure space and
$W$ be a kernel s.t.  Assumptions \ref{ass:graphon}.\ref{ass:graphon1}-\ref{ass:graphon12}. are satisfied. 
 Consider a graph $(G, \mathbf{f}) \sim (W, f)$ with $N$ nodes and corresponding graph features, where $N$ satisfies (\ref{eq:largeN}). If the event $\mathcal{F}_{\rm Lip}^p$ from Lemma \ref{lemma:C2} occurs, then for every MPNN $\Theta$ satisfying Assumption \ref{ass:graphon}.\ref{ass:graphon4} and every $f:\chi \to \mathbb{R}^{F}$ with Lipschitz constant $L_f$, 
\[
\begin{aligned}
\Big\| \Theta_G^P(\mathbf{f}) - \Theta_W^P(f) \Big\|_\infty^2 & \leq \frac{S_1 + S_2 \|f\|^2_\infty}{N} + \frac{R_1 + R_2\|f\|_\infty^2 + R_3 L_f^2}{N^{\frac{1}{D_\chi + 1}}}  + \frac{T_1 + T_2 \|f\|_\infty^2}{N^{\frac{1}{D_\chi + 1}}} \log(N) \\
& + \frac{S_3 + S_4 \|f\|^2_\infty}{N}\log^2(2/p) + 
\frac{R_4 + R_5\|f\|_\infty^2}{N^{\frac{1}{D_\chi + 1}}} \log(2/p),
\end{aligned}
\]
where the constants are defined in (\ref{eq:thmC10-constants}) below.
\end{corollary}
\begin{proof}
We have
\[
\begin{aligned}
 & \Big\| \Theta_G^P(\mathbf{f}) - \Theta_W^P(f) \Big\|_\infty  \\ 
 & = \big\| \frac{1}{N} \sum_{i=1}^N \Theta_G(\mathbf{f})(X_i) - \int_\chi \Theta_W(f)(y) d\mu(y) \big\|_\infty \\
 & \leq \big\| \frac{1}{N} \sum_{i=1}^N \Theta_G(\mathbf{f})(X_i) - \frac{1}{N} \sum_{i=1}^N\big( S^X \Theta_W(f) \big)(X_i) \big\|_\infty
 \\ & +
 \big\| \frac{1}{N} \sum_{i=1}^N\big( S^X \Theta_W(f) \big)(X_i)  - \int_\chi \Theta_W(f)(y) d\mu(y) \big\|_\infty\\
& \leq   \frac{1}{N} \sum_{i=1}^N \big\| \Theta_G(\mathbf{f})(X_i) - \big( S^X \Theta_W(f) \big)(X_i) \big\|_\infty 
\\ & +
 \big\| \frac{1}{N} \sum_{i=1}^N\big( S^X \Theta_W(f) \big)(X_i)  - \int_\chi \Theta_W(f)(y) d\mu(y) \big\|_\infty \\
& \leq  \frac{1}{N} \sum_{i=1}^N \big\| \Theta_G(\mathbf{f})(X_i) - \big( S^X \Theta_W(f) \big)(X_i) \big\|_\infty \\
&+ N^{-\frac{1}{2(D_\chi+1)}}\Bigg(2 (Z_1^{(T)}+Z_2^{(T)}\|f\|_\infty+Z_3^{(T)}L_f) + \frac{C_\chi}{\sqrt{2}} (B_1^{(T)} + B_2^{(T)} \|f\|_\infty )   \\
& \cdot \sqrt{\log(C_\chi) + \frac{D_\chi}{2(D_\chi+1)} \log(N) + \log(2/p)}   \Bigg)\\
& = \d\big(\Theta_G(\mathbf{f}), \Theta_W(f)\big) \\
& +   N^{-\frac{1}{2(D_\chi+1)}}\Bigg(2 (Z_1^{(T)}+Z_2^{(T)}\|f\|_\infty+Z_3^{(T)}L_f) + \frac{C_\chi}{\sqrt{2}} (B_1^{(T)} + B_2^{(T)} \|f\|_\infty )  
 \\
& \cdot \sqrt{\log(C_\chi) + \frac{D_\chi}{2(D_\chi+1)} \log(N) + \log(2/p)}   \Bigg),
\end{aligned}
\]
where the last inequality holds  by Lemma \ref{lemma:C8}.
Together with Theorem  
\ref{thm:C7}, we get 
\begin{equation}
    \label{eq:corC14-1}
\begin{aligned}
& \big\| \Theta_G^P(\mathbf{f}) - \Theta_W^P(f)  \big\|_\infty \\
& \leq
\frac{\Omega_1  +\Omega_2\log(2/p)+
\Omega_3 \|f\|_\infty + \Omega_4\|f\|_\infty \log(2/p)
}{N^{\frac{1}{2}}}    \\
& +  \frac{\Omega_5 + \Omega_6\|f\|_\infty + \Omega_7 L_f}{N^{\frac{1}{ 2(D_\chi + 1) }}}
\\
 & + 
\frac{\Omega_8 + \Omega_9 \|f\|_\infty}{N^{\frac{1}{2( D_\chi + 1) }}}\sqrt{\log(C_\chi) + \frac{D_\chi}{2(D_\chi+1)} \log(N) + \log(2/p)}
\\
 & + 
 N^{-\frac{1}{2(D_\chi+1)}}\Bigg(2 (Z_1^{(T)}+Z_2^{(T)}\|f\|_\infty+Z_3^{(T)}L_f) + \frac{C_\chi}{\sqrt{2}} (B_1^{(T)} + B_2^{(T)} \|f\|_\infty )   \\
& \cdot  \sqrt{\log(C_\chi) + \frac{D_\chi}{2(D_\chi+1)} \log(N) + \log(2/p)}   \Bigg).
\end{aligned}
\end{equation}
Now we use the inequality
\[
\left(\sum_{i=1}^n a_i\right)^2 \leq n \sum_{i=1}^n a_i^2 
\]
for any $a_i \in \mathbb{R}_+$, $i = 1,\ldots, N$, and square both sides of (\ref{eq:corC14-1}) to get
\[
\begin{aligned}
& \big\| \Theta_G^P(\mathbf{f}) - \Theta_W^P(f)  \big\|_\infty^2 \\
& \leq 14 \frac{\Omega_1^2 + \Omega_3^2\|f\|_\infty^2}{N} + 14 \frac{\Omega_5^2 + \Omega_6^2\|f\|_\infty^2 + \Omega_7^2 L_f^2}{N^{\frac{1}{D_\chi + 1}}}\\
& + 14\frac{\Omega_8^2 + \Omega_9^2 \|f\|_\infty^2}{N^{\frac{1}{ D_\chi + 1 }}}\left(\log(C_\chi) + \frac{D_\chi}{2(D_\chi+1)} \log(N)\right)\\
& +56\frac{
 (Z_1^{(T)})^2 + (Z_2^{(T)})^2\|f\|_\infty^2 + (Z_3^{(T)})^2 L_f
 ^2}{N^{\frac{1}{D_\chi+1}}} \\
 & + 7\frac{\Big(C_\chi^2  (B_1^{(T)})^2 + C_\chi^2(B_2^{(T)})^2 \|f\|_\infty^2 \Big)
 \big(\log(C_\chi) + \frac{D_\chi}{2(D_\chi+1)} \log(N)\big)}{N^{\frac{1}{D_\chi+1}}}\\
 & + 14\frac{\big(\Omega_2^2 + \Omega_4^2 \|f\|_\infty^2 \big)  \log^2(2/p)}{N} + 14\frac{\Omega_8^2 + \Omega_9^2 \|f\|_\infty^2}{N^{\frac{1}{ D_\chi + 1 }}}\log(2/p)\\
 & + 7 \frac{ \big(C_\chi^2  (B_1^{(T)})^2 + C_\chi^2(B_2^{(T)})^2 \|f\|_\infty^2 \big)    \log(2/p)}{N^{\frac{1}{D_\chi+1}}}
\\
& =: \frac{S_1 + S_2 \|f\|^2_\infty}{N} + \frac{R_1 + R_2\|f\|_\infty^2 + R_3 L_f^2}{N^{\frac{1}{D_\chi + 1}}}  + \frac{T_1 + T_2 \|f\|_\infty^2}{N^{\frac{1}{D_\chi + 1}}} \log(N) \\
& + \frac{S_3 + S_4 \|f\|^2_\infty}{N}\log^2(2/p) + 
\frac{R_4 + R_5\|f\|_\infty^2}{N^{\frac{1}{D_\chi + 1}}} \log(2/p),
\end{aligned}
\]
where 
\begin{equation}
    \label{eq:thmC10-constants}
\begin{aligned}
 S_1 & = 14 \Omega_1^2 \\
 S_2 & = 14 \Omega_3^2 \\
  S_3 & = 14 \Omega_2^2 \\
 S_4 & = 14 \Omega_4^2 \\
 R_1 & = 14 \Omega_5^2 + 14 \Omega_8^2 \log(C_\chi) + 56 (Z_1^{(T)})^2 +7 C_\chi^2 (B_1^{(T)})^2\log(C_\chi) \\
 R_2 & = 14 \Omega_6^2 + 14 \Omega_9^2  \log(C_\chi)+ 56 (Z_2^{(T)})^2+ 7 C_\chi^2 (B_2^{(T)})^2\log(C_\chi) \\
 R_3 & = 14 \Omega_7^2 + 56 (Z_3^{(T)})^2\\
 R_4 & = 14 \Omega_8^2 +  7 C_\chi^2 (B_1^{(T)})^2 \\
 R_5 & = 14 \Omega_9^2 +  7 C_\chi^2 (B_2^{(T)})^2\\
 T_1 & = 14 \Omega_8^2 \frac{D_\chi}{2(D_\chi + 1)} + 7 C_\chi^2 (B_1^{(T)})^2\frac{D_\chi}{2(D_\chi + 1)}\\
 T_2 &=  14 \Omega_9^2 \frac{D_\chi}{2(D_\chi + 1)} + 7 C_\chi^2 (B_2^{(T)})^2\frac{D_\chi}{2(D_\chi + 1)} ,
\end{aligned}
\end{equation}
and $\Omega_1,\ldots,\Omega_9$ are defined in (\ref{eq:defConstantsUniform}), and $B_1^{(T)}$ and $B_2^{(T)}$ are defined in (\ref{eq:B'}) and (\ref{eq:B''}).
\end{proof}

We now write a version of Theorem \ref{thm:MainInProb} (about the convergence error of MPNNs) with detailed constants, and prove it.

\begin{theorem}
\label{thm:unifExpValue}
Let $(\chi,d, \P) $ be a metric-measure space and
$W$ be a kernel  s.t.  Assumptions \ref{ass:graphon}.\ref{ass:graphon1}-\ref{ass:graphon12}. and Assumptions \ref{ass:graphon}.\ref{Ass:KernelDiagBounded} are satisfied.   Consider a graph $(G, \mathbf{f}) \sim (W, f)$ with $N$ nodes and corresponding graph features. Then,
for every $f:\chi \to \mathbb{R}^{F}$ with Lipschitz constant $L_f$, 
\[
\begin{aligned}
&  \E_{X_1, \ldots, X_N \sim \mu^N} \left[\sup_{\Theta \in \mathrm{Lip}_{L,B}} \left\| 
\Theta^P_G(\mathbf{f}) - 
\Theta^P_W(f) \right\|_\infty^2 \right] \\
& \leq 6\sqrt{\pi} \Bigg(\frac{S_1 + S_3 + (S_2 + S_4)\|f\|_\infty^2}{N} + \frac{R_1 + R_4 + (R_2 + R_5)\|f\|_\infty^2 +  R_3  L_f^2}{N^{\frac{1}{D_\chi +1}}} \\
& + \frac{\big(T_1 + T_2\|f\|_\infty^2\big) \log(N) }{N^{\frac{1}{D_\chi +1}} } \Bigg)+ \mathcal{O}\left(\exp(-N)N^{3T-\frac{3}{2}}\right),
\end{aligned}
\]
where the constants are defined in (\ref{eq:thmC10-constants}). 
\end{theorem}
\begin{proof}
For any $p >0$, we have with probability at least $1-2p$ for every $\Theta \in \mathrm{Lip}_{L,B}$, by Corollary \ref{cor:C10}, that
\[
\big\| \Theta_G^P(\mathbf{f}) - \Theta_W^P(f) \big\|_{\infty}^2  \leq  H_1 + H_2 \log(2/p) + H_3 \log^2(2/p)
\]
if (\ref{eq:largeN}) holds, where  
\[ 
\begin{aligned}
  &   H_1 =\frac{ S_1 + S_2 \|f\|^2_\infty}{N} + \frac{R_1 + R_2\|f\|_\infty^2 + R_3 L_f^2}{N^{\frac{1}{D_\chi + 1}}}  + \frac{T_1 + T_2 \|f\|_\infty^2}{N^{\frac{1}{D_\chi + 1}}} \log(N) ,
  \\
& H_2 = \frac{R_4 + R_5\|f\|_\infty^2}{N^{\frac{1}{D_\chi + 1}}}
\text{  and   }
 H_3 =\frac{S_3 + S_4 \|f\|^2_\infty}{N}.
\end{aligned}
\]
Further, for every $p \in (0,1/2)$, we consider $k > 0$ such that $ p = 2 \exp(-k^2) $. This means, if $p$ respectively $k$ satisfies (\ref{eq:largeN}), we have with probability at least $1- 4\exp(-k^2) $ for every $\Theta \in \mathrm{Lip}_{L,B}$,
\[
\big\| \Theta_G^P(\mathbf{f}) - \Theta_W^P(f) \big\|_{\infty}^2  \leq  H_1 + H_2k + H_3k^2.
\]

If $k$ does not satisfy (\ref{eq:largeN}), we get
\[
k > N_0 = D_1 + D_2 \sqrt{N},
\]
where $D_1 \in \mathbb{R}$ and $D_2 > 0$ are the matching constants in $(\ref{eq:largeN})$. 
By Lemma \ref{lemma:DeterministicMPNNBound} and Lemma \ref{lemma:RecRelNorm}, we get in this case 
\begin{equation}
    \label{eq:defq(N)}
\begin{aligned}
\big\| \Theta_G^P(\mathbf{f}) - \Theta_W^P(f) \big\|_\infty^2 & = 
 \left\| \frac{1}{N}\sum_{i=1}^{N}
  \Theta_G(\mathbf{f})_i - \int_\chi\Theta_W(f)(y) d\P(y)\right\|_\infty^2 \\
  & \leq  \frac{4}{N}\sum_{i=1}^{N}
  \|\Theta_G(\mathbf{f})_i\|_\infty^2 + 2 \Big\| \int_\chi\Phi_W(f)(y) d\P(y)\Big\|_\infty^2 \\
& \leq  \frac{4}{N}
  \| \Theta_G(\mathbf{f})\|_{2;\infty}^2 + 2\|\Theta_W(f) \|_\infty^2 \\
  & \leq \frac{4}{N} N^{2T}( A' + A'' \|f\|_\infty^2) +
2(B_1^{(T)} + \|f\|_\infty B_2^{(T)})^2 =: q(N),
\end{aligned}
\end{equation}
where the first inequality holds by applying the triangle inequality and Cauchy-Schwarz.

We then calculate the expected value by partitioning the integral over the event space into the following sum. 
\begin{equation}
\label{eq:C18-1}
\begin{aligned}
     &\E_{X_1,\ldots,X_N \sim \mu^N} \left[\sup_{\Theta \in \mathrm{Lip}_{L,B}}\big\| \Theta_G^P(\mathbf{f}) - \Theta_W^P(f) \big\|_\infty^2 \right]  \\
  \leq  &   \sum_{k = 0}^{N_0} \mathbb{P}\big(H_1 + H_2k + H_3k^2 \leq\sup_{\Theta \in \mathrm{Lip}_{L,B}} \big\| \Theta_G^P(\mathbf{f}) - \Theta_W^P(f) \big\|_\infty^2 
< H_1  + H_2(k+1) +H_3(k+1)^2\big)
\\
& \cdot \big(  H_1  + H_2(k+1) +H_3(k+1)^2 \big) \\
+ &\sum_{k = N_0}^{\infty} \mathbb{P}\big(H_1 + H_2k + H_3k^2 \leq \sup_{\Theta \in \mathrm{Lip}_{L,B}}\big\| \Theta_G^P(\mathbf{f}) - \Theta_W^P(f) \big\|_\infty^2 
< H_1  + H_2(k+1) +H_3(k+1)^2\big)  \\
& \cdot q(N) 
\\
\end{aligned}
\end{equation}
To bound  the second sum, note that it is a finite sum, since $\big\| \Theta_G^P(\mathbf{f}) - \Theta_W^P(f) \big\|_\infty^2$ is bounded by $q(N)$, which is defined in (\ref{eq:defq(N)}). 
The summands are zero if $H_1 + H_2k + H_3k^2 > q(N)$, which holds for $k > \sqrt{\frac{q(N)}{H_3}}$. Hence, we calculate with the right-hand-side of (\ref{eq:C18-1}) by
\begin{equation}
\begin{aligned}
& \leq 2 \sum_{k=0}^{N_0} 2 \exp(-k^2) \cdot \big(  H_1  + H_2(k+1) +H_3(k+1)^2 \big) 
\\
& + \sum_{k=N_0}^{\left\lceil \sqrt{\frac{q(N)}{H_3}} \right\rceil}4 \exp(-N_0^2) \cdot q(N)
\\
& \leq 2 \int_0^{\infty} 2 \exp(-k^2) \cdot \big(  H_1  + H_2(k+1) +H_3(k+1)^2 \big)
\\
& + 4 \exp(-N_0^2) q(N)\left\lceil \sqrt{\frac{q(N)}{H_3}}\right\rceil
,
\end{aligned}
\end{equation}
where 
$q(N) = O(N^{2T-1})$ is a polynomial in $N$ as defined above.
The first term on the right-hand-side is bounded by using 
\[
\int_0^\infty 2(t+1)^2e^{-t^2} dt, \;\int_0^\infty 2(t+1) e^{-t^2} dt , \; \int_0^\infty  2e^{-t^2} dt \leq 3\sqrt{\pi}.
\]
For the second term we remember that $N_0 = D_1 + D_2\sqrt{N}$. Hence,
\begin{align*}
&\E_{X_1,\ldots,X_N \sim \mu^N} \left[\sup_{\Theta \in \mathrm{Lip}_{L,B}}\big\| \Theta_G^P(\mathbf{f}) - \Theta_W^P(f) \big\|_\infty^2 \right]\\
&\leq 6 \sqrt{\pi} (H_1 + H_2 + H_3) + \mathcal{O}(\exp(-N)N^{3T-\frac{3}{2}}).
\end{align*} 
\end{proof}

\section{Generalization Analysis}
\label{AppendixGenError}
In this section, we provide details on our generalization analysis of MPNNs. In Subsection \ref{subsec:ProbSpaceDataset}, we detail the data distribution from the graph classification task, which was introduced in Subsection \ref{subsec:GenDataDistr}.  In Subsection \ref{subsec:GenErrorProof}, we provide a detailed version and a proof for Theorem \ref{thm:mainGenError} (about the generalization bound of MPNNs). This is followed by a derivation of the asymptotics  of our generalization bound in Subsection \ref{subsec:disGenBound} and a comparison of the asymptotics of our generalization bound with other related generalization bounds in Subsection \ref{subsec:comp}. 

\subsection{The Probability Space of the Dataset}
\label{subsec:ProbSpaceDataset}

Recall that the measure on the space $\chi^j$ is denoted by $\mu^j$.
Given a class $j$ and $N\in \mathbb{N}$, the space of graphs with $N$ nodes from class $j$ is defined to be $(\chi^j)^N$.
The measure on $(\chi^j)^N$ is defined to be $(\mu^j)^N$, namely, the direct product of the measure $\mu^j$ with itself $N$ times. The space $\mathcal{G}_j$ of graphs of any size, which are sampled from class $j$, is defined to be 
\[\mathcal{G}_j := \bigcup_{n\in\mathbb{N}}(\chi^j)^N.\]
The measure on $\mathcal{G}_j$ is denoted by $\mu_{\mathcal{G}_j}$, and defined as follows.
\begin{definition}
\label{def:measure_1class}
 A set of graphs $S\subset \mathcal{G}_j$ is called measurable, if for each $N\in\mathbb{N}$, the restriction 
 \[S_N:=\{G\in S\ |\  G{\rm \ has\ } N {\rm \ nodes}\} \subset (\chi^j)^N\]
 is measurable with respect to $(\mu^j)^N$. The measure of a measurable set $S\subset \mathcal{G}_j$ is defined to be
\[\mu_{\mathcal{G}_j}(S) := \sum_{N=1}^{\infty}\nu(N)(\mu^j)^N(S_N),\]
where $\nu(N)$ is the probability of choosing a graph with $N$ nodes (see Subsection \ref{subsec:GenDataDistr}).
\end{definition}
The space of graphs of either of the classes $j=1,\ldots, \Gamma$ is defined to be
\[\mathcal{G}:=\bigcup_{j=1}^{\Gamma}\mathcal{G}_j.\]
The measure on $\mathcal{G}$ is denoted by $\mu_{\mathcal{G}}$, and defined as follows.
\begin{definition}
\label{def:data_dist}
 A set of graphs $S\subset \mathcal{G}$ is called measurable, if for each $j=1,\ldots,\Gamma$, the restriction 
 \[S_j:=\{G\in S\ |\  G{\rm \ is\ sampled\ from\ class\ }j\} \subset \mathcal{G}_j\]
 is measurable with respect to $\mu_{\mathcal{G}_j}$. The measure of a measurable $S\subset \mathcal{G}$ is defined to be
\[\mu_{\mathcal{G}}(S) = \sum_{j=1}^{\Gamma}\gamma_j\mu_{\mathcal{G}_j}(S_j),\]
where $\gamma_j$ is the probability of choosing class $j$ (see Subsection \ref{subsec:GenDataDistr}).   
\end{definition}
With these notations, the space of graph datasets of size $m$ is defined to be $\mathcal{G}^m$ with the direct product measure $\mu_{\mathcal{G}}^m$. We denote a random graph sampled from the space of graphs by $(G,\mathbf{f},y)\sim \mu_{\mathcal{G}}$. 
 Here, $y$ denotes the class of the graph, namely, the value $y$ such that  $(G,\mathbf{f})$ is sampled from class $y$.

The next lemma is direct, and given without proof.

\begin{lemma}
The spaces $\{\mathcal{G},\mu_{\mathcal{G}}\}$ and $\{\mathcal{G}_j,\mu_{\mathcal{G}_j}\}$, $j=1,\ldots, \Gamma$, are measure spaces, and $\mu_{\mathcal{G}}$ and $\mu_{\mathcal{G}_j}$, $j=1,\ldots, \Gamma$, are probability measures.
\end{lemma}

Let us next derive a re-parameterization  of the space of datasets $\mathcal{G}^m$. Given $\mathcal{T}\sim \mu_{\mathcal{G}}^m$, 
for every $j=1, \ldots, \Gamma$, let $m_j$ denote the number of graphs in $\mathcal{T}$ that fall into the class $j$. Note that $\mathbf{m}=(m_1,\ldots,m_{\Gamma})$ has a multinomial distribution with parameters $m$ and $\boldsymbol{\gamma}=(\gamma_1, \ldots, \gamma_\Gamma)$, which we denote by $\mathrm{MN}_{m,\boldsymbol{\gamma}}$. Conditioning the choice of the graphs on the choice of $\mathbf{m}$, we can formulate the data sampling procedure as first sampling $\mathbf{m}$ from $\mathrm{MN}_{m,\boldsymbol{\gamma}}$, and then sampling $\{G_i^j, \mathbf{f}_i^j\}_{i=1}^{m_j} \sim (\mu_{\mathcal{G}_j})^{m_j}$, 
 $j=1\ldots,\Gamma$ independently of each other. Now, the measure $\mu_{\mathcal{G}}^m$ of the space of datasets 
 can be parameterized as follows. 
 
 First, we define the following measure space. Let $\mathbf{m}=(m_1,\ldots,m_{\Gamma})$ satisfy $\sum_{j=1}^{\Gamma}m_j=m$. We define the space \[\mathcal{G}^{\mathbf{m}}:= \prod_{j=1}^{\Gamma} \mathcal{G}_j^{m_j},\]
 with the measure
\begin{equation}
  \mu_{\mathcal{G}^{\mathbf{m}}}:= \prod_{j=1}^{\Gamma} \mu_{\mathcal{G}_j}^{m_j}.  \label{eq:mras_data_2}
\end{equation}
 The space $\mathcal{G}^{\mathbf{m}}$ is interpreted as the space of datasets with exactly $m_j$ samples in each class $j$.

We can now show the following parametrization of the measure space $\mathcal{G}^m$ of datasets of size $m$. The lemma is direct, and given without proof.
 
 \begin{lemma}
 \label{lemm:cond_data}
  A  set of datasets $S\subset {\mathcal{G}}^m$  is measurable, if and only if for every $\mathbf{m}=(m_1,\ldots,m_{\Gamma})$ with $\sum_{j=1}^{\Gamma}m_j=m$, the restriction 
  \[S_{\mathbf{m}} = \{ \mathcal{T}\in S\ |\   \forall 1\leq j\leq \Gamma,\  \ \mathcal{T} {\rm \ contains\ } m_j {\rm \ graphs\ from\ class\ } j \} \subset \mathcal{G}^{\mathbf{m}}\]
 is measurable with respect to $\mu_{\mathcal{G}^{\mathbf{m}}}$.
 
 With these notations, $\mu_{\mathrm{G}}^m$ is decomposed as follows: $\mathcal{G}^m = \bigcup_{\mathbf{m}} \mathcal{G}^{\mathbf{m}}$, and for every measurable set of datasets $S\subset \mathcal{G}^m$,
\[\mu_{\mathrm{G}}^m(S)
=
\sum_{\mathbf{m}:\ m_1+\ldots+m_{\Gamma}=m}\mu_{\mathrm{MN}_{m,\boldsymbol{\gamma}}}(\mathbf{m})\sum_{j=1}^{\Gamma}\sum_{i=1}^{m_j}\mu_{{\rm G}_j}(S_{\mathbf{m}}).\] 
 \end{lemma}
\subsection{Proof of Theorem \ref{thm:mainGenError}}
\label{subsec:GenErrorProof}


The following corollary computes the expected robustness of a random graph, of arbitrary size, sampled from $\mu_{\mathcal{G}_j}$, and is a direct result of Definition \ref{def:measure_1class} and Theorem \ref{thm:unifExpValue}.
\begin{corollary}
\label{cor:StabVarGraphSize}
Let  $\{(W^j, f^j)\}$ be a RGM on the corresponding metric-measure space $(\chi^j, d^j, \mu^j)$  that satisfies Assumptions \ref{ass:graphon}.\ref{ass:graphon1}.-\ref{ass:graphon12}. and \ref{ass:graphon}.\ref{Ass:KernelDiagBounded}. Let $\mu_{\mathcal{G}_j}$ be the distribution from Definition \ref{def:measure_1class}.  Then, 
\[
\begin{aligned}
&  \E_{(G^j, \mathbf{f}^j) \sim \mu_{\mathcal{G}_j}} \big[ \sup_{\Theta \in \mathrm{Lip}_{L,B}} \big\| 
\Theta^P_{G^j}(\mathbf{f}^j) - 
\Theta^P_{W^j}(f^j) \big\|_\infty^2 \big] \\
& \leq 6\sqrt{\pi} \Bigg( \big(S_1^{(j)} + S_3^{(j)} + (S_2^{(j)} + S_4^{(j)})\|f^j\|_\infty^2 \big) \E_{N \sim \nu}\left[ N^{-1} \right] \\
& + \big(R_1^{(j)} + R_4^{(j)} + (R_2^{(j)} + R_5^{(j)})\|f^j\|_\infty^2 + R_3^{(j)} L_{f^j}^2\big) \E_{N \sim \nu}\left[ N^{-\frac{1}{D_{\chi^j} + 1}} \right] \\
& +  \big(T_1^{(j)} + T_2^{(j)}\|f^j\|_\infty^2\big)  \E_{N \sim \nu}\left[ \log(N) N^{-\frac{1}{D_{\chi^j} +1}}   \right] \Bigg)+ \mathcal{O}\left( \E_{N \sim \nu} \left[\exp(-N)N^{3T-\frac{3}{2}} \right]\right),
\end{aligned}
\]
where $S_l^{(j)},R_l^{(j)},T_l^{(j)}$ are the according constants from Theorem \ref{thm:unifExpValue} for each class $j$ and are defined in (\ref{eq:thmC10-constants}).
\end{corollary}

When sampling a dataset $\mathcal{T}\sim p^m$, the numbers of samples $m_j$ that fall in class $\chi^j$, for $j=1,\ldots,\Gamma$, are distributed multinomially. 
 We hence recall a concentration of measure result for multinomial variables.


\begin{lemma}[Proposition A.6 in \cite{Vaart}, Bretagnolle-Huber-Carol inequality]
\label{lemma:BHCineq}
If the random vector $(m_1, \ldots m_\Gamma)$ is multinomially distributed with parameters $m$ and $\gamma_1, \ldots, \gamma_\Gamma$, then 
\[
\mathbb{P}\left( \sum_{i=1}^\Gamma| m_i - m\gamma_i | \geq 2 \sqrt{m} \lambda \right)
\leq 2^\Gamma \exp(-2\lambda^2)
\]
for any $\lambda > 0$. 
\end{lemma}

We now write a version of Theorem \ref{thm:mainGenError} (about the generalization error of MPNNs) with detailed constants, and prove it.

\begin{theorem}
\label{thm:reformulatedTheorem2}
Let  $\{(W^j, f^j)\}_{j=1}^\Gamma$ be a collection of RGMs on corresponding metric-measure spaces $\{(\chi^j, d^j, \mu^j)\}_{j=1}^\Gamma$ such that each one satisfies Assumptions \ref{ass:graphon}.\ref{ass:graphon1}.-\ref{ass:graphon12}. and \ref{ass:graphon}.\ref{Ass:KernelDiagBounded}. Let $\mu_{\mathcal{G}}$ denote the data distribution from Definition \ref{def:data_dist}. 
Let $\mathcal{T}=\big((G_1, \mathbf{f}_1,y_1), \ldots, (G_m, \mathbf{f}_m,y_m)\big)\sim \mu_\mathcal{G}^m$ be a dataset of graphs.  Then, 
\[
\begin{aligned}
  &   \E_{\mathcal{T} \sim \mu_{\mathcal{G}}^m} \left[\sup_{\Theta \in \mathrm{Lip}_{L,B}}  \left( \frac{1}{m} \sum_{i=1}^m  \mathcal{L}(\Theta_{G_i}^P(\mathbf{f}_i), y_i) - \E_{(G, \mathbf{f},y)\sim \mu_{\mathcal{G}}}\left[ \mathcal{L}(\Theta_{G}^P(\mathbf{f}), y) \right] \right)^2 \right] \\
  & \leq 2^\Gamma \frac{8 \|\mathcal{L}\|_\infty^2}{m} \pi + \frac{6\sqrt{\pi}}{ m} 2^\Gamma  \Gamma  \sum_{j=1}^\Gamma \gamma_j L_\mathcal{L}^2 \Bigg(  \sqrt{\pi} \Big( \big(S_1^{(j)} + S_3^{(j)} + (S_2^{(j)} + S_4^{(j)})\|f^j\|_\infty^2 \big) \E_{N \sim \nu}\left[ N^{-1} \right] \\
& + \big(R_1^{(j)} + R_4^{(j)} + (R_2^{(j)} + R_5^{(j)})\|f^j\|_\infty^2 + R_3^{(j)} L_{f^j}^2\big) \E_{N \sim \nu}\left[ N^{-\frac{1}{D_{\chi^j} + 1}} \right] \\
& +  \big(T_1^{(j)} + T_2^{(j)}\|f^j\|_\infty^2\big)  \E_{N \sim \nu}\left[  \log(N) N^{-\frac{1}{D_{\chi^j} +1}} \right] \Big)+ \mathcal{O}\left( \E_{N \sim \nu} \left[\exp(-N)N^{3T-\frac{3}{2}} \right]\right)\Bigg),
\end{aligned}
\]
where $S_l^{(j)},R_l^{(j)},T_l^{(j)}$ are the according constants from Theorem \ref{thm:unifExpValue} for each class $j$ and are defined in (\ref{eq:thmC10-constants}). 
\end{theorem}
\begin{proof}


Given $\mathbf{m}=(m_1,\ldots,m_{\Gamma})$ with $\sum_{j=1}^{\Gamma}m_j=m$, recall that $\mathcal{G}^{\mathbf{m}}$ is the space of datasets with fixed number of samples $m_j$ from each class $j=1,\ldots,\Gamma$.  
The probability measure on $\mathcal{G}^{\mathbf{m}}$ is given by $\mu_{\mathcal{G}^{\mathbf{m}}}$ (see (\ref{eq:mras_data_2})).
Similarly to the notation of Lemma \ref{lemm:cond_data}, we denote the conditional choice of the dataset on the choice of $\mathbf{m}$ by 
\[\mathcal{T}_{\mathbf{m}}:=\big\{\{G_i^j, \mathbf{f}_i^j\}_{i=1}^{m_j}\big\}_{j=1}^{\Gamma} \sim \mu_{\mathcal{G}^{\mathbf{m}}}.\]
Given $k\in\mathbb{Z}$, denote by $\mathcal{M}_k$ the set of all $\mathbf{m}=(m_1,\ldots, m_{\Gamma})\in \mathbb{N}_0^{\Gamma}$ with $\sum_{j=1}^{\Gamma}m_j=m$, such that 
$2\sqrt{m} k \leq \sum_{j=1}^\Gamma |m_j - m\gamma_j| <2\sqrt{m} (k+1)$.
Using these notations, we decompose the expected generalization error as follows.  
\begin{equation}
    \label{eq:thmC13-1}
    \begin{aligned}
    & \E_{\mathcal{T} \sim \mu_{\mathcal{G}}^m} \left[ \sup_{\Theta \in \mathrm{Lip}_{L,B}} \left( \frac{1}{m} \sum_{i=1}^m  \mathcal{L}(\Theta_{G_i}^P(\mathbf{f}_i), y_i) - \E_{(G, \mathbf{f},y)\sim \mu_{\mathcal{G}}}\left[ \mathcal{L}(\Theta_{G}^P(\mathbf{f}), y) \right] \right)^2 \right] \\
  &  = \E_{\mathcal{T} \sim \mu_{\mathcal{G}}^m} \left[\sup_{\Theta \in \mathrm{Lip}_{L,B}}  \left( \frac{1}{m} \sum_{j=1}^\Gamma  \sum_{i=1}^{m_j} \mathcal{L}(\Theta_{G_i^j}^P(\mathbf{f}_i^j), y_j) - \E_{(G, \mathbf{f},y)\sim \mu_{\mathcal{G}}}\left[ \mathcal{L}(\Theta_{G}^P(\mathbf{f}), y) \right] \right)^2 \right]\\
 & = \E_{\mathcal{T}  \sim  \mu_{\mathcal{G}}^m} \left[\sup_{\Theta \in \mathrm{Lip}_{L,B}}\left(\sum_{j=1}^\Gamma \left( \frac{1}{m}\sum_{i=1}^{m_j} \mathcal{L}(\Theta_{G_i^j}^P(\mathbf{f}_i^j), y_j) - \gamma_j\E_{(G^j, \mathbf{f}^j )\sim \mu_{\mathcal{G}_j}}\left[ \mathcal{L}(\Theta_{G^j}^P(\mathbf{f}^j), y_j) \right] \right) \right)^2 \right] \\
& \leq \sum_{k} \mathbb{P}\big(
\mathbf{m}\in \mathcal{M}_k
\big) \times 
\sup_{\mathbf{m}\in \mathcal{M}_k} \E_{\mathcal{T}_{\mathbf{m}} \sim \mu_{\mathcal{G}^{\mathbf{m}}}} \left[\sup_{\Theta \in \mathrm{Lip}_{L,B}}\left(\sum_{j=1}^\Gamma \left( \frac{1}{m}\sum_{i=1}^{m_j} \mathcal{L}(\Theta_{G_i^j}^P(\mathbf{f}_i^j), y_j) \right.\right.\right.\\  & \quad \ \ \quad \ \ \quad \ \ \quad \ \ \quad \ \ \quad \ \ \quad \ \ \quad \ \ \quad \ \ \quad \ \ \quad \ \  \left.\left.\left. - \frac{1}{m} \sum_{i=1}^{m \gamma_j}\E_{(G^j, \mathbf{f}^j )\sim \mu_{\mathcal{G}_j}}\left[ \mathcal{L}(\Theta_{G^j}^P(\mathbf{f}^j), y_j) \right] \right) \right)^2 \right] 
 \end{aligned}
\end{equation}


We bound the last term of (\ref{eq:thmC13-1}) as follows.
For $j=1, \ldots, \Gamma$, if $m_j \leq m\gamma_j$, we add "ghost samples", i.e., we add additional i.i.d. sampled graphs $(G_{m_j}^j, \mathbf{f}_{m_j}^j), \ldots, (G_{m \gamma_j}^j, \mathbf{f}_{m \gamma_j}^j) \sim (W^j, f^j) $. By convention, for any two $l,q\in\mathbb{N}_0$ with $l<q$, we define \[\sum_{j=q}^l c_j = - \sum_{j=l}^q c_j\]
for any sequence $c_j$ of reals, and define $\sum_{j=q}^q c_j=0$.
With these notations, we have 
\begin{equation}
\label{eq:lastthm-2}
\begin{aligned}
& \E_{\mathcal{T}_{\mathbf{m}} \sim \mu_{\mathcal{G}^{\mathbf{m}}}} \left[\sup_{\Theta \in \mathrm{Lip}_{L,B}}\left(\sum_{j=1}^\Gamma \left( \frac{1}{m}\sum_{i=1}^{m_j} \mathcal{L}(\Theta_{G_i^j}^P(\mathbf{f}_i^j),y_j) \right.\right.\right.\\  & \quad \ \ \quad \ \     \quad \ \  \left.\left.\left. - \frac{1}{m} \sum_{i=1}^{m \gamma_j}\E_{(G^j, \mathbf{f}^j )\sim \mu_{\mathcal{G}_j}}\left[ \mathcal{L}(\Theta_{G^j}^P(\mathbf{f}^j), y_j) \right] \right) \right)^2 \right] 
\\
& =  \E_{\mathcal{T}_{\mathbf{m}} \sim  \mu_{\mathcal{G}^{\mathbf{m}}}} \Bigg[\sup_{\Theta \in \mathrm{Lip}_{L,B}}\Bigg(\sum_{j=1}^\Gamma \Bigg( \frac{1}{m}\sum_{i=1}^{m\gamma_j} \mathcal{L}(\Theta_{G_i^j}^P(\mathbf{f}_i^j),y_j) +  \frac{1}{m}\sum_{i=m\gamma_j}^{m_j} \mathcal{L}(\Theta_{G_i^j}^P(\mathbf{f}_i^j), y_j)
 \\ & - \frac{1}{m} \sum_{i=1}^{m \gamma_j}\E_{(G^j, \mathbf{f}^j )\sim \mu_{\mathcal{G}_j}}\big[ \mathcal{L}(\Theta_{G^j}^P(\mathbf{f}^j),y_j) \big] \Bigg) \Bigg)^2 \Bigg] 
\\
& \leq \E_{\mathcal{T}_{\mathbf{m}} \sim  \mu_{\mathcal{G}^{\mathbf{m}}}} \left[\sup_{\Theta \in \mathrm{Lip}_{L,B}}2 \left(\sum_{j=1}^\Gamma \left( \frac{1}{m}\sum_{i=1}^{m \gamma_j} \mathcal{L}(\Theta_{G_i^j}^P(\mathbf{f}_i^j),y_j) \right.\right.\right.\\  & \quad \ \ \quad \ \     \quad \ \ \quad   \left.\left.\left. - \frac{1}{m} \sum_{i=1}^{m \gamma_j}\E_{(G^j, \mathbf{f}^j )\sim \mu_{\mathcal{G}_j}}\left[ \mathcal{L}(\Theta_{G^j}^P(\mathbf{f}^j), y_j) \right] \right) \right)^2 \right]
\\ & + \E_{\mathcal{T}_{\mathbf{m}} \sim  \mu_{\mathcal{G}^{\mathbf{m}}}} \left[2
\left(\sum_{j=1}^\Gamma \left(
\frac{1}{m} |m\gamma_j - m_j| \|\mathcal{L}\|_\infty
\right)\right)^2
\right].
\end{aligned}
\end{equation}
Let us first bound the last term of the above bound. Since any $\mathbf{m\in \mathcal{M}_k}$ satisfies $\sum_{j=1}^\Gamma |m_j - m\gamma_j| <2\sqrt{m} (k+1)$,  we have
\[
\begin{aligned}
  \E_{\mathcal{T}_{\mathbf{m}} \sim \mu_{\mathcal{G}^{\mathbf{m}}}} \left[2
\left(\sum_{j=1}^\Gamma \left(
\frac{1}{m} |m\gamma_j - m_j| \|\mathcal{L}\|_\infty
\right)\right)^2
\right] 
& \leq    \frac{2}{m^2} \|\mathcal{L}\|_\infty^2 \left(\sum_{j=1}^\Gamma |m\gamma_j -m_j | \right)^2
\\
& \leq    \frac{2}{m^2} \|\mathcal{L}\|_\infty^2 4 m (k+1)^2 
  = \frac{8 \|\mathcal{L}\|_\infty^2}{m} (k+1)^2.
\end{aligned}
\]
Hence, by Lemma \ref{lemma:BHCineq}, 
\[
\begin{aligned}
&  \sum_{k} \mathbb{P}\big(
\mathbf{m}\in \mathcal{M}_k
\big) \times
\sup_{\mathbf{m}\in \mathcal{M}_k} \E_{\mathcal{T}_{\mathbf{m}} \sim \mu_{\mathcal{G}^{\mathbf{m}}}} \left[2
\left(\sum_{j=1}^\Gamma \left(
\frac{1}{m} |m\gamma_j - m_j| \|\mathcal{L}\|_\infty
\right)\right)^2
\right] \\
& \leq  \sum_{k} \mathbb{P}\big(
\mathbf{m}\in \mathcal{M}_k
\big) \times \frac{8 \|\mathcal{L}\|_\infty^2}{m} (k+1)^2 \\
&  \leq \sum_{k} 2^\Gamma \exp(-2 k^2)  \frac{8 \|\mathcal{L}\|_\infty^2}{m} (k+1)^2 \\
&  \leq \int_0^\infty2^\Gamma \exp(-2 k^2)  \frac{8 \|\mathcal{L}\|_\infty^2}{m} (k+1)^2 dk\\
& =  2^\Gamma  \frac{8 \|\mathcal{L}\|_\infty^2}{m} \int_0^\infty \exp(-2 k^2)  (k+1)^2  dk \\ &\leq   2^\Gamma \frac{8 \|\mathcal{L}\|_\infty^2}{m} \pi.
\end{aligned}
\]
To bound the first term of the right-hand-side of (\ref{eq:lastthm-2}), we have 
\[
\begin{aligned}
 & \E_{\mathcal{T}_{\mathbf{m}} \sim \mu_{\mathcal{G}^{\mathbf{m}}}} \left[\sup_{\Theta \in \mathrm{Lip}_{L,B}}\left(\sum_{j=1}^\Gamma \left( \frac{1}{m}\sum_{i=1}^{m \gamma_j} \mathcal{L}(\Theta_{G_i^j}^P(\mathbf{f}_i^j),y_j)  \right.\right.\right.\\  & \quad \ \ \quad \ \     \quad \ \  \left.\left.\left. - \frac{1}{m} \sum_{i=1}^{m \gamma_j}\E_{(G^j, \mathbf{f}^j )\sim \mu_{\mathcal{G}_j}}\left[\sup_{\Theta \in \mathrm{Lip}_{L,B}} \mathcal{L}(\Theta_{G^j}^P(\mathbf{f}^j), y_j) \right] \right) \right)^2 \right]\\
\leq &  \Gamma \sum_{j=1}^\Gamma\E_{\mathcal{T}_{\mathbf{m}} \sim \mu_{\mathcal{G}^{\mathbf{m}}}}\left[\sup_{\Theta \in \mathrm{Lip}_{L,B}} \left( \frac{1}{m}\sum_{i=1}^{m\gamma_j} \mathcal{L}(\Theta_{G_i^j}^P(\mathbf{f}^j_i),y_j)  \right.\right. \\  & \quad \ \ \quad \ \ \quad \ \   \quad \ \ \quad  \left.\left.  - \frac{1}{m}\sum_{i=1}^{m \gamma_j}\E_{(G^j, \mathbf{f}^j ) \sim \mu_{\mathcal{G}_j}}\left[\sup_{\Theta \in \mathrm{Lip}_{L,B}} \mathcal{L}(\Theta_{G^j}^P(\mathbf{f}^j), y_j) \right] \right)^2 \right] \\
= & \Gamma\sum_{j=1}^\Gamma \Var_{(G^j, \mathbf{f}^j) \sim \mu_{\mathcal{G}_j}}\left[\sup_{\Theta \in \mathrm{Lip}_{L,B}} \frac{1}{m} \sum_{i=1}^{\gamma_j\cdot m} \mathcal{L}(\Theta_{G^j}^P(\mathbf{f}^j),y_j) \right]  \\
= & \Gamma\sum_{j=1}^\Gamma \frac{\gamma_j}{m} \Var_{(G^j, \mathbf{f}^j) \sim \mu_{\mathcal{G}_j}}\left[\sup_{\Theta \in \mathrm{Lip}_{L,B}} \mathcal{L}(\Theta_{G^j}^P(\mathbf{f}^j),y_j) \right]
\\
\leq & \Gamma\sum_{j=1}^\Gamma \frac{\gamma_j}{m} \E_{(G^j, \mathbf{f}^j) \sim \mu_{\mathcal{G}_j}}\left[\sup_{\Theta \in \mathrm{Lip}_{L,B}} \left| \mathcal{L}(\Theta_{G^j}^P(\mathbf{f}^j),y_j) - \mathcal{L}(\Theta_{W^j}^P (f^j),y_j)\right|^2 \right]\\
\leq & \Gamma\sum_{j=1}^\Gamma \frac{\gamma_j}{m} \E_{(G^j, \mathbf{f}^j) \sim \mu_{\mathcal{G}_j}}\left[\sup_{\Theta \in \mathrm{Lip}_{L,B}} L_\mathcal{L}^2\| \Theta_{G^j}^P(\mathbf{f}^j) - \Theta_{W^j}^P (f^j)\|_\infty^2 \right].
\end{aligned}
\]
We now apply Corollary \ref{cor:StabVarGraphSize} to get 
\[
\begin{aligned}
\leq \Gamma & \sum_{j=1}^\Gamma \frac{\gamma_j}{m} L_\mathcal{L}^2 \Bigg( 6 \sqrt{\pi} \Bigg( \big(S_1 + S_3 + (S_2 + S_4)\|f^j\|_\infty^2 \big) \E_{N \sim \nu}\left[ N^{-1} \right] \\
& + \big(R_1 + R_4 + (R_2 + R_5)\|f^j\|_\infty^2 + R_3 L_{f^j}^2\big) \E_{N \sim \nu}\left[ N^{-\frac{1}{D_{\chi^j} + 1}} \right] \\
& +  \big(T_1 + T_2\|f^j\|_\infty^2\big)  \E_{N \sim \nu}\left[ \frac{\log(N)}{N^{\frac{1}{D_\chi^j +1}} } \right] \Bigg)+ \mathcal{O}\left( \E_{N \sim \nu} \left[\exp(-N)N^{3T-\frac{3}{2}} \right]\right) \Bigg).
\end{aligned}
\]

Hence, by Lemma \ref{lemma:BHCineq}, 
\[
\begin{aligned}
& \sum_{k} \mathbb{P}\big(
\mathbf{m}\in \mathcal{M}_k
\big) \times
\sup_{\mathbf{m}\in \mathcal{M}_k} \E_{\mathcal{T}_{\mathbf{m}} \sim \mu_{\mathcal{G}^{\mathbf{m}}}} \left[\sup_{\Theta \in \mathrm{Lip}_{L,B}}\left(\sum_{j=1}^\Gamma \left( \frac{1}{m}\sum_{i=1}^{m \gamma_j} \mathcal{L}(\Theta_{G_i^j}^P(\mathbf{f}_i^j),y_j)   \right.\right.\right.\\  & \quad \ \ \quad \ \ \quad \ \ \quad \ \ \quad \ \ \quad \ \ \quad \ \ \quad \ \ \quad \ \ \quad \ \ \quad   \left.\left.\left. - \frac{1}{m} \sum_{i=1}^{m \gamma_j}\E_{(G^j, \mathbf{f}^j )\sim \mu_{\mathcal{G}_j}}\left[ \mathcal{L}(\Theta_{G^j}^P(\mathbf{f}^j),y_j) \right] \right) \right)^2 \right] \\
& \leq \frac{\sqrt{\pi}}{2} 2^\Gamma \sum_{j=1}^\Gamma \frac{\gamma_j}{m} \E_{\mathcal{T}_{\mathbf{m}} \sim \mu_{\mathcal{G}^{\mathbf{m}}}} \left[\sup_{\Theta \in \mathrm{Lip}_{L,B}}\left(\sum_{j=1}^\Gamma \left( \frac{1}{m}\sum_{i=1}^{m \gamma_j} \mathcal{L}(\Theta_{G_i^j}^P(\mathbf{f}_i^j),y_j)   \right.\right.\right.\\  & \quad \ \ \quad \ \ \quad \ \ \quad \ \ \quad \ \ \quad \ \ \quad \ \ \quad \ \   \left.\left.\left. - \frac{1}{m} \sum_{i=1}^{m \gamma_j}\E_{(G^j, \mathbf{f}^j )\sim \mu_{\mathcal{G}_j}}\left[ \mathcal{L}(\Theta_{G^j}^P(\mathbf{f}^j),y_j) \right] \right) \right)^2 \right]\\
& \leq  \frac{\sqrt{\pi}}{2} 2^\Gamma  \Gamma  \sum_{j=1}^\Gamma \frac{\gamma_j}{m} L_\mathcal{L}^2 \Bigg( 6 \sqrt{\pi} \Bigg( \big(S_1^{(j)} + S_3^{(j)} + (S_2^{(j)} + S_4^{(j)})\|f^j\|_\infty^2 \big) \E_{N \sim \nu}\left[ N^{-1} \right] \\
& + \big(R_1^{(j)} + R_4^{(j)} + (R_2^{(j)} + R_5^{(j)})\|f^j\|_\infty^2 + R_3^{(j)} L_{f^j}^2\big) \E_{N \sim \nu}\left[ N^{-\frac{1}{D_{\chi^j} + 1}} \right] \\
& +  \big(T_1^{(j)} + T_2^{(j)}\|f^j\|_\infty^2\big)  \E_{N \sim \nu}\left[ \frac{\log(N)}{N^{\frac{1}{D_\chi^j +1}} } \right] \Bigg)+ \mathcal{O}\left( \E_{N \sim \nu} \left[\exp(-N)N^{3T-\frac{3}{2}} \right]\right)\Bigg) ,
\end{aligned}
\]
where $S_l^{(j)},R_l^{(j)},T_l^{(j)}$ are the according constants from Theorem \ref{thm:unifExpValue} for each class $j$ and are defined in (\ref{eq:thmC10-constants}).
All in all, we get 
\[
\begin{aligned}
  &   \E_{\mathcal{T} \sim \mu_{\mathcal{G}}^m} \left[ \sup_{\Theta \in \mathrm{Lip}_{L,B}} \left( \frac{1}{m} \sum_{j=1}^\Gamma  \sum_{i=1}^{m_j} \mathcal{L}(\Theta_{G_i^j}^P(\mathbf{f}_i^j),y_j) - \E_{(G, \mathbf{f} , y)\sim \mu_{\mathcal{G}}}\left[ \mathcal{L}(\Theta_{G}^P(\mathbf{f}), y) \right] \right)^2 \right] \\
  & \leq 2^\Gamma \frac{8 \|\mathcal{L}\|_\infty^2}{m} \pi + \frac{\sqrt{\pi}}{ m} 2^\Gamma  \Gamma  \sum_{j=1}^\Gamma \gamma_j L_\mathcal{L}^2 \Bigg(  6\sqrt{\pi} \Big( \big(S_1^{(j)} + S_3^{(j)} + (S_2^{(j)} + S_4^{(j)})\|f^j\|_\infty^2 \big) \E_{N \sim \nu}\left[ N^{-1} \right] \\
& + \big(R_1^{(j)} + R_4^{(j)} + (R_2^{(j)} + R_5^{(j)})\|f^j\|_\infty^2 + R_3^{(j)} L_{f^j}^2\big) \E_{N \sim \nu}\left[ N^{-\frac{1}{D_{\chi^j} + 1}} \right] \\
& +  \big(T_1^{(j)} + T_2^{(j)}\|f^j\|_\infty^2\big)  \E_{N \sim \nu}\left[ \frac{\log(N)}{N^{\frac{1}{D_\chi^j +1}} } \right] \Big)+ \mathcal{O}\left( \E_{N \sim \nu} \left[\exp(-N)N^{3T-\frac{3}{2}} \right]\right)\Bigg).
\end{aligned}
\]
 We define 
\begin{equation}
\label{eq:defEndConstant}
    C= 6 \sqrt{\pi} \max_{j=1, \ldots, \Gamma} \left(\sum_{i=1}^4S_i^{(j)} + \sum_{i=1}^5R_i^{(j)} + \sum_{i=1}^2T_i^{(j)}\right),
\end{equation}
 leading to 
\[
\begin{aligned}   & \E_{\mathcal{T}\sim \mu^m_\mathcal{G}}\Big[\sup_{\Theta \in \mathrm{Lip}_{L,B}}\Big(R_{emp}(\Theta^P)   - R_{exp}(\Theta^P) \Big)^2 \Big]  \\ &\leq   \frac{2^\Gamma8\|\mathcal{L}\|_\infty^2\pi}{m} + \frac{2^\Gamma \Gamma L_{\mathcal{L}}^2   C}{m} \sum_{j}  \gamma_j \big(1 + \|f^j\|^2_\infty + L_{f^j}^2\big) 
 \\  & \quad \ \ \quad \ \ \quad \ \ \quad \ \ \quad \ \  \quad \ \ \quad \ \ \quad \ \ \cdot      \left( \E_{N \sim \nu} \left[ \frac{1}{N} + \frac{1 + \log(N)}{N^{1/D_{\chi^j} + 1}} + \mathcal{O}\left( \exp(-N)N^{3T-\frac{3}{2}} \right) \right] \right).
\end{aligned}
\]
\end{proof}

\subsection{Asymptotics of the Generalization Bound}
\label{subsec:disGenBound}
In this  subsection, we derive the asymptotic dependency of our generalization bound in Theorem \ref{thm:mainGenError} with respect to the uniform Lipschitz  bound $L$ of the message and update function, the depth $T$, the maximal hidden dimension $h$ and the average graph size, that we denote in this section by abuse of notation $N$. Since we bound the expected square generalization error, and most other related generalization bounds are formulated in high probability, we transform our bound in expectation to a bound in high probability, using, e.g., Markov's Inequality  (and then taking the square root of the square error). By this, the comparison with other generalization bounds formulated in high probability are valid.  Hence, we focus on the constant $\sqrt{C}$, where $C$ is the constant from Theorem \ref{thm:mainGenError}. We reformulated Theorem \ref{thm:mainGenError} as Theorem \ref{thm:reformulatedTheorem2}, where we observed that  $C \leq 6 \sqrt{\pi} \max_{j=1, \ldots, \Gamma} \left(\sum_{i=1}^4S_i^{(j)} + \sum_{i=1}^5R_i^{(j)} + \sum_{i=1}^2T_i^{(j)}\right)$, where $S_l^{(j)},R_l^{(j)},T_l^{(j)}$ are the according constants from Theorem \ref{thm:unifExpValue} for each class $j$ and are defined in (\ref{eq:thmC10-constants}). 
For a better presentation, we drop the  class-superscript by setting $S_l = \max_j S_l^{(j)}$, for $l=1,\ldots,4$, $R_l = \max_j R_l^{(j)}$, for $l=1,\ldots,5$ and $T_l = \max_j T_l^{(j)}$, for $l=1,2$. Further, denote $C_\chi = \max_j C_{\chi^j}$, $D_\chi = \max_j D_{\chi^j}$, $L_W = \max_j L_{W^j}$  and $\|W\|_\infty= \max_j \|W^j\|_\infty$. 

The constants $R_i, S_i$ and  $T_i$ are bounded by a polynomial of order  $2$ in $\Omega_j$, for $j=1, \ldots, 9$, defined in (\ref{eq:defConstantsUniform}).  
    The constants $\Omega_j$, $j=1, \ldots, 9$,  depend on a polynomial of degree one  in $Z_1^{(l)}, Z_2^{(l)}, Z_3^{(l)}$, $B_1^{(l)},  B_2^{(l)}$ and on a polynomial of degree at most $T-1$ in $K^{(l)}$ for $l=1, \ldots, T-1$. Here, $Z_1^{(l)},Z_2^{(l)},Z_3^{(l)}$ are defined in (\ref{eq:z1z2z3}),  $B_1^{(l)}$ and $B_2^{(l)}$ are defined in (\ref{eq:B'}) and (\ref{eq:B''}), and 
    \[
K^{(l')}  = \sqrt{(L_{\Psi^{(l')}})^2 
  + 
   \frac{8\|W\|_\infty^2}{\cmin^2} (L_{\Phi^{(l')}})^2 (L_{\Psi^{(l')}})^2}.
\]
 
 Hence, our strategy is as follows. We first work out the asymptotic behaviour of $Z_1^{(l)}, Z_2^{(l)}, Z_3^{(l)}$, $B_1^{(l)},  B_2^{(l)}$ and  $K^{(l)}$ for $l=1, \ldots, T-1$ with respect to the parameters. Then,  we derive the asymptotics of $\Omega_j$, $j=1, \ldots, 9$. These already agree with the asymptotic of $\sqrt{C}$.
 For this, we write
 $A \lesssim x^k$ if $A$ is bounded by a polynomial of order $k$ in $x$.
 
 We begin with observing that  $K^{(l')} \lesssim L^2 \frac{\|W\|_\infty}{\mathrm{d}_{\mathrm{min}}}$. 
Since we only consider MPNNs $\Theta \in \mathrm{Lip}_{L, B}$, we have for $l=1, \ldots, T$,
\[
\begin{aligned}
B^{(l)}_1 & \leq \sum_{k=1}^{l}  \big(
L_{\Psi^{(k)}} \frac{\cmax}{\cmin}\|\Phi^{(k)}(0,0)\|_\infty+ \|\Psi^{(k)}(0,0)\|_\infty \big) \prod_{l' = k+1}^{l}  L_{\Psi^{(l')}} \big( 1 + \frac{\cmax}{\cmin}  L_{\Phi^{(l')}} \big) \\
& \lesssim \sum_{k=1}^l L B \frac{\|W\|_\infty}{\mathrm{d}_{\mathrm{min}}} \left(
\frac{\|W\|_\infty}{\mathrm{d}_{\mathrm{min}}} L^2
\right)^{l-k} \lesssim \frac{\|W\|_\infty^l}{\mathrm{d}_{\mathrm{min}}^l } L^{2l} B .
\end{aligned}
\]
and 
\[
\begin{aligned}
    B_2^{(l)} & \leq \prod_{k = 1}^{l} L_{\Psi^{(k)}} \left(1  + \frac{\cmax}{\cmin}  L_{\Phi^{(k)}} \right) \\
    & \lesssim \frac{\|W\|_\infty^l}{\mathrm{d}_{\mathrm{min}}^l } L^{2l}.
\end{aligned}
\]

For $l=1, \ldots, T$, the constants $Z_1^{(l)}, $ are defined in (\ref{eq:z1z2z3}). We have
\[
\begin{aligned}
 Z_1^{(l)} & \leq \sum_{k=1}^{l}  \Bigg(\Big(
L_{\Psi^{(k)}}\frac{\cl}{\cmin}  \|\Phi^{(k)}(0,0)\|_\infty  +
L_{\Psi^{(k)}}\cmax
\|\Phi^{(k)}(0,0)\|_\infty 
 \frac{\cl}{\cmin^2}\Big)  \\
 & +  B_1^{(k-1)} \Big(
L_{\Psi^{(k)}}\frac{\cl}{\cmin}  L_{\Phi^{(k)}}   +
L_{\Psi^{(k)}}\cmax
 L_{\Phi^{(k)}} 
 \frac{\cl}{\cmin^2}\Big) \Bigg)  \prod_{l' = k+1}^{l} L_{\Psi^{(l')}}  \Big(1+\frac{\cmax}{\cmin} L_{\Phi^{(l')}}  \Big) \\
 & \lesssim \sum_{k=1}^l B_1^{(k-1)} L^2 \frac{L_W}{\mathrm{d}_{\mathrm{min}}^2} \left(\frac{\|W\|_\infty}{\mathrm{d}_{\mathrm{min}}}L^2\right)^{l-k} \\
 & \lesssim \sum_{k=1}^l B \|W\|_\infty \frac{\|W\|_\infty^{k-1}}{\mathrm{d}_{\mathrm{min}}^{k-1}} (L^2)^{k-1}  L^2 \frac{L_W}{\mathrm{d}_{\mathrm{min}}^2} \left(\frac{\|W\|_\infty}{\mathrm{d}_{\mathrm{min}}}L^2\right)^{l-k} \\
 & \lesssim B \frac{\|W\|_\infty^{l} L_W}{\mathrm{d}_{\mathrm{min}}^{l+1}} L^{2l}.
\end{aligned}
\]
We have
\[
\begin{aligned}
 Z_2^{(l)} & \leq \sum_{k=1}^{l}   B_2^{(k)}   \Big(
L_{\Psi^{(k)}}\frac{\cl}{\cmin}  L_{\Phi^{(k)}}   +
L_{\Psi^{(k)}}\cmax
 L_{\Phi^{(k)}} 
 \frac{\cl}{\cmin^2}\Big)  \prod_{l' = k+1}^{l} L_{\Psi^{(l')}}  \Big(1+\frac{\cmax}{\cmin} L_{\Phi^{(l')}}  \Big) \\
 & \lesssim \sum_{k=1}^l B_2^{(k-1)} L^2 \frac{L_W}{\mathrm{d}_{\mathrm{min}}}\|W\|_\infty \left(\frac{\|W\|_\infty}{\mathrm{d}_{\mathrm{min}}} L^2\right)^{l-k} \\
 & \lesssim \sum_{k=1}^l \frac{\|W\|^{k-1}_\infty}{\mathrm{d}_{\mathrm{min}}^{k-1}} (L^2)^{k-1} L^2 \frac{L_W}{\mathrm{d}_{\mathrm{min}}}\|W\|_\infty \left(\frac{\|W\|_\infty}{\mathrm{d}_{\mathrm{min}}} L^2\right)^{l-k} \\
 & \lesssim \frac{L_W  \|W\|_\infty^l}{\mathrm{d}_{\mathrm{min}}^l} L^{2l}.
\end{aligned}
\]
We have
\[
\begin{aligned}
Z_3^{(l)}  &   \leq    \prod_{k=1}^{l}
L_{\Psi^{(k)}} \Big(1+\frac{\cmax}{\cmin} L_{\Phi^{(k)}}\Big) \\
& \lesssim \frac{   \|W\|_\infty^l}{\mathrm{d}_{\mathrm{min}}^l} L^{2l}.
\end{aligned}
\]
 
 For $i=1, \ldots, 9$, the constant $\Omega_i$ depends on $K^{(l)}$ for which we 
 have
 \[
 K^{(l')}  \leq \sqrt{(L_{\Psi^{(l')}})^2 
  + 
   \frac{8\|W\|_\infty^2}{\cmin^2} (L_{\Phi^{(l')}})^2 (L_{\Psi^{(l')}})^2} \lesssim \frac{\|W\|_\infty}{\mathrm{d}_{\mathrm{min}}} L^2
 \]
 
 For $\Omega_1$, we calculate 
 \[
 \begin{aligned}
 \Omega_1  &   \leq  \sum_{l=1}^{T} L_{\Psi^{(l)}}  
4 \frac{\zeta \cl  \big(\sqrt{\log (C_\chi)} +  \sqrt{D_\chi} \big)}{\mathrm{d}_{min}^2} \|W\|_\infty\big(L_{\Phi^{(l)}} B_1^{(l-1)}+ \|\Phi^{(l)}(0,0)\|_\infty \big) \prod_{l' = l+1}^{T} K^{(l')} \\ & \lesssim \big(\sqrt{\log (C_\chi)} +  \sqrt{D_\chi} \big) \sum_{l=1}^{T} L^2 B_1^{(l-1)} \frac{L_W\|W\|_\infty}{\mathrm{d}_{\mathrm{min}}^2} (L^2)^{T-l} \\
& \lesssim \big(\sqrt{\log (C_\chi)} +  \sqrt{D_\chi} \big)B L^{2T} \frac{L_W\|W\|_\infty^{T}}{\mathrm{d}_{\mathrm{min}}^{T+1}}
 \end{aligned}
 \]
 
 Similar calculations lead to 
 \[
 \Omega_i \lesssim \big(\sqrt{\log (C_\chi)} +  \sqrt{D_\chi} \big)B (L^2)^{T}  \frac{L_W \|W\|_\infty^{T}}{\mathrm{d}_{\mathrm{min}}^{T+1}}.
 \]

Hence,
\begin{equation}
\label{eq:GenBoundAsyptotics}
  GE \lesssim \frac{2^{\Gamma/2}}{\sqrt{m}} +  \frac{2^{\Gamma/2} \big(\sqrt{\log (C_\chi)} +  \sqrt{D_\chi} \big)B L^{2T}  L_W \|W\|_\infty^{T}}{\sqrt{m}\mathrm{d}_{\mathrm{min}}^{T+1}} \E_{N \sim \nu}\left[ \frac{\sqrt{\log(N)}}{N^{\frac{1}{2(D_\chi + 1)}}} \right].
\end{equation}

\subsection{Generalization Bound Comparison}
\label{subsec:comp}
In this subsection, we compare our generalization bound, especially the asymptotics derived in the previous subsection, with other related generalization bounds.
Since related work does neither consider the same network architecture, nor the same data distribution as our work, we emphasize the setting of each of the cited results. We then write the asymptotics of the cited bounds in terms of the maximal hidden dimension $h$, depth $T$, Lipschitz bound  $L$ of the message and update functions, maximal node $d$ degree and graph size $N$. We recall (\ref{eq:GenBoundAsyptotics}), where we derived the asymptotics of our generalization bound from Theorem \ref{thm:mainGenError} with respect to $T, L$ and $N$ as
\[
\mathcal{O}\left(\E_{N \sim \nu}\left[ \frac{\sqrt{\log(N)}}{N^{\frac{1}{2(D_\chi + 1)}}} \right] \right) \; ,   \; \quad \ \   \mathcal{O}\left( L^{2T} \right) 
\]
and $\mathcal{O}(1)$ with respect to  $h$.

\subsubsection{PAC-Bayesian Approach based Bound}
 \label{sub:pacbound}

The generalization analysis of \cite{liao2021a} considers  MPNNs  with sum aggregation for a $K$-class graph classification setting. The authors differentiate between the input   node feature  vectors $\mathbf{x}_v$, which is an unchanged input for every layer, and the node embedding/representation in the $l$-th layer $\mathbf{f}^{(l)}$, where they take $\mathbf{f}^{(0)}=0$.  More formally, the MPNNs takes the following form.
\begin{definition}
\label{def:PACMPNN}
Let $G$ be a graph with graph features $\mathbf{x}$. A MPNN   (in \cite{liao2021a}) with $T$ layers is defined by taking the input feature representation $\mathbf{f}^{(0)}=0$, and mapping it to the features $\mathbf{f}^{(l)}$ in the $l$-th layer, which are defined recursively by
\begin{equation}
\label{eq:PACnetwork}
   \mathbf{f}_v^{(l)} = \Psi\left( 
W_1 \mathbf{x}_v + W_2 \rho\left( 
\sum_{u\in \mathcal{N}(v)} \Phi(\mathbf{f}_u^{(l-1)})
\right) 
\right),
\end{equation}
where $\rho$, $\Psi$ and $\Phi$ are nonlinear transformations, and $W_1$ and $W_2$ are linear transformations. This is followed by a global pooling layer, which takes as an input $\mathbf{f}^{(T-1)}\in \mathbb{R}^{N\times K}$, and returns the vector
\[ \frac{1}{N} \mathbf{1}_N \mathbf{f}^{(T-1)} W_T \in \mathbb{R}^{1 \times K},  \]
where $W_T$ is a linear transformation.
 Here $\mathbf{1}_N$ denotes the vector $(1, \ldots, 1) \in \mathbb{R}^{1 \times N}$, where $N$ is the number of nodes in the graph.
\end{definition}

The message and update functions in Definition \ref{def:PACMPNN} are the same in every layer. 
It is assumed 
 that  $\Psi, \rho$ and $\Phi$ have Lipschitz constants $L_\Psi, L_\rho$ and $L_\Phi$. Furthermore it is assumed that $W_1, W_2$ and $W_T$ have bounded norms, i.e., $\|W_1\|_2 \leq B_1, \|W_2\|_2 \leq B_2$ and $\|W_T\|_2 \leq B_T$.

The expected multiclass margin loss is then defined as 
\[
R_{\mathcal{D}, \gamma}(\Theta) = \mathbb{P}_{  (G, \mathbf{x}, \mathbf{y}) \sim \mathcal{D}} \left(
\big(\Theta^P_G(\mathbf{x})\big)_{\mathbf{y}} \leq \gamma + \max_{j \neq \mathbf{y}} \big(\Theta^P_G(\mathbf{x})\big)_j
\right),
\]
where $\mathcal{D}$ is the unknown data distribution,
  $\gamma > 0$ and $\Theta_G^P$ is the MPNN after pooling.   Accordingly, the empirical loss is defined
as 
\[
R_{\mathcal{T}, \gamma}(\Theta) = \frac{1}{m} \sum_{(G_i,\mathbf{x}_i, \mathbf{y}_i) \in \mathcal{T}} \mathbbm{1} \left(
\big(\Theta^P_{G_i}(\mathbf{x}_i)\big)_{\mathbf{y}_i} \leq \gamma + \max_{j \neq\mathbf{y}_i} \big(\Theta^P_{G_i}(\mathbf{x}_i)\big)_j
\right),
\]
where the summand $ \mathbbm{1} \left(
\big(\Theta^P_{G_i}(\mathbf{x}_i)\big)_{\mathbf{y}_i} \leq \gamma + \max_{j \neq\mathbf{y}_i} \big(\Theta^P_{G_i}(\mathbf{x}_i)\big)_j
\right)$ is equal to $1$ if   $\big(\Theta^P_{G_i}(\mathbf{x}_i)\big)_{\mathbf{y}_i} \leq \gamma + \max_{j \neq\mathbf{y}_i} \big(\Theta^P_{G_i}(\mathbf{x}_i)\big)_j$ and otherwise $0$.

Furthermore, the following assumptions hold for the training set and the considered MPNNs
\begin{assumption} $ $
\label{ass:PACbound}
\begin{enumerate} 
    \item The training set $\mathcal{T} = \{(G_1, \mathbf{x}_1, \mathbf{y}_1), \ldots, (G_m, \mathbf{x}_m, \mathbf{y}_m)  \}$ is drawn i.i.d. from some distribution $\mathcal{D}$, where all graphs are simple and have node degrees at most $d-1$.
    \item The maximum hidden dimension across all layers is $h$.
    \item The node features are drawn in an $l^2$-ball with radius $B$ from the node feature space $\mathcal{X}$. 
\end{enumerate}
\end{assumption}

 The generalization bound is formulated in terms of the following constants: $\zeta = \min\left( \|W_1\|_2, \|W_2\|_2, \|W_T\|_2 \right)$, $|w|^2_2 = \|W_1\|_F^2 + \|W_2\|_F^2 + \|W_T\|_F^2$, $\lambda = \|W_1\|_2 \|W_T\|_2$, 
$\xi= L_\Psi \frac{(d\mathcal{C})^{l-1}-1}{d\mathcal{C}-1}$, and the percolation complexity $\mathcal{C} = L_\Psi L_\rho L_\Phi\|W_2\|_2$. We summarize the main result \cite[Theorem 3.4]{liao2021a}   as follows.

\begin{theorem} 
Let $T> 1$. 
 Then for any $\delta, \gamma >0$, with probability at least $1-\delta$ over the choice of the training set $\mathcal{T}\sim\mathcal{D}^m$  of $m$ graphs,  for any $T$-layered MPNN $\Theta$, we have, 
\begin{enumerate}
    \item If $d \mathcal{C} = 1$, then
\[
\begin{aligned}
R_{\mathcal{D},0}(\Theta) & \leq R_{\mathcal{T}, \gamma}(\Theta) \\
& + \mathcal{O}\left( \sqrt{\frac{B^2  \max\left(
\zeta^{-6}, \lambda^3L_\Psi^3\right) (T+1)^4h \log(Th) |w|_2^2  + \log \frac{m}{\delta}} {\gamma^2 m}} 
\right).
\end{aligned}
\]
\item 
If $d \mathcal{C} \neq 1$, then
\[
\begin{aligned}
R_{\mathcal{D},0}(\Theta) & \leq R_{\mathcal{T}, \gamma}(\Theta) \\
& + \mathcal{O}\left( \sqrt{\frac{B^2 \left( \max\left(
\zeta^{-(T+1)}, (\lambda \xi)^{(T+1)/T}\right) \right)^2T^2 h \log(Th)  |w|_2^2  + \log \frac{m(T+1)}{\delta}} {\gamma^2 m}}
\right).
\end{aligned}
\]
\end{enumerate}
\end{theorem}

We only consider the non-degenerative case $d \mathcal{C} \neq 1$, as it is the generic case, which can again be split into two cases. As the authors in \cite{liao2021a} mention, 
these two cases correspond to $\max(\zeta^{-1}, (\lambda \xi)^{\frac{1}{T}})  = \zeta^{-1}$ (case A) and $\max(\zeta^{-1}, (\lambda \xi)^{\frac{1}{T}})  = (\lambda \xi)^{\frac{1}{T}}$ (case B). In practice case B  occurs more often, where  the generalization bound depends on the parameters with orders $\mathcal{O}\left( d^{\frac{(T+1)(T-2)}{T}}\right)$, $\mathcal{O}\left( \sqrt{h \log h} \right)$
and 
$\mathcal{O}\left( \lambda^{1+\frac{1}{T}}\xi^{1+\frac{1}{T}}\sqrt{\|W_1\|_F^2 + \|W_2\|_F^2 +\|W_l\|_F^2}
\right)$.  In case A, the generalization bound depends on the parameters with orders $\mathcal{O}\left( \sqrt{h \log h} \right)$
and 
$\mathcal{O}\left(  \zeta^{-(T+1)} \sqrt{\|W_1\|_F^2 + \|W_2\|_F^2 +\|W_l\|_F^2}
\right)$.

We now describe  the architecture in Definition \ref{def:PACMPNN} in terms of the message passing framework from (\ref{eq:gMPNN}). For $l=1, \ldots, T$, we denote by $\mathbf{m}_i^{(l)}$ and $\mathbf{f}_i^{(l)}$ the message and graph feature of node $i$ in the $l$-th layer, respectively.
Given a simple graph $G$ with node features $(\mathbf{x}_i)_{i}$,  we set $\mathbf{f}_i = \mathbf{x}_i$ as the input for the MPNN. Then the message function in the first layer  is given by $ \Phi^{(1)}(\mathbf{f}_i, \mathbf{f}_j ) = \mathbf{f}_i$. We recall that the message in MPNNs with sum aggregation is calculated as $\mathbf{m}_i^{(1)} = \sum_{j \in \mathcal{N}(i)} \Phi^{(1)}(\mathbf{f}_i, \mathbf{f}_j )$. The update function in the first layer is given by $\Psi^{(1)}(\mathbf{f}_i, \mathbf{m}_i^{(1)} ) = \big(\Phi(W_1 \mathbf{f}_i ), \mathbf{f}_i \big)$.
For $l=2, \ldots, T-1$, the message functions are defined as
\[
\Phi^{(l)}(\mathbf{f}_i^{(l-1)}, \mathbf{f}_j^{(l-1)}) = \Phi( \mathbf{f}_i^{(l-1)} )
\]
and the update functions are defined as 
\[
\Psi^{(l)}(\mathbf{f}_i^{(l-1)}, \mathbf{m}_i^{(l)}) =  \Psi  \Big( W_1 (\mathbf{f}_i^{(l-1)})_2 + W_2 \rho\big(\mathbf{m}_i^{(l)} \big), (\mathbf{f}_i^{(l-1)})_2 \Big),
\]
where $(\mathbf{f}_i^{(l-1)})_2$ stays unchanged through all layers, and is equal to the input graph features $\mathbf{x}_i$. The aggregation scheme is given by sum aggregation.  Finally, the pooling in Definition \ref{def:PACMPNN} can be described by a graph MPNN layer with update function $W_T$ followed by average pooling.  With this construction of  message and update functions the MPNN $\Theta = \big( (\Psi^{(l)})_{l=1}^T, (\Phi^{(l)})_{l=1}^T\big)$ with sum aggregation matches the architecture in Definition \ref{def:PACMPNN}.

We summarize the Lipschitz bounds for the message and update functions by $L_{\Phi^{(1)}} = 1, L_{\Phi^{(T)}} = 1, L_{\Psi^{(1)}} = L_\Phi, L_{\Phi^{(T)}} = \|W_T\|_2$ and $L_{\Phi^{(l)}} = L_\Phi, L_{\Psi^{(l)}} = L_\Psi\big(\|W_1\|_2 + \|W_2\|_2 L_\rho\big) $ for $l=2, \ldots, T-1$.
For deriving our generalization bound in Theorem \ref{thm:MainInProb}, we assume that there exists a uniform Lipschitz bound for the message and update functions, denoted by $L$. Hence, we assume that $L_\Phi \leq L$,  $L_\Psi \|W_1\|_2 +  L_\Psi\|W_2\|_2 L_\rho \leq L$ and $\|W_T\|_2 \leq L$.

For simplicity and better comparison with our generalization bound, we make use of the following upper bounds, 
\begin{equation}
\label{eq:PACupperBounds}
  \begin{aligned}
& \mathcal{C} = L_\Psi L_\rho L_\Phi \|W_2\| \leq L^2, \\
& \xi = L_\Psi \frac{(d \mathcal{C} )^{T-1} - 1}{d\mathcal{C}-1}\leq L(L^2)^{T-2}, \\
&\zeta = \min( \|W_1\|_2, \|W_2\|_2, \|W_l\|_2 ) \leq L \text{ and } \\
& \lambda = \|W_1\|_2  \|W_l\|_2 \leq L.
\end{aligned}
\end{equation}
This leads to
\[
\begin{aligned}
\mathcal{O}\left( \lambda^{1+\frac{1}{l}}\xi^{1+\frac{1}{l}}\sqrt{\|W_1\|_F^2 + \|W_2\|_F^2 +\|W_l\|_F^2}
\right) & = \mathcal{O}\left(L^{1 + \frac{1}{T}} 
  (L(L^2)^{T-2})^{1 + \frac{1}{T}} L \right) \\ & = \mathcal{O}\left(L^{2T -2/T +1}\right).
\end{aligned}
\]
Hence, the asympotics of the generalization bound in \cite{liao2021a} with respect to the maximal hidden dimension $h$, the Lipschitz bound $L$, the depth $T$ and the maximum node degree $d$ can be summarized respectively as
\[
\begin{aligned}
& \mathcal{O}\left( d^{\frac{(T+1)(T-2)}{T}} \right)  \; , \;
  \mathcal{O}\left(\sqrt{h \log h}\right)  \; \text{ and } \;
  \mathcal{O}\left(L^{ 2T -2/T +1 }\right).
\end{aligned}
\]

\subsubsection{Rademacher Complexity based Bound}
We next analyze the bound derived in  \cite{pmlr-v119-garg20c}. 
Since \cite{pmlr-v119-garg20c} consider the same architecture, defined in Definition \ref{def:PACMPNN},  as  \cite{liao2021a}, we adopt the notation from Subsection \ref{sub:pacbound}.
The authors in \cite{pmlr-v119-garg20c} consider    a binary graph classification task with the same Assumptions \ref{ass:PACbound} on the training set and the MPNN as in \cite{liao2021a}.  
The main result  can be summarized as follows. 

\begin{theorem}
Let $T>1$. Then for any $\delta, \gamma >0$, with probability at least $1-\delta$ over the choice of the training set $\mathcal{T}\sim\mathcal{D}^m$  of $m$ graphs,  for any $T$-layered MPNN $\Theta$, we have, 
\[
\begin{aligned}
&R_{\mathcal{D},0}(\Theta)  \leq R_{\mathcal{T}, \gamma}(\Theta) \\
& + \mathcal{O}\left(
\frac{1}{\gamma m} + h \|W_T\|_2 Z \sqrt{\frac{\log\left(
\|W_T\|_2 \sqrt{m} \max\left(
Z, \xi \sqrt{h} \max\left(
B \|W_1\|_2, \bar{R}\|W_2\|_2
\right)
\right)
\right)}{\gamma^2 m}}  + 
\sqrt{\frac{\frac{1}{\delta}}{m}}
\right),
\end{aligned}
\]
where  $\bar{R}$ is a constant specified in \cite{pmlr-v119-garg20c} that satisfies $\bar{R} \leq L_\rho L_\Phi d B \|W_1\|_2 \xi$, and $Z = B \|W_1\|_2 \|W_T\|_2$.
\end{theorem}

 We only consider the case $\max\left( Z, \xi \sqrt{h} \max \left( BB_1, \bar{R}B_2 \right) \right) = \xi \sqrt{h} \bar{R}B_2 $, which is the generic case  (see  \cite[Subsection A.5.2]{liao2021a} for the other cases).  
%
Thus the generalization bound from \cite{pmlr-v119-garg20c} depends on the parameters with orders $ \mathcal{O}\left( d^{T-1} \sqrt{\log(d^{2T-3})} \right),  
  \mathcal{O}\left( h \sqrt{\log \sqrt{h}} \right) $ and $
  \mathcal{O}\left( \lambda \mathcal{C} \xi \sqrt{\log (\|W_2\|_2 \lambda \xi^2)} \right).
$
 
Similarly to   Subsection \ref{sub:pacbound}, we consider a uniform Lipschitz bound $L$ for the message and update functions.  We thus consider the upper bounds on $\xi, \lambda$ and $\mathcal{C}$, summarized in (\ref{eq:PACupperBounds}),
which leads to
\[
\begin{aligned}
& \mathcal{O}\left( \lambda \mathcal{C} \xi \sqrt{\log (\|W_2\|_2 \lambda \xi^2))} \right) = \mathcal{O}\left( L^{2T} \sqrt{\log (L^{4T-4}) } \right).
\end{aligned}
\]

Hence, the asympotics of the Rademacher based generalization bound in \cite{pmlr-v119-garg20c} with respect to the maximal hidden dimension $h$, the Lipschitz bound $L$, the depth $T$ and the maximum node degree $d$ can be summarized as
\[
\begin{aligned}
  \mathcal{O}\left( d^{T-1} \sqrt{\log(d^{2T-3} )} \right)  , \:
 \mathcal{O}\left( h \sqrt{\log\sqrt{h}}  \right)  \text{ and } \mathcal{O}\left( L^{2T} \sqrt{\log (L^{4T-4}) } \right).
\end{aligned}
\]

\paragraph{VC-Dimension Based Bound \cite{scarselli2018vapnik}}

The work by \cite{scarselli2018vapnik}  considers graph neural networks in supervised classification or regression tasks, where the input is a graph $G$ with graph feature map $\mathbf{x}$  and one node of interest $v$ in which we want to produce a prediction. They apply a recurrent graph neural network  on the graph $G$ with graph feature $\mathbf{x}$, and then evaluate the output graph feature map $\mathbf{f}$ only in $v$. They then calculate the loss between $\mathbf{f}(v)$  and its given desired target $\mathbf{y}$.
More formally, the training dataset $\mathcal{T}$ is defined as $\mathcal{T} = \{(G^i, \mathbf{x}^i, v^i, \mathbf{y}^i) \; |  \; 1 \leq i \leq m \}$, where $m$ is the number of graphs and each tuple $(G^i, \mathbf{x}^i, v^i, \mathbf{y}^i)$ denotes a graph $G^i$ with graph features $\mathbf{x}^i$,
the supervised node $v^i$, and the desired target $\mathbf{y}^i$ for that node.  

Given a graph $G=(V,E)$ with graph features $\mathbf{x}$ the graph neural network architecture is defined implicitly, as a method that solves a system of equations, and the solution is the output of the network. The equation is given by
\begin{equation}
    \label{eq:VC1}
\mathbf{f}_i = \sum_{j \in \mathcal{N}(i)} \Phi(\mathbf{x}_i, \mathbf{f}_j, \mathbf{x}_j), \text{  $\forall i \in V$  }
\end{equation} 
where $\Phi$ is a multi-layer-perceptron with  input $[\mathbf{x}_i, \mathbf{f}_j, \mathbf{x}_j]$, and the solution $\mathbf{f}$ to (\ref{eq:VC1}) is defined as the output of this part of the network. The output of the network $\mathbf{o}_i \in \mathbb{R}$ for the node $i$ is then defined   by
\begin{equation}
    \label{eq:VC2}
\mathbf{o}_i = g(\mathbf{x}_i, \mathbf{f}_i),
\end{equation}
where $g$ is a multi-layer-perceptron.
Given the training data set $\mathcal{T}$, the empirical loss  $R_{\mathrm{emp}}$ is then defined by the sum of the squared errors, i.e.,
\[
R_{\mathrm{emp}} = \sum_{i=1}^m (\mathbf{y}^i - \mathbf{o}_{v^i})^2.
\]
One way to solve the fixed point problem (\ref{eq:VC1}) is by a fixed point iteration, which means that we can interpret the architecture as a recurrent message passing network (theoretically with infinite depth), where all message functions in all layers are equal to $\Phi$.

 


\cite{scarselli2018vapnik} derive VC-dimension bounds for the mapping that takes as an input a tuple $(G,\mathbf{x},v)$ of a graph $G$ with features $\mathbf{x}$ and node of interest $v$ and outputs $\mathbf{o}_v$ as defined in (\ref{eq:VC1}) and (\ref{eq:VC2}). 
The VC-dimension bound  depends on the total number of parameters $p$ of the network and a predefined maximum graph size $N$. Furthermore, the bound for the VC-dimension depends on the choice of the activation function in the  MLPs $\Phi$ and $g$. If the activation is given by tanh and logistic sigmoid activations   the VC-dimension
scales as $\mathcal{O}(p^4 N^2)$. Since $p$ can be related to the maximum hidden dimension $h$ by $p \in \mathcal{O}(h^2)$, the VC-dimension scales as $\mathcal{O}(h^8 N^2)$. Consequently, the asymptotics of the generalization bounds in \cite{scarselli2018vapnik} with respect to $h$ and $N$ can be summarized as
\[
\mathcal{O}(h^4) \text{ and  } \mathcal{O}(N).
\]
 For piecewise polynomial activations the VC-dimension scales as $\mathcal{O}(h^4 \log(N)N)$, hence the generalization bound scales in this case as 
 \[
\mathcal{O}(h^2) \text{ and  } \mathcal{O}(\sqrt{\log(N)N})
 \]
 with respect to $h$ and $N$.

\section{Details on Numerical Experiments and Additional Experiments}
\label{appendix:numericalExp}

In this section we report additional experiments and write all details corresponding to Section \ref{sec:NumResults}. We First give an example that illustrate our convergence theorem (Theorem \ref{thm:MainInProb}), and then introduce a comparison between our generalization bound and the Rademacher complexity \cite{pmlr-v119-garg20c} and PAC-Bayesian \cite{liao2021a} bounds, evaluated on synthetic datasets.

\subsection{Convergence Experiments}
\label{Convergence Experiments}
In this section, we show simple numerical experiments 
on the convergence of sampled MPNNs  from  a random geometric graph model, on toy data. 
 We consider random geometric graphs  \cite{RGGPenrose}, 
 which can be described by using RGMs with the kernel $W(x,y) = \mathbbm{1}_{B_r(x)}(y)$ on $[0,1]^2$, equipped with the uniform distribution and the standard Euclidean norm. Here $\mathbbm{1}_{B_r(x)}$ is the indicator function of the ball around $x$ with radius $r$. Even though $\mathbbm{1}_{B_r(x)}(y)$ is not Lipschitz continuous, and hence does not satisfy the conditions of Theorems  \ref{thm:MainInProb}
 , $\mathbbm{1}_{B_r(x)}(y)$ can be approximated by a Lipschitz continuous function.
 As the metric-space signal we consider a random low frequency signal (see Figure \ref{tbl:conv1}). 
 
 For our network, we choose untrained MPNNs with random weights, where each layer is defined using EdgeConv \cite{Bronstein_2017} with mean aggregation, and is implemented using Pytorch Geometric \cite{Fey/Lenssen/2019}.
More precisely, we consider MPNNs with $2$ layers. The message function in the first layer is defined as 
$\Phi^{(1)}(\mathbf{f}_i, \mathbf{f_j}) = h^{(1)}  ( \mathbf{f}_i, \mathbf{f}_j - \mathbf{f}_i)$, where $h^{(1)}$ is a $1$-layered MLP with ReLU activation, input dimension $2$ and output dimension $3$.  The message function in the second layer is defined as 
$\Phi^{(2)}(\mathbf{f}_i^{(1)}, \mathbf{f_j}^{(1)}) = h^{(2)}   ( \mathbf{f}_i^{(1)}, \mathbf{f}_j^{(1)} - \mathbf{f}_i^{(1)})$, where $h^{(2)}$ is a $1$-layered MLP with ReLU activation, input dimension $6$ and output dimension $1$. The update functions are given by $\Psi(\mathbf{f}_i^{(1)}, \mathbf{m}_i^{(2)}) = \mathbf{m}_i^{(2)}$.  This is a followed by an average pooling layer. 

 We ran the experiments that depend on random variables 10 times and report the average results with error bars that indicate the standard error. 
 One run consists of the following steps.
 We consider 10 different graph sequences, where each graph sequence contains randomly sampled graphs of $2^i$ nodes, with $i = 1,\ldots, 13$.
We then consider 50 (different) randomly initialized MPNNs, and compute for each graph sequence the worst-case error  between the  output of the cMPNN to its sampled graphs, i.e., for every graph size $N$, we pick the MPNN with the highest error. We then average the resulting 10 errors over the 10 different graph sequences, to approximate the expected error over the choice of the graph.
In Figure \ref{tbl:conv1}, we plot the average error over the 10 runs     on the logarithmic y-axis and the number  of nodes on the x-Axis. We also provide a log-log-graph of this relation. Recall that in a log-log-graph a function of the form $f(x) = x^c$ appears as a line with slope $c$. We observe that in this toy example the worst-case error, which corresponds roughly to the uniform convergence result in Theorem \ref{thm:MainInProb}, 
decays faster than our theoretical worst-case error bound  $-1/6$. This suggests that, at least for band limited signals  on random geometric graphs, our convergence bounds are not tight.  

\begin{figure}[t]
\begin{center}
    \begin{tabular}{M{60mm}M{60mm}}
  \includegraphics[width=\linewidth]{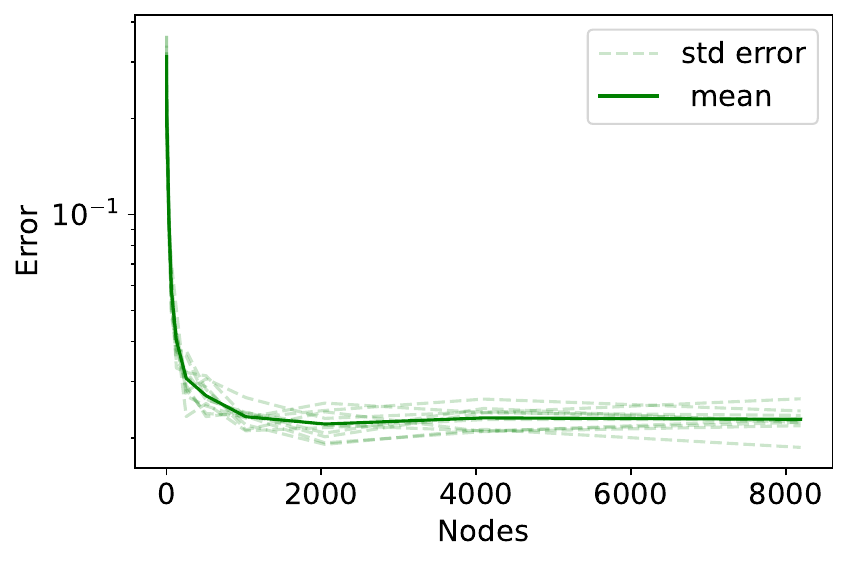} &  
  \includegraphics[width=\linewidth]{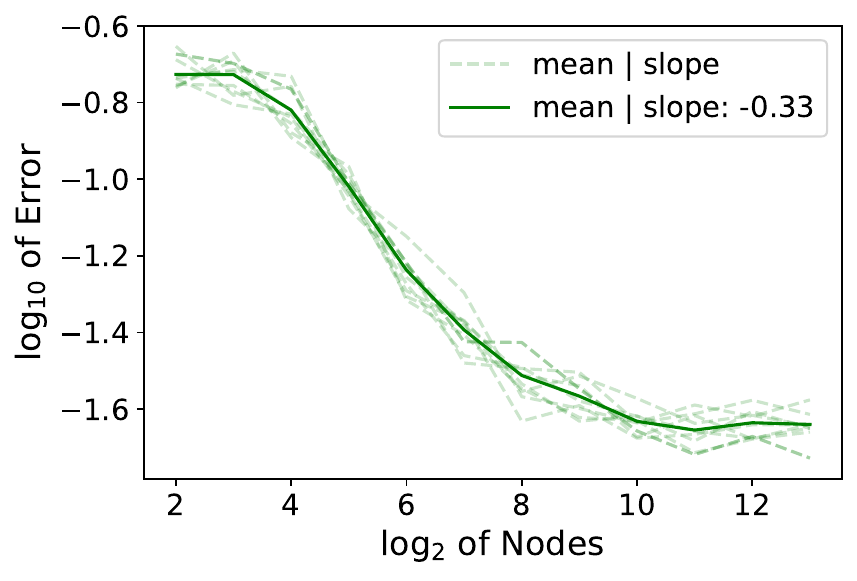}
  \end{tabular}
    \end{center}
    \caption{The average worst-case error  between MPNNs realized on graphs and on the limit RGM, with varying number of nodes, drawn from the  RGM $W(x,y)=\mathbbm{1}_{B_r(x)}(y)$ (where $\mathbbm{1}_{B_r(x)}$ is the indicator function of the ball around $x$ with radius $r=0.2$ in the space $([0,1]^2, \|\cdot\|_{\mathbb{R}^2}, \mathcal{L})$), and a random low frequency signal.  Left: graph sizes on the $x$-Axis and error on logarithmic $y$-Axis.   Right: $\log_2$ of the graph sizes on the $x$-Axis and $\log_{10}$ of the error on the $y$-Axis. The slope of the curve represents the exponential dependency of the error on $N$.} 
    \label{tbl:conv1}
\end{figure}



Computing the exact cMPNN would involve computing integrals. To approximate this integral, we sampled a large graph from the RGM. 
For the largest graph, we choose $2^{14}$ nodes. Our smaller graphs consist of $2^i$ nodes, with $i = 1,\ldots, 13${, and are sampled directly from the  RGM.} As the metric-space signal we   consider a discrete random band-limited signal of resolution 256x256, defined as $f = \mathcal{F}^{-1} (v)$, where $v$ consists of randomly chosen Fourier coefficients  in the low positive frequency band 20x20 such that the coefficients in the lowest positive frequency band 8x8 are amplified by a factor of $10$, 
 and $\mathcal{F}^{-1}$ is the inverse Finite Fourier Transform.


\subsection{Generalization Experiments}

In this subsection, we provide details for the numerical experiments from Section \ref{sec:NumResults} and report additional generalization experiments.

\subsubsection{Dataset}
\label{APP:Dataset}
We create three different synthetic datasets of random graphs from different random graph models. The domains of the graphons (the metric space), is taken as the Euclidean space $[0,1]$.  First, we consider Erdös-Rényi graphs with edge probably $0.4$ with constant signal, represented by $(W_1, f_1)$ with $W_1(x,y) = 0.4$ and $f_1(x) = 0.5$. We also consider
a smooth version of a stochastic block model, represented by $(W_2, f_2)$ with $W_2(x,y) = \sin(2\pi x)\sin(2\pi y)/2\pi+0.25$ and $f_2(x) = \sin(x)/2$. Last, we consider an exponential radial graphon, represented by $(W_3, f_3)$ with $W_3(x,y) = \exp(-|x-y|^2)/2$ and $f_3(x) = 0.5x$. For each graphon, we create 50K graphs of size 50. We call the Erdös-Renyi dataset \textbf{ER}, the stochstic block model dataset \textbf{SBM}, and the exponential radial dataset \textbf{EXP}. We then consider all possible pairs, i.e., \textbf{ER}-\textbf{SBM}, \textbf{ER}-\textbf{EXP} and \textbf{SBM}-\textbf{EXP}, and train a binary classifier for each pair. We split each dataset to 90\% training examples and 10\% test.

\begin{table}
\caption{Readout of the constants after training on all synthetic datasets. Each column 
represents the value of the respective dataset parameter used for the calculations of our generalization bounds.\label{table:datasetInfos}
}
\begin{center}
\vskip 0.15in
\begin{tabular}{cccccccc} 
\midrule
   &  $\|W\|_\infty$   & $L_W$ & $\|f\|_\infty$ & $f_L$ & $N$ & $m$ & $d_{\mathrm{min}}$ \cr  
\midrule
\makecell{\textbf{ER}-\textbf{SBM} }  &   $0.41$ & $0.5$ &  $0.5$ & $0.5$ & $50$ & $100$K & $0.25$   \\
\makecell{\textbf{ER}-\textbf{EXP}} & $0.5$ & $1$ & $0.5$ & $0.5$ & $50$ & $100$K & $0.373$  \\
\makecell{\textbf{EXP}-\textbf{SBM} } & $0.5$ & $1$ & $0.5$ & $0.5$ & $50$ & $100$K & $0.25$ \\   
\midrule
\end{tabular}
\end{center}
\end{table}

\subsubsection{MPNN Details}
\label{subsec:GraphSageDetail}
For our network, we choose MPNNs intialized with random weights, where each layer is defined using GraphSage \cite{hamilton2017inductive}, and is implemented with Pytorch Geometric \cite{Fey/Lenssen/2019}. We consider MPNNs with 1,2 and 3 layers. The message functions are defined by \[\Phi^{(l)}(\mathbf{f}^{(l-1)}_i, \mathbf{f}^{(l-1)}_j) = \mathbf{f}^{(l-1)}_j.\] The update functions are given by \[\Psi^{(l)}(\mathbf{f}^{(l-1)}_i, \mathbf{m}^{(l-1)}_i) = \rho\big(W_1^{(l)} \mathbf{f}^{(l-1)}_i+ W_2^{(l)} \mathbf{m}^{(l-1)}_i\big),\] where 
$W_1^{(1)} \in \mathbb{R}^{128 \times 1}$, $W_2^{(1)} \in \mathbb{R}^{128 \times 1}$ and  $W_1^{(2)}, W_1^{(3)}\in \mathbb{R}^{128 \times 128}$, $W_2^{(3)},W_2^{(3)} \in \mathbb{R}^{128 \times 128}$. We then consider a global mean pooling layer,  
 and apply a last linear layer $Q$ (including bias) with
input dimension $128$ and output dimension $2$. This last linear layer is seen as part of the loss function in the analysis, and contributes to the generalization bound via the Lipschitz constant and infinity norm of the loss, as seen in Theorem \ref{thm:mainGenError}.

\subsubsection{Experimental Setup}
The loss is given by soft-max composed with cross-entropy (composed on the last MLP).
We consider Adam with learning rate $lr=0.01$. For experiments with weight decay, we use an $l^2$-regularization on the weights with factors $0.27, 0.15$ and $0.05$ for the \textbf{ER}-\textbf{SBM} dataset. For the \textbf{ER}-\textbf{EXP} dataset we consider weight decay factors $0.37,0.15$ and $0.05$. For the \textbf{SBM}-\textbf{EXP} dataset we consider $0.28,0.05$ and $0.05$. 
We train for 1 epoch. The batch size is 64. We consider 1, 2 and 3 layers.

\subsubsection{Details on Computations of Our Bound}
We compute our generalization bound according to the full formula given in  Theorem \ref{thm:reformulatedTheorem2}. The terms  depending on the dataset are:
the size of the training dataset $m$, the average graph size $N$, the minimum degree $d$,
the largest infinity norm of the graphons
$\|W\|_\infty$,   largest Lipschitz norm of the graphons $L_{W}$, the largest infinity norm of the metric-space signal $\|f\|_\infty$, the largest Lipschitz norm of the metric-spaces signals $L_{f}$ and the number of classes is $\Gamma = 2$. For every dataset, we summarize these terms depending on the dataset in Table \ref{table:datasetInfos}.

Our bound depend also on the Lipschitz constants of the trained GraphSage MPNN, i.e., 
on the Lipschitz norms $L_{\Psi^{(l)}}$ and $L_{\Phi^{(l)}}$  of the update function $\Psi^{(l)}$ and message function $\Phi^{(l)}$, given in Subsection \ref{subsec:GraphSageDetail}. We have
$L_{\Phi^{(l)}} = \big\| [W_1^{(l)},W_2^{(l)}] \big\|_\infty$  and $L_{\Phi^{(l)}} = 1$. We readout the norms $\big\| [W_1^{(l)},W_2^{(l)}] \big\|_\infty$ for every layer, and plug it into our bound.
The bound also depends on the infinity norm and Lipschitz constant of the loss. We compute these constants in the next subsection.

\subsubsection{Computation of the Infinity Norm and Lipschitz Constant of the Loss}

Next we bound the Lipschitz constant and infinity norm of the loss. 
Namely, we derive properties of softmax composed on cross-entropy.
Softmax composed with the cross-entropy loss in the case of binary classes take the form
\[
\mathcal{L}_{{\rm CE}}(\mathbf{x};\mathbf{y}) = -y_1 \log \left(\frac{e^{x_1}}{e^{x_1} + e^{x_2}}\right) - y_2 \log \left(\frac{e^{x_2}}{e^{x_1} + e^{x_2}}\right),
\]
where $\mathbf{x}=(x_1,x_2)\in\mathbb{R}^2$ and $(y_1, y_2) \in \{e_1, e_2 \}$ depends on the target label, where $e_1=(1,0)$ and $e_2=(0,1)$. When the target label is fixed, we  write in short $\mathcal{L}_{{\rm CE}}(\mathbf{x}) :=\mathcal{L}_{{\rm CE}}(\mathbf{x};\mathbf{y})$.

\begin{lemma}
The loss $\mathcal{L}_{{\rm CE}}$ is Lipschitz continuous with Lipschitz constant 1.
Additionally, $\mathcal{L}_{{\rm CE}}$ is locally bounded in the following sense: \[\|\mathcal{L}_{{\rm CE}}\|_{L^\infty([-K, K]^2)} \leq \log(1+ e^{2K}),\]
where
$\|\mathcal{L}_{{\rm CE}}\|_{L^\infty([-K, K]^2)} = \max_{\mathbf{x}\in [-K,K]^2}\|\mathcal{L}_{{\rm CE}}(\mathbf{x})\|$.

\end{lemma}
\begin{proof}
We compute
\begin{align*}
\frac{\partial}{\partial x_1} \mathcal{L}_{{\rm CE}}(x_1, x_2) =& -y_1 \left(1 - \frac{e^{x_1}}{e^{x_1} + e^{x_2}}\right) + y_2 \frac{e^{x_1}}{e^{x_1} + e^{x_2}}\\
=& (y_1 + y_2)\frac{e^{x_1}}{e^{x_1} + e^{x_2}} - y_1\\
=& \frac{e^{x_1}}{e^{x_1} + e^{x_2}} - y_1
\end{align*}
Since $\frac{e^{x_1}}{e^{x_1} + e^{x_2}} \in [0, 1]$ and $y_1 \in \{0,1\}$, this implies
\[
\left|\frac{\partial}{\partial x_1} \mathcal{L}_{{\rm CE}}(x_1, x_2)\right| \leq 1.
\]
By symmetry we conclude that $\mathcal{L}_{{\rm CE}}$ is Lipschitz continuous with constant 1.

Last,  let $(x_1, x_2) \in [-K, K]^2$ and without loss of generality $y_1 = 1$ and $y_2 = 0$. We have
\begin{align*}
|\mathcal{L}_{{\rm CE}}(x_1, x_2)| =& - \log \left(\frac{e^{x_1}}{e^{x_1} + e^{x_2}}\right) \\
=& \log\left( 1 + \frac{e^{x_2}}{e^{x_1}} \right)\\
\leq & \log\left( 1 + e^{2K} \right).
\end{align*}
\end{proof}

The above lemma tells us that in order to bound the infinity norm of the loss we must bound the domain of the loss - the output of the MPNN. 
\begin{lemma}
Let $\Theta = \big((\Phi^{(l)})_{l=1}^T, (\Psi^{(l)})_{l=1}^T \big)$ be a MPNN s.t.  Assumption \ref{ass:graphon4}. is satisfied. Consider a  graph with $N$ nodes and a graph feature map  $\mathbf{f} \in \mathbb{R}^{N \times F}$.
Then, 
\[
 \| \Theta_G^P(\mathbf{f})\|_{\infty}  \leq  A' + A'' \|\mathbf{f}\|_{\infty; \infty},
\]
where
\[
\begin{aligned}
A' & = \sum_{l=1}^{T} \Big(
L_{\Psi^{(l)}}    \|\Phi^{(l )}(0,0)\|_\infty  + \|\Psi^{(l )}(0,0)\|_\infty
 \Big)   \prod_{l' = l+1}^T 
   L_{\Psi^{(l')}}
   \big( L_{\Phi^{(l')}} + 1\big)
   \end{aligned}
\]
and
\[
A'' = \prod_{l=1}^T 
 L_{\Psi^{(l)}} 
   \big(1+ L_{\Phi^{(l)}}\big).
\]
\end{lemma}
\begin{proof}
Let $G$ be a graph with weight matrix  $\mathbf{W} = (W_{i,j})_{i,j=1\ldots,N}$.
Let $l=0, \ldots, T-1$. Then, for $k =0, \ldots, l$, we have
\begin{equation}
    \label{eq:BoundgMPNN1}
\begin{aligned}
\|\mathbf{f}^{(k+1)}_i\|_\infty & =  \Big\| \Psi^{(k+1)} \Big(\mathbf{f}^{(k)}_i, \mathbf{m}_i^{(k+1)}\Big) \Big\|_\infty \\ 
& \leq \Big\| \Psi^{(k+1)} \Big(\mathbf{f}^{(k)}_i,\mathbf{m}_i^{(k+1)} \Big) - \Psi^{(k+1)}(0,0) \Big\|_\infty  +   \|\Psi^{(k+1)}(0,0)\|_\infty
\\
& \leq L_{\Psi^{(k+1)}} \Big(\| \mathbf{f}^{(k)}_i \|_\infty
+ \big\|\mathbf{m}_i^{(k+1)}\big\|_\infty \Big) + \|\Psi^{(k+1)}(0,0)\|_\infty,
\end{aligned}
\end{equation}
where $\mathbf{m}_i^{(k+1)} = \frac{1}{\sum_{j=1}^N W_{i,j}} \sum_{j=1}^N W_{i,j} \Phi^{(k+1)}\big( \mathbf{f}^{(k)}_i , \mathbf{f}^{(k)}_j\big)$.
For this message term, we have
\begin{equation}
    \label{eq:BoundgMPNN2}
\begin{aligned}
\|\mathbf{m}_i^{(k+1)} \|_\infty& = 
\left\|   \frac{1}{\sum_{j=1}^N W_{i,j}} \sum_{j=1}^N W_{i,j} \Phi^{(k+1)}\big( \mathbf{f}^{(k)}_i , \mathbf{f}^{(k)}_j\big) \right\|_\infty \\
&\leq 
\left\| \max_{j=1,\ldots, N} \Phi^{(k+1)}\big( \mathbf{f}^{(k)}_i , \mathbf{f}^{(k)}_j\big) \right\|_\infty \\
& \leq \max_{j=1, \ldots, N} L_{\Phi^{(k+1)}} \|\mathbf{f}_j^{(k)}\|_\infty + \|\Phi^{(k+1)}(0,0)\|_\infty.
\end{aligned}
\end{equation}

Denote $\|\mathbf{f}\|_{\infty; \infty} = \max_{i=1, \ldots, N} \max_{j=1, \ldots, F} | \mathbf{f}_{i,j} |$ for $\mathbf{f} \in \mathbb{R}^{N \times F}$.
We have as a result of (\ref{eq:BoundgMPNN1}) and (\ref{eq:BoundgMPNN2})
\[
\begin{aligned}
& \|\mathbf{f}^{(k+1)}\|_{\infty; \infty} \\ 
& \leq L_{\Psi^{(k+1)}} \Big( \|\mathbf{f}^{(k)}\|_{\infty;\infty}  +\big(
L_{\Phi^{(k+1)}} \|\mathbf{f}^{(k)}\|_{\infty; \infty} + \|\Phi^{(k+1)}(0,0)\|_\infty
\big)\Big) + \|\Psi^{(k+1)}(0,0)\|_\infty
\end{aligned}
\]
which we can write as
\[
\begin{aligned}
& \|\mathbf{f}^{(k+1)}\|_{\infty; \infty} \\ & \leq  L_{\Psi^{(k+1)}}\Big(  1
+ L_{\Phi^{(k+1)}} \Big) \| \mathbf{f}^{(k)} \|_{\infty; \infty} + L_{\Psi^{(k+1)}}\|\Phi^{(k+1)}(0,0)\|_\infty  + \|\Psi^{(k+1)}(0,0)\|_{ \infty}.
\end{aligned}
\]
We apply Lemma \ref{lemma:RecRecGen} to solve this recurrence relation, to get
\begin{align*}
\|\mathbf{f}^{(k)}\|_{\infty; \infty} \leq &\sum_{l=1}^k\left( L_{\Psi^{(l)}}\|\Phi^{(l)}(0,0)\|_\infty + \|\Psi^{(l)}(0,0) \|_\infty\right) \prod_{l'=l+1}^k L_{\Psi^{(l')}}(1 + L_{\Phi^{(l')}})\\
+& \|\mathbf{f}^{(0)}\|_{\infty; \infty} \prod_{l=1}^k L_{\Psi^{(l)}}(1 + L_{\Phi^{(l)}})
\end{align*}
Now, since for general bounded functions $F: \chi \to \mathbb{R}^n $ and $x_1,\ldots x_N \in \chi$
\[
\left\|\frac{1}{N}\sum_{i=1}^N F(x_i)  \right\|_\infty \leq \|F\|_{\infty,\infty},
\]
 the proof is done.
\end{proof}

Note that using  our analysis, for the MPNN architecture presented in Section \ref{subsec:GraphSageDetail}, the loss is not just $\mathcal{L}_{\rm CE}$, but the composition of $\mathcal{L}_{\rm CE}$ on the last linear layer of the network. We denote this total loss by $\mathcal{L}_{\rm total} = \mathcal{L}_{\rm CE}\circ Q$. Hence, in our analysis the Lipschitz constant of the total loss is bounded by
\[L_{\mathcal{L}_{\mathbf{\rm total}}} = \| Q \|_{\infty},\]
where $\| Q \|_{\infty}$ is the induced infinity norm of the matrix $Q$. The infinity norm bound of the total loss is bounded by
\[\| \mathcal{L}_{\rm total} \|_{\infty} \leq \log(1 + e^{2(\|Q\|_\infty K+b)}),\]
where $K$ is the infinity norm of the MPNN.

\subsubsection{Details on the  Computation of Bounds from Other Papers}
The papers \cite{liao2021a} and \cite{pmlr-v119-garg20c} do not provide generalization bounds for general MPNNs, but only for a specific architecture -- GNNs with mean field updates, as defined in Definition \ref{def:PACMPNN}, namely
\[ \mathbf{f}_i^{(l)} = \Psi\left( 
W_1 \mathbf{x}_i + W_2 \rho\left( 
\sum_{u\in \mathcal{N}(v)} \Phi(\mathbf{f}_j^{(l-1)})
\right) 
\right),\]
where $\rho$, $\Psi$ and $\Phi$ are nonlinear transformations, and $W_1$ and $W_2$ are linear transformations. This is followed by a global pooling layer, which takes as an input $\mathbf{f}^{(T-1)}\in \mathbb{R}^{N\times K}$, and returns the vector
\[ \frac{1}{N} \mathbf{1}_N \mathbf{f}^{(T-1)} W_T \in \mathbb{R}^{1 \times K},  \]
where $W_T$ is a linear transformation.
 Here $\mathbf{1}_N$ denotes the vector $(1, \ldots, 1) \in \mathbb{R}^{1 \times N}$, where $N$ is the number of nodes in the graph. 
As described in Subsection \ref{sub:pacbound}, GNNs with mean field updates are a special case of MPNNs.   

 The generalization bounds in \cite{liao2021a} and \cite{pmlr-v119-garg20c} are formulated in terms of the following constants:  $\zeta = \min\left( \|W_1\|_2, \|W_2\|_2, \|W_T\|_2 \right)$, $|w|^2_2 = \|W_1\|_F^2 + \|W_2\|_F^2 + \|W_T\|_F^2$, $\lambda = \|W_1\|_2 \|W_T\|_2$, 
$\xi= L_\Psi \frac{(d\mathcal{C})^{l-1}-1}{d\mathcal{C}-1}$, and the percolation complexity $\mathcal{C} = L_\Psi L_\rho L_\Phi\|W_2\|_2$,   where $L_\Psi, L_\rho$ and $L_\Phi$ are the Lipschitz constants of $\Psi, \rho$ and $\Phi$. For the calculation of the generalization bounds, we use the fully non-asymptotic generalizations bounds, provided in \cite[Subsection A.7]{liao2021a}. There, the PAC-Bayesian based bound is given by
\begin{equation}
    \sqrt{\frac{42^2B^2\left(\max\left( \eta^{-(T+1)}, (\lambda\zeta)^{\frac{T+1}{T}} \right) \right)^2 T^2 h \log(4Th) |w|^2_2}{\gamma^2 m}}.
\end{equation}
The Rademacher based bound is given by 
\begin{equation}
    48h\|W_T\|_2Z \sqrt{\frac{3\log\left( 
    24 \|W_T\|_2 \sqrt{m} \max\left(
    Z, M \sqrt{h} \max(B\|W_1\|_2, \bar{R}\|W_2\|_2)
    \right)
    \right)}{\gamma^2m}}.
\end{equation}

 Note that GraphSage cannot be described in terms of mean field update networks, and vice versa. In order to still report some comparison between the generalization bounds, we offer some conversion between the constants of the two methods, and then apply the PAC-Bayes and Rademacher bounds on the converted bounds. It should be noted that the comparison is a bit like ``comparing apples to oranges,'' but still gives insight into the respective bounds, their asymptotics, and their usefulness in practical situations. 
 
 Since the  transformation by $ \Psi(W_1 (\cdot) + W_2\circ \rho(\cdot))$  can be seen as an update function, similarly to the one in GraphSage, we set in the PAC-Bayes bound $\|W_1\|_2 =\|W_2\|_2 = 1/T \sum_{l=1}^T L_{\Psi^{(l)}}$, where $L_{\Psi^{(l)}}$ is the Lipschitz constant of the update function of GraphSage in the $l$-th layer. The message function in GraphSage is the identity, which corresponds to $\Phi$. We thus convert this to $L_\Phi = 1$ in the PAC-Bayes  generalization bound. 
Finally, we give a lower bound for the maximum degree over all graphs in the datasets by setting  $d = N d_{\mathrm{min}}$ (note that the PAC-Bayes and Rademacher complexity based bounds increase with increasing maximum degree). 

\subsubsection{Generalization Comparison Results}

The results are reported in Figure \ref{fig:ERSBM}.
The different experimental setting are given on the x-Axis. We report experiments for MPNNs with depth $T=1,2,3$  with weight decay (WD) and without weight decay (w/o WD).
The bound values are reported in a logarithmic y-Axis to improve comparability.
In addition to the figures, we also provide numerical values of the bound calculations in Table \ref{table:comparisonNumValues}. 

Our generalization bound is tighter than the PAC-Bayes  bound and the Rademacher  bound under all settings, i.e., for all datasets, for all depths, with weight decay and also without weight decay.

\begin{figure}
\begin{minipage}{\linewidth}
  \centering
  \includegraphics[width=.4\linewidth]{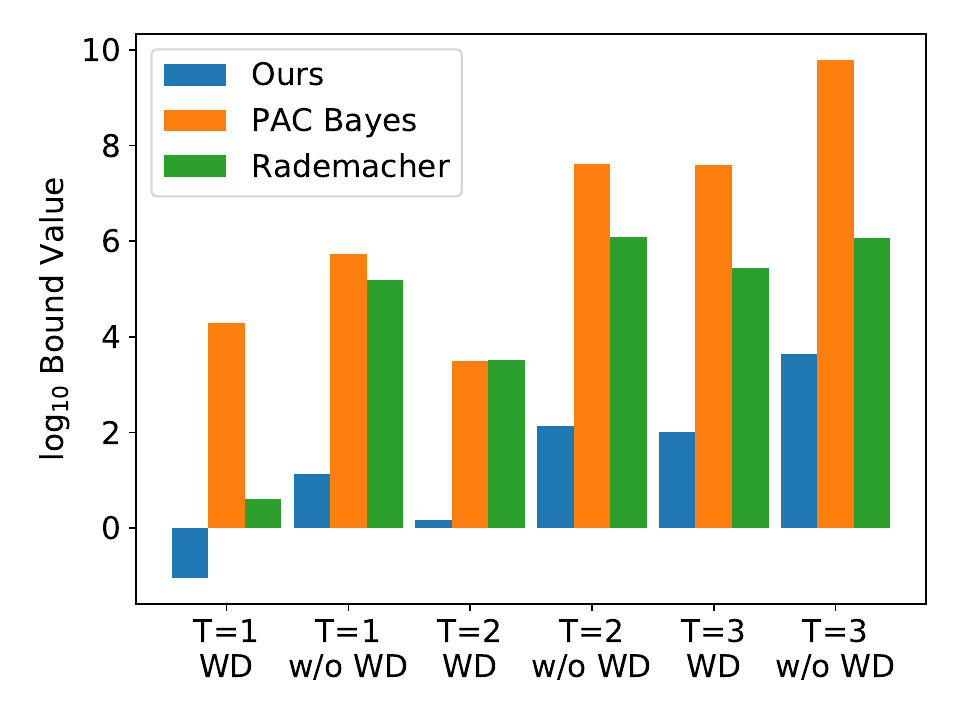}
  \caption*{i) Generalization bounds on  \textbf{ER}-\textbf{SBM}}
\end{minipage}
\vspace{1em} 

\begin{minipage}[b]{0.5\linewidth}
  \centering
  \includegraphics[width=.8\linewidth]{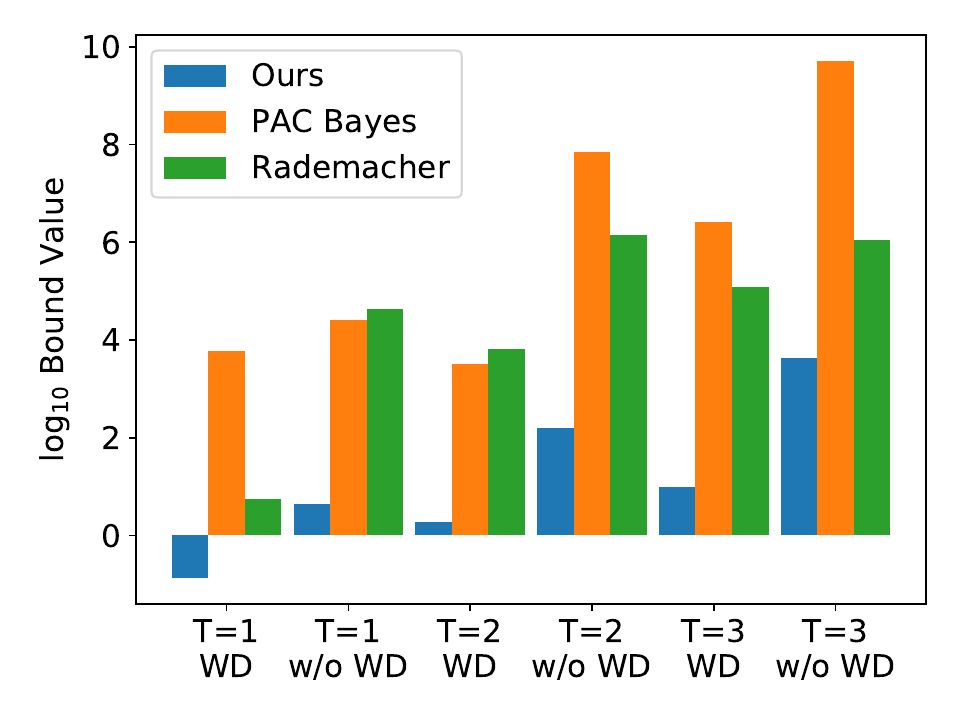}
  \caption*{ii) Generalization bounds on  \textbf{ER}-\textbf{EXP}}
\end{minipage}
\hfill
\begin{minipage}[b]{0.5\linewidth}
  \centering
  \includegraphics[width=.8\linewidth]{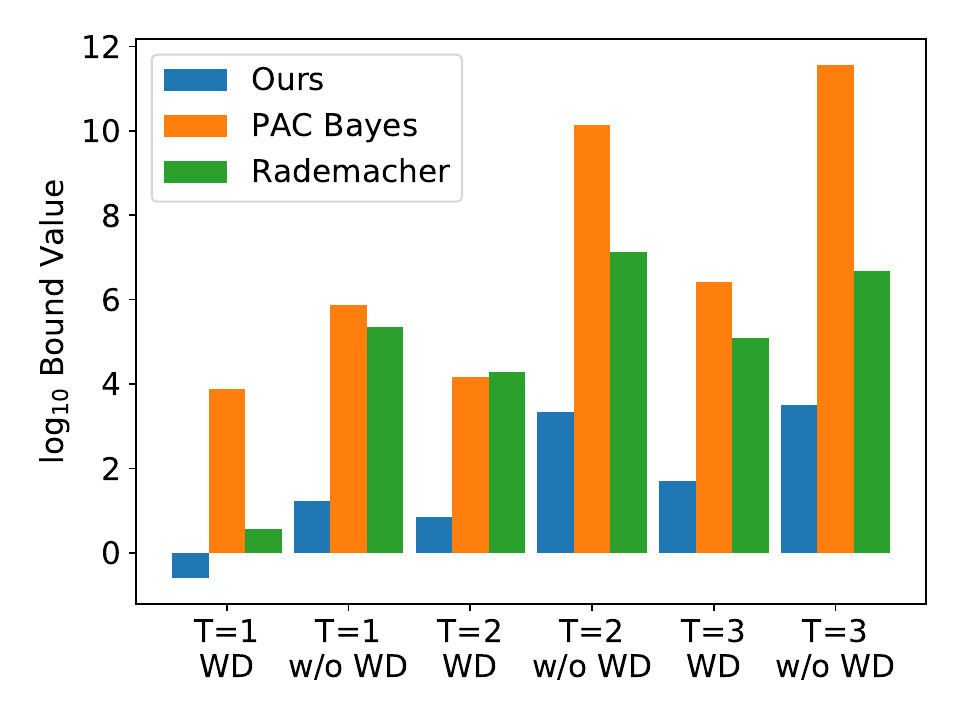}
  \caption*{iii) Generalization bounds on   \textbf{EXP}-\textbf{SBM}}
\end{minipage}
\caption{The generalization bounds on all datasets, i.e., i) \textbf{ER}-\textbf{SBM}, ii) \textbf{ER}-\textbf{EXP} and iii) \textbf{EXP}-\textbf{SBM}, for different number of layers $T=1,2,3$ with weight decay (WD) and without weight decay (w/o WD). Bounds are given in a $\log_{10}$-scale. \label{fig:ERSBM}}
\end{figure}

\begin{table}
\caption{Bound comparisons on all synthetic datasets.\label{table:comparisonNumValues}}
\begin{center}
\vskip 0.15in
\begin{tabular}{cccc} 
\midrule
   T = 1  WD &  ER - SBM   & ER - EXP & SBM - EXP  \cr 
\midrule \midrule

\makecell{Rademacher 
} & $3.9597\times 10^0$& $5.5278 \times 10^0$ & $3.6869 \times 10^0$   \\
\makecell{PAC-Bayesian 
} & $1.9597\times 10^{4}$ & $ 5.8187 \times 10^3$ & $7.8245 \times 10^3$  \\
Ours  & $\mathbf{8.9113\times 10^{-2}}$ & $\mathbf{1.3299 \times 10^{-1}}$ & $\mathbf{2.4561 \times 10^{-1}}$      \\
\midrule \midrule
T = 1 w/o WD  \cr 
\midrule \midrule

\makecell{Rademacher 
} & $1.7329\times 10^5$& $4.3686 \times 10^4 $ & $2.2311 \times 10^5$   \\
\makecell{PAC-Bayesian 
} & $6.0161\times 10^5$ & $ 2.6146 \times 10^4$ & $7.2969 \times 10^5$  \\  Ours  & $\mathbf{ 1.4408\times 10^1}$ & $\mathbf{4.4856 \times 10^0}$ & $\mathbf{1.71143 \times 10^1}$      \\
\midrule \midrule
T = 2  WD \cr 
\midrule \midrule

\makecell{Rademacher 
} & $3.2439 \times 10^3$& $ 6.3963 \times 10^3$ & $1.9428 \times 10^4$   \\
\makecell{PAC-Bayesian 
} & $3.0695 \times 10^3$ & $3.2992 \times 10^3$ & $1.4299 \times 10^4$  \\  Ours  & $\mathbf{1.4788 \times 10^0}$ & $\mathbf{1.8582 \times 10^0}$ & $\mathbf{7.0910 \times 10^0}$      \\
\midrule \midrule
T = 2 w/o WD  \cr 
\midrule \midrule

\makecell{Rademacher 
} & $1.1943 \times 10^6$ & $1.4238\times 10^6 $ & $1.3619 \times 10^7$   \\
\makecell{PAC-Bayesian 
} & $4.0392\times 10^7$ & $6.9262 \times 10^7$ & $1.3817 \times 10^{10}$  \\  Ours  & $\mathbf{1.3526 \times 10^2}$ & $\mathbf{1.5942 \times 10^2}$ & $\mathbf{2.1931 \times 10^3}$      \\
\midrule \midrule
T = 3  WD \cr 
\midrule \midrule

\makecell{Rademacher 
} & $2.7221 \times 10^5$& $1.2286 \times 10^5$ & $1.2529\times 10^5$   \\
\makecell{PAC-Bayesian 
} & $3.9963 \times 10^7$ & $2.6016 \times 10^6$ & $2.6141 \times 10^6$  \\  Ours  & $\mathbf{1.0255 \times 10^2}$ & $\mathbf{9.9522 \times 10^0}$ & $\mathbf{4.9247 \times 10^1}$     

  \\
  \midrule \midrule
T = 3 w/o WD   \cr 
\midrule \midrule

\makecell{Rademacher 
} & $1.1762 \times 10^6$& $1.0872 \times 10^6$ & $4.8271 \times 10^6$   \\
\makecell{PAC-Bayesian 
} & $6.225 \times 10^9$ & $5.1689 \times  10^9$ & $3.7028 \times 10^{11} $  \\  Ours  & $\mathbf{4.3375 \times 10^3}$ & $\mathbf{4.2497 \times 10^3}$ & $\mathbf{3.2436\times 10^3}$      \\
\end{tabular}
\end{center}
\end{table}

\subsection{Additional Comparison of the Generalization Bounds}
\label{subsec:othernumResults}

In this subsection, we present additional plots of the generalization bounds which showcase the dependency on the average graph sizes in the dataset.
The parameters in these plots are set not for a specific dataset and trained network. The plots can be interpreted as the bounds corresponding to training with certain  constraints or regularization terms leading to the respective constants (Lipschitz bounds and infinity norms).

We consider a theoretical setting in which we assume that the following parameters are given: The dataset has 50K graphs, randomly sampled from RGMs with graphons that have maximum infinity norm $\|W\|_\infty = 0.4$ and Lipschitz norm $L_W = 0.5$. We assume that the metric-space signal are bounded by $0.5$ and have Lipschitz constants of maximum $0.5$. Furthermore, we assume there is a linear layer after pooling such that the norms of weight matrix  and of the bias are upper bounded by $0.5$ and $0.1$, respectively. The infinity and Lipschitz norms of the loss function are assumed to be bounded by $1$.

We then consider different datasets with graphs of average size $N=2^4, 2^5, \ldots, 2^{25}$. Since the PAC-Bayes and Rademacher generalization bounds scale with the maximum node degree $d$ of the graphs, we estimate  the degree by setting $d = N\cdot d_{\mathrm{min}}$, where $d_{\mathrm{min}}$ is the graphon degree. We report our generalization bound with respect to the graph size in Figure \ref{tbl:conv4}.
The comparison with other generalization bounds is given in Figure \ref{tbl:conv2}. As expected by our theoretical results, our generalization bound decays with respect to the average graph size. In contrast, we see that both the PAC Bayes based bound and the Rademacher based bound increase with respect to the increasing graph size. 
 
In Figure \ref{fig:Lnorm} we showcase the dependency of our generalization bound on the Lipschitz constant of the graphons. For this, we fix the graph sizes in the dataset to 1000, and compute the resulting bounds for increasing Lipschitz norms. The rest of the parameters are as specified above. We plot the generalization bound for MPNNs with depth $1,2$ and $3$ in Figure \ref{fig:Lnorm}.

\begin{figure}
\begin{center}
    \begin{tabular}{M{60mm}M{60mm}}
  \includegraphics[width=\linewidth]{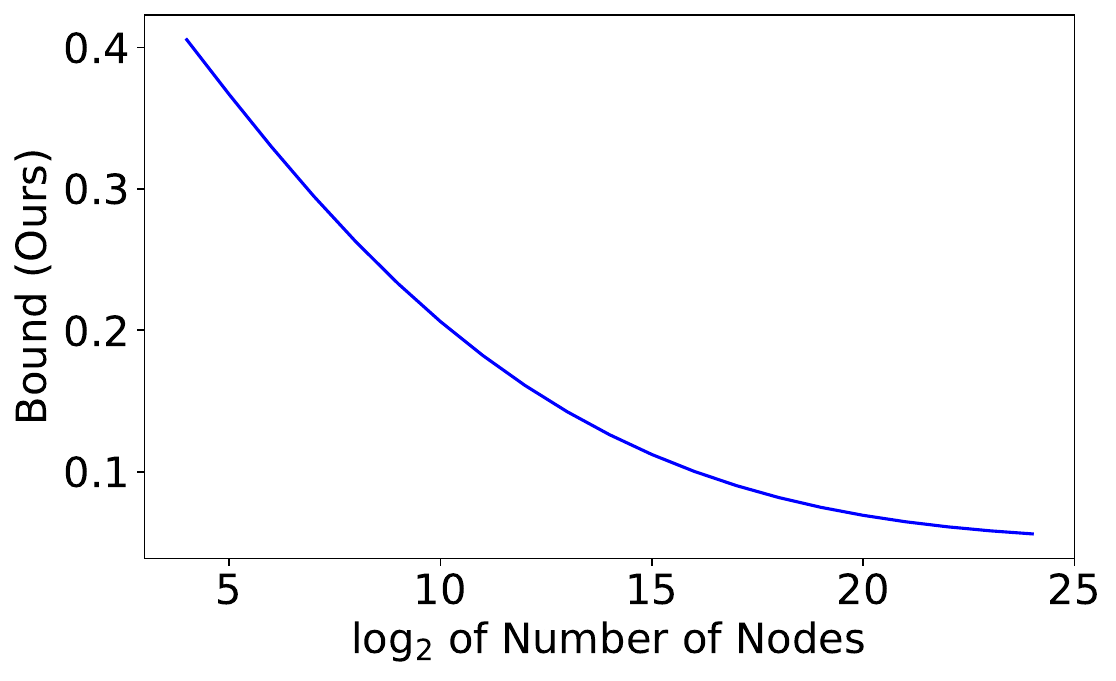} &  
  \includegraphics[width=\linewidth]{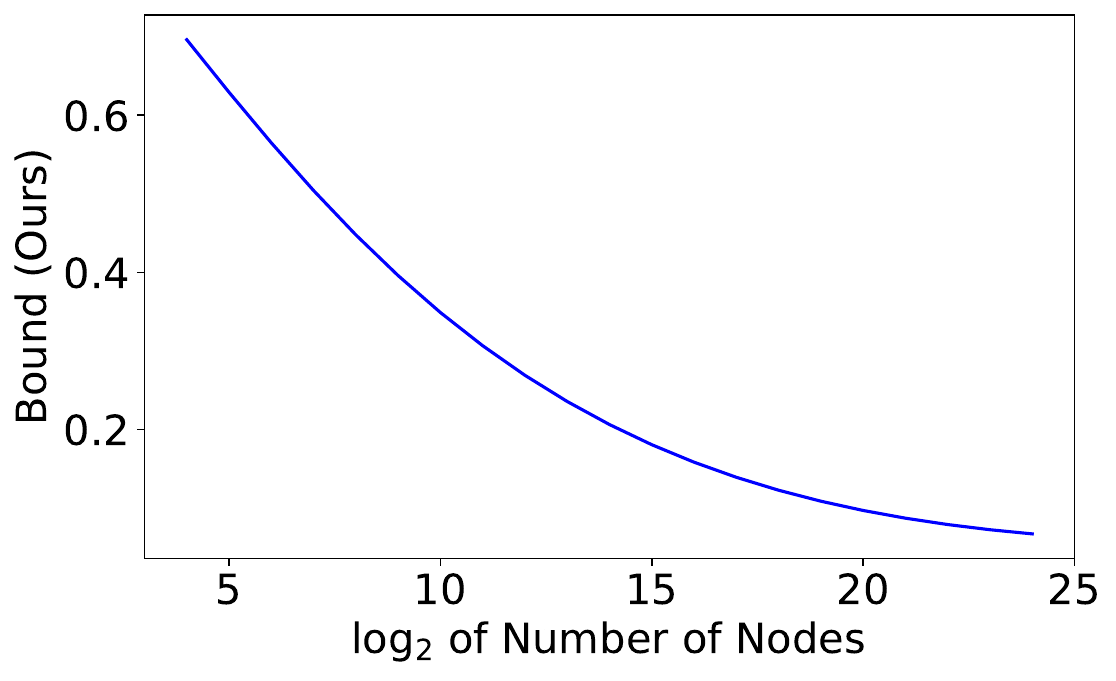}
  \end{tabular}
    \end{center}
    \caption{Our generalization bounds with respect to increasing average graph sizes. On the x-axis we give the average number of nodes in the dataset in $\log_2$-scale. On the y-axis, we give our generalization bound.} 
    \label{tbl:conv4}
\end{figure}

\begin{figure}
\begin{center}
    \begin{tabular}{M{60mm}M{60mm}}
  \includegraphics[width=\linewidth]{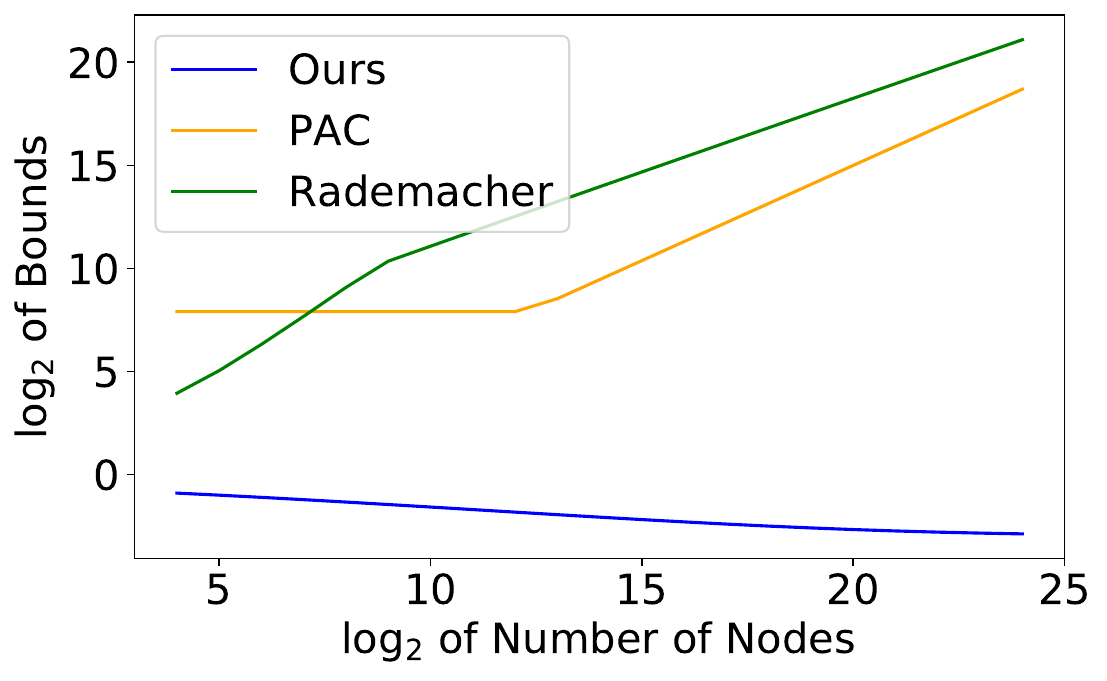} &  
  \includegraphics[width=\linewidth]{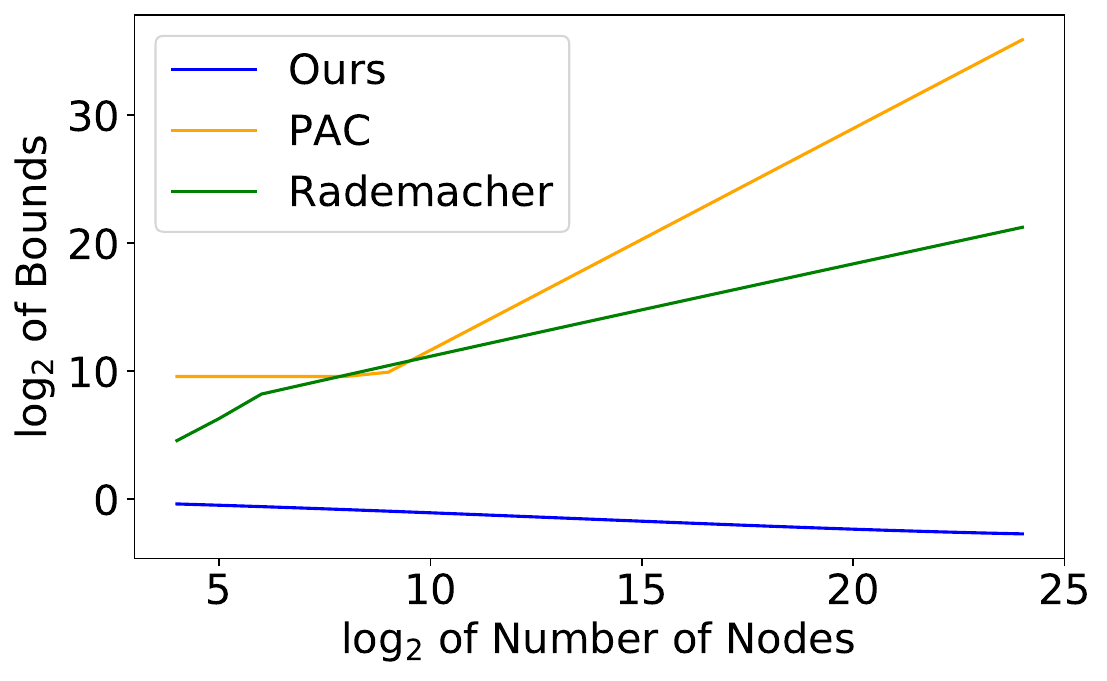}
  \end{tabular}
    \end{center}
    \caption{The generalization bounds with respect to  increasing average graph sizes. On the x-axis, we give the average number of nodes in the dataset in $\log_2$-scale. On the y-Axis, we give our generalization bound, the PAC-Bayes based bound and the Rademacher complexity based bound for MPNNs with depth 2 (left) and depth 3 (right) also in $\log_2$-scale.} 
    \label{tbl:conv2}
\end{figure}

\begin{figure}
\begin{center}
  \includegraphics[width=0.4\linewidth]{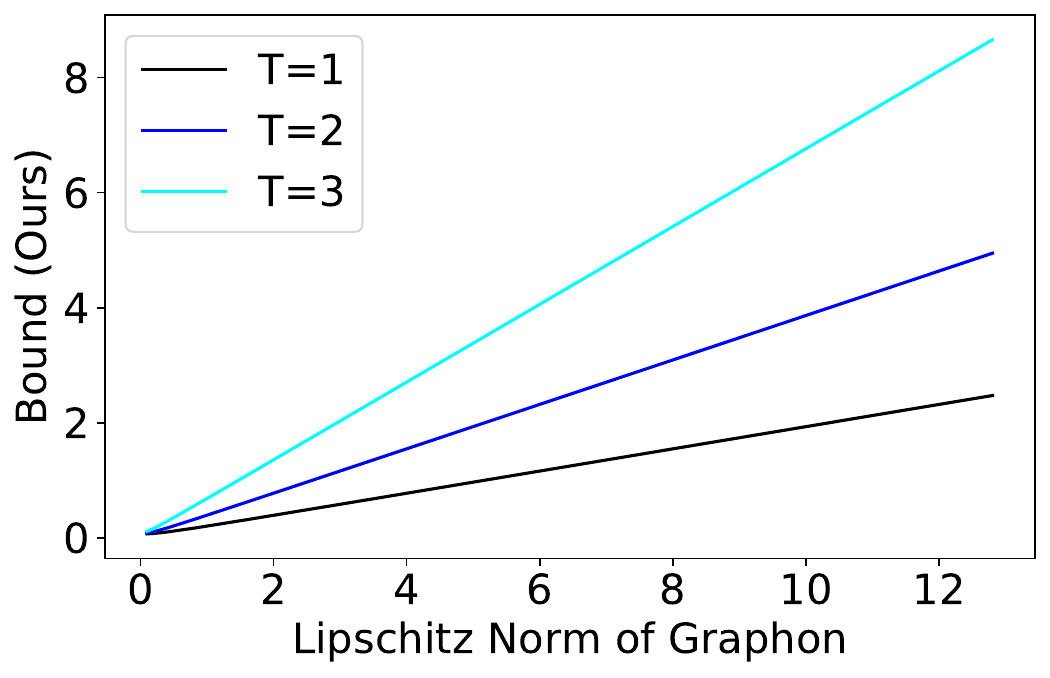} 
    \end{center}
    \caption{Our generalization bounds with respect to increasing average Lipschitz norm of the graphon. On the x-axis we give the maximal Lipschitz norm of the graphons from which we sampl the dataset. On the y-axis, we give our generalization bound.  The rest of the parameters are equal to the parameters in the setting of Figure \ref{tbl:conv2} (see Subsection \ref{subsec:othernumResults}) with fixed graph size $N=1000$.} 
    \label{fig:Lnorm}
\end{figure}

\newpage

\section{Background in Random Processes}
\label{AppendixC}
In this section, we provide background information in probability theory, and focus on random processes and concentration of measure inequalities. 

\begin{definition}[Definition 7.1.1. in \cite{vershynin_2018}]
A \emph{random process} is a collection of random variables $(Y_t)_{t\in T}$ on the same probability space, which are indexed by the elements $t$ of some set $T$.
\end{definition}


The following lemma provides an upper bound on the probability that the sum of bounded independent random variables deviates from its expected value by more than a certain amount.

\begin{theorem}[Hoeffding's Inequality]
\label{thm:Hoeffdings}
Let $Y_1, \ldots, Y_N$ be independent random variables such that $a \leq Y_i \leq b$ almost surely. Then, for every $k>0$, 
\[
\mathbb{P}\Big(
\Big|
\frac{1}{N} \sum_{i=1}^N (Y_i - \E[Y_i] )
\Big| \geq k
\Big) \leq 
2 \exp\Big(-
\frac{2 k^2 N}{( b-a)^2}
\Big).
\]
\end{theorem}


\begin{definition}[Definition 2.5.6 in \cite{vershynin_2018}]
 A random variable $Y$ is called a \emph{sub-Gaussian random variable}  if there exists a constant $K \in \mathbb{R}$ such that $\E\big[ \exp \big(Y^2/K^2 \big) \big] \leq 2$. The \emph{sub-Gaussian norm} of a sub-Gaussian random variable $X$ is defined as
\[
\|Y\|_{\psi_2} = \inf \Big\{ t > 0 : \E\big[  \exp \big(Y^2/t^2\big) \big] \leq  2 \Big\}.
\]
\end{definition}


\begin{lemma}[Example 2.5.8 in \cite{vershynin_2018}]
 \label{lemma:subGaussianNorm2}
 Any bounded random variable $Y$ is sub-Gaussian with
 \[
 \|Y\|_{\psi_2} \leq \frac{1}{\sqrt{\ln(2)}} \|Y\|_\infty.
 \]
\end{lemma}

\begin{definition}[Sub-Gaussian increments, Definition 8.1.1 in \cite{vershynin_2018}]
\label{def:subGaussian}
Consider a random process  $(Y_x)_{x \in \chi}$ on a metric space $(\chi,d)$. We say that the process has \emph{sub-Gaussian increments} if there exists a constant $K \geq 0$ such that
\[
\|Y_x -Y_{x'}\|_{\psi_2} \leq Kd(x,x')
\]
for all $x, x' \in \chi$.  
 We call $ (\|Y_x -Y_{x'}\|_{\psi_2})_{x,x' \in \chi}$ the sub-Gaussian increments of $(Y_x)_{x\in \chi}$.
\end{definition}

\begin{lemma}[Centering of sub-Gaussian random variables, Lemma 2.6.8 in \cite{vershynin_2018}]
\label{lemma:centering}
If $Y$ is a sub-Gaussian random variable, then so is $Y- \E[Y]$, and 
\[
\|Y - \E[Y] \|_{\psi_2} \leq \Big(\frac{2}{\ln(2)} +1 \Big) \|Y\|_{\psi_2}.
\]
\end{lemma}

\begin{lemma}[Proposition 2.6.1 in \cite{vershynin_2018}]
\label{lemma:subGaussianNorm1}
Let $Y_1, \ldots, Y_N$ be independent mean-zero sub-Gaussian random variables. Then, $\sum_{i=1}^N Y_i$ is also a sub-Gaussian random variable, and
\[  
\|\sum_{i=1}^N Y_i\|_{\psi_2}^2 \leq \frac{2}{\sqrt{2}} e 
\sum_{i=1}^N \|Y_i\|_{\psi_2}^2.
\]
\end{lemma}

\begin{theorem}[Dudley's Inequality, Theorem 8.1.6 in \cite{vershynin_2018}]
\label{thm:Dudleys}
Let $(Y_x)_x$ be a random process on a metric space $(\chi, d)$ with sub-Gaussian increments, i.e., there exists a $K\geq 0$ such that $
\|Y_x -Y_{x'}\|_{\psi_2} \leq Kd(x,x')
$
for all $x, x' \in \chi$. Then, for every $u \geq 0$, the event
\[
\sup_{x,x' \in \chi} |Y_x -Y_{x'}| \leq CK \Big(
\int_0^\infty \sqrt{\log \mathcal{C}(\chi, \varepsilon, d)} d\varepsilon +
u \mathrm{diam}(\chi)
\Big)
\]
holds with probability at least $1- 2 \exp(-u^2 )$, where $\mathcal{C}(\chi, \varepsilon, d)$ is defined in Definition \ref{def:epsCover} and $C$ is a universal constant, specified in \cite[Chapter 8]{vershynin_2018}. 
\end{theorem}
\end{document}